\documentclass[11pt,a4paper]{article}
\usepackage[a4paper]{geometry}

\usepackage[dvipsnames]{xcolor}
\definecolor{DarkBlue}{HTML}{02079e}
\definecolor{DarkRed}{HTML}{9e0202}
\usepackage[linktocpage=true]{hyperref}
\hypersetup{
	colorlinks=true,
	urlcolor=DarkBlue,
	citecolor=DarkBlue,
	linkcolor=DarkRed
}

\usepackage{tocloft}
\usepackage{stmaryrd}
\usepackage{caption}

\usepackage{natbib}
\usepackage{amsthm}
\newtheorem{definition}{Definition}[section]
\newtheorem{theorem}[definition]{Theorem}
\newtheorem{lemma}[definition]{Lemma}
\newtheorem{proposition}[definition]{Proposition}

\theoremstyle{definition}
\newtheorem{example}[definition]{Example}
\newtheorem{remark}[definition]{Remark}

\usepackage{algorithm}
\usepackage{algorithmicx}
\usepackage{algpseudocode}

\usepackage{amsmath}
\usepackage{amssymb}
\usepackage{mathtools}
\usepackage{mathrsfs}
\usepackage{bbm}
\usepackage{bm}

\usepackage{tikz-cd}
\usepackage{tikzit}

\usepackage{enumitem,kantlipsum}

\usepackage{subcaption}

\usepackage{mdframed}

\newcommand{\Nat}{\mathbb{N}}

\newcommand{\R}{\mathbb{R}}

\newcommand{\norm}[1]{\lVert#1\rVert}

\newcommand{\ind}{\mathbbm{1}}

\DeclarePairedDelimiterX{\infdivx}[2]{(}{)}{#1\;\delimsize\|\;#2}

\tikzstyle{morphism}=[fill=white, draw=black, shape=rectangle]
\tikzstyle{medium box}=[fill=white, draw=black, shape=rectangle, minimum width=0.8cm]
\tikzstyle{medium large morphism}=[fill=white, draw=black, shape=rectangle, minimum width=1.2cm]
\tikzstyle{large morphism}=[fill=white, draw=black, shape=rectangle, minimum width=1.7cm]
\tikzstyle{bn}=[fill=black, draw=black, shape=circle, inner sep=1.5pt]
\tikzstyle{state}=[fill=white, draw=black, regular polygon, regular polygon sides=3, minimum width=0.8cm, shape border rotate=180, inner sep=0pt]
\tikzstyle{long state}=[fill=white, draw=black, shape=isosceles triangle, isosceles triangle apex angle=90, shape border rotate=270]
\tikzstyle{medium state}=[fill=white, draw=black, regular polygon, regular polygon sides=3, minimum width=1.3cm, inner sep=0pt, shape border rotate=180]
\tikzstyle{large state}=[fill=white, draw=black, regular polygon, regular polygon sides=3, minimum width=2.2cm, shape border rotate=180, inner sep=0pt]
\tikzstyle{wn}=[fill=white, draw=black, shape=circle, inner sep=1.5pt]
\tikzstyle{likelihood}=[fill=white, draw=black, regular polygon, regular polygon sides=3, minimum width=0.8cm, shape border rotate=0, inner sep=0pt]

\tikzstyle{arrow}=[->]
\tikzstyle{dashed box}=[-, dashed]
\tikzstyle{new edge style 0}=[-, fill={rgb,255: red,148; green,162; blue,255}, draw=none]
\input{comonoids.tikzdefs}

\title{Stochastic Neural Network Symmetrisation in \\ Markov Categories}
\author{Rob Cornish \\ \small \textit{Department of Statistics, University of Oxford}}
\date{}

\begin{document}

\maketitle

\begin{abstract}
	We consider the problem of \emph{symmetrising} a neural network along a group homomorphism: given a homomorphism $\varphi : \HH \to \G$, we would like a procedure that converts $\HH$-equivariant neural networks to $\G$-equivariant ones.
We formulate this in terms of Markov categories, which allows us to consider neural networks whose outputs may be stochastic, but with measure-theoretic details abstracted away.
We obtain a flexible and compositional framework for symmetrisation that relies on minimal assumptions about the structure of the group and the underlying neural network architecture.
Our approach recovers existing canonicalisation and averaging techniques for symmetrising deterministic models, and extends to provide a novel methodology for symmetrising stochastic models also.
Beyond this, our findings also demonstrate the utility of Markov categories for addressing complex problems in machine learning in a conceptually clear yet mathematically precise way.

\end{abstract}

\tableofcontents

\section{Introduction}

In many machine learning problems, it is useful to have a neural network that is \emph{equivariant} with respect to some group actions.
That is, for some group $\G$ acting on input and output spaces $\X$ and $\Y$, we would like a neural network $f : \X \to \Y$ that satisfies
\begin{equation} \label{eq:equivariance}
	f(g \cdot x) = g \cdot f(x)
\end{equation}
for all $x \in \X$ and $g \in \G$.
A special case of this is \emph{invariance}, which takes the action on $\Y$ to be trivial, and so the requirement becomes $f(g \cdot x) = f(x)$ instead.
Such constraints arise in many applications involving some geometric structure, such as computer vision, or scientific problems where the data involved are known to follow certain symmetries \citep{bronstein2017geometric,bronstein2021geometric}.
However, most off-the-shelf neural networks are not equivariant.
Unless care is taken, even after training on data that contains symmetries, typically \eqref{eq:equivariance} will fail to hold, possibly to a large degree.
This can reduce performance and robustness, and so an active research area considers how to develop neural networks that are equivariant by design.

\paragraph{Intrinsic equivariance vs.\ symmetrisation}

Following \citet{yarotsky2018universal}, it is helpful to distinguish between two major approaches to obtaining equivariant neural networks.
A significant body of work has focussed on \emph{intrinsic equivariance}, which imposes certain constraints on individual layers of a neural network to ensure that the network as a whole is equivariant \citep{cohen2016group,ravanbakhsh2017equivariance,finzi2021practical}.
In contrast, a recent line of work may be described as \emph{symmetrisation}, which takes an unconstrained neural network and \emph{modifies} it in some way to become equivariant.
For example, when $\G$ is finite, the function
\begin{equation} \label{eq:group-theoretic-janossy-pooling}
	x \,\, \mapsto \,\, \frac{1}{|\G|} \sum_{g \in \G} f(g^{-1} \cdot x)
\end{equation}
that averages over the elements of the group is seen always to be invariant, regardless of $f$ \citep{yarotsky2018universal}.
This observation formed the basis of the \emph{Janossy pooling} approach of \citet{murphy2018janossy}, who took $\G$ to be the symmetric group of permutations, and obtained in this way a neural network that does not depend on the ordering of its inputs.
Subsequently, \emph{frame averaging} \citep{puny2022frame} extended this to the case of equivariance and to more general groups, and provided a technique for reducing the cost of the averaging operation, which can be expensive when $\G$ is large.
An initially parallel approach of \emph{canonicalisation} was proposed by \citet{kaba2023equivariance}, which relies on a single representative element of the group that is chosen in an equivariant way, and thereby avoids averaging altogether.
Both techniques were then generalised by \emph{probabilistic symmetrisation} \citep{kim2023learning}, which averages over a random element of the group that is sampled in an equivariant way.
A related approach of \emph{weighted frames} was also recently proposed by \citet{dym2024equivariant}.
Overall, symmetrisation approaches are attractive as they can leverage unconstrained neural network architectures as their ``backbone'', while still ensuring equivariance overall.
This leads to greater modelling flexibility, which these earlier works have shown can often improve performance compared with intrinsic approaches.

\paragraph{Stochastic equivariance}

\begin{figure}
	\centering
	\includegraphics[width=.8\textwidth]{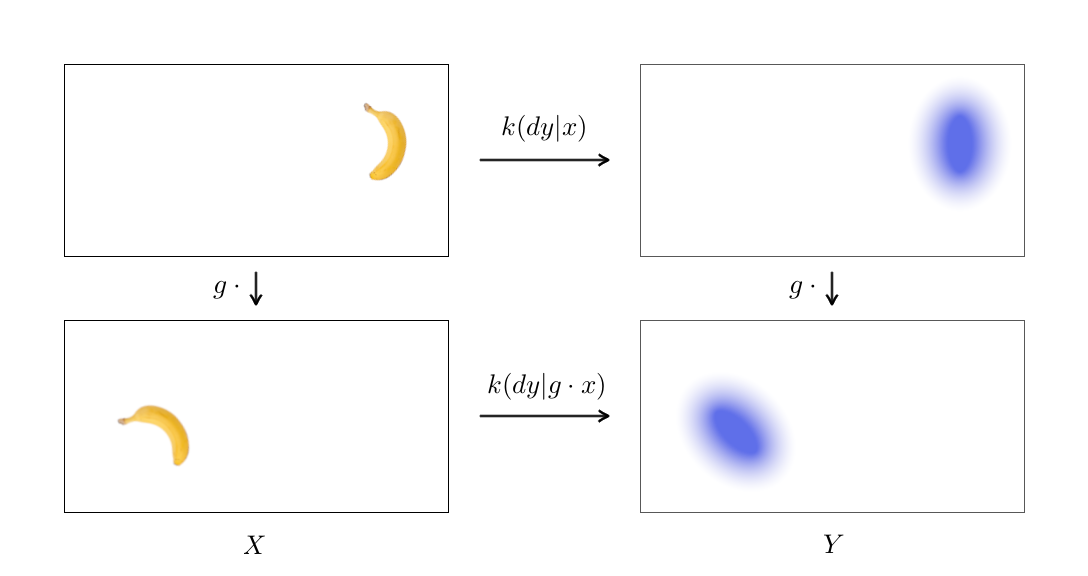}
	\caption{An illustration of stochastic equivariance.
	Here $\X$ is a space of images, $\Y$ is a space of coordinates, and the group $\G$ consists of 2D rotations and translations.
	The model $k : \X \to \Y$ produces a noisy estimate of the location of the banana in its input, with repeated samples depicted here in blue.
	Stochastic equivariance means that the overall \emph{distribution} of these samples varies with the action of the group as shown.
	This is distinct from equivariance at the level of \emph{individual} samples, which is a more rigid constraint.}
	\label{fig:equivariance}
\end{figure}

In this work, we will consider neural networks $f$ that are allowed to depend on some additional randomness $\bm{U}$. %
For such models, the following \emph{stochastic} notion of equivariance is useful in many applications:
\begin{equation} \label{eq:equivariance-stochastic}
	f(g \cdot x, \bm{U}) \eqd g \cdot f(x, \bm{U})
\end{equation}
for all $x \in \X$, $g \in \G$, where $\eqd$ denotes equality in distribution.
This generalises the original deterministic condition \eqref{eq:equivariance}, which is recovered as the degenerate case where $\bm{U}$ is constant.
It is also more general than the pointwise notion considered in equation (17) of \citet{bloem2020probabilistic} and (1) of \citet{lawrence2024improving}, which uses \emph{almost sure} equality in place of the distributional equality here.
In particular, \eqref{eq:equivariance-stochastic} only implies equivariance across repeated samples of $\bm{U}$ rather than for any individual one.
As a result, stochastic equivariance provides a more flexible notion of symmetry that encompasses a wider range of $f$ and $\bm{U}$ than these alternative conditions.

Throughout much of the following, it is convenient to regard the pair of $f$ and $\bm{U}$ as a single entity thought of as a ``stochastic neural network'', rather than decoupling these as in \eqref{eq:equivariance-stochastic}.
We formalise this in terms of \emph{Markov kernels}, defined in Section \ref{sec:markov-kernels}, which are a standard construction in probability theory for modelling conditional distributions, or ``functions with stochastic output''.
In short, a Markov kernel $k : \X \to \Y$ encodes for each $x \in \X$ a probability distribution on $\Y$ that we denote $k(dy|x)$.
We will say that $k$ is \emph{stochastically equivariant} (or simply \emph{equivariant}) if
\begin{equation} \label{eq:equivariance-stochastic-markov-desideratum}
	k(dy|g \cdot x) = g \cdot k(dy|x)
\end{equation}
for all $x \in \X$, $g \in \G$, where the right-hand side denotes the \emph{pushforward} of $k(dy|x)$ by $g$, or in other words the distribution of $g \cdot \bm{Y}$ when $\bm{Y} \sim k(dy|x)$.
See Figure \ref{fig:equivariance} for an illustration.
If \eqref{eq:equivariance-stochastic} holds, we can always obtain an equivariant Markov kernel $k$ by defining $k(dy|x)$ as the distribution of $f(x, \bm{U})$.
Markov kernels therefore subsume our discussion above, and are our main case of interest in what follows.

In the machine learning literature, stochastic equivariance is often defined in terms of conditional \emph{densities} rather than Markov kernels.
Specifically, a conditional density $p(y|x)$ is said to be equivariant if
\[
	p(g \cdot y | g \cdot x) = p(y|x)
\]
for all $x \in \X$, $y \in \Y$, and $g \in \G$, where it is assumed implicitly that the action on $\Y$ has unit Jacobian \citep[Proposition 1]{xu2022geodiff}.
This condition has appeared in the context of generative modelling \citep{xu2022geodiff,hoogeboom2022equivariant,anand2022protein,yim2023se3} and reinforcement learning \citep{brehmer2023edgi,yang2024equibot}, and in situations where uncertainty quantification is required \citep{minartz2024equivariant}.
It is also relevant for the probabilistic symmetrisation method of \citet{kim2023learning}, who require this to hold for a particular subcomponent of their model.
Given its close resemblance to the original deterministic condition \eqref{eq:equivariance}, we regard the kernel-based condition \eqref{eq:equivariance-stochastic-markov-desideratum} as a more natural starting definition of stochastic equivariance and will emphasise this in what follows.
However, as Proposition \ref{prop:stochastic-equivariance-density} below shows, the density-based definition can be recovered as a special case of this, and so we do not lose generality with this approach.

\paragraph{Markov categories}

We use \emph{Markov categories} \citep{cho2019disintegration,fritz2020synthetic} as a framework for reasoning about stochastically equivariant neural networks.
An overview of this topic is provided in Section \ref{sec:markov-categories} below.
At a high level, rather than dealing with Markov kernels directly, we study the behaviour of abstract entities known as \emph{morphisms} that behave like Markov kernels in a precise sense.
In doing so, we can prove results about Markov kernels using intuitive, high-level, and often purely diagrammatic arguments, and without needing to worry about measure theoretic details.
We also gain significant additional generality, and can specialise to various other settings of interest in a seamless way.
For example, although we emphasise stochastic equivariance in what follows, our results are still valid in Markov categories that happen to be purely deterministic, and so also apply to existing work on deterministic symmetrisation as a result.

One consequence of this approach is that we need to generalise various concepts from classical, set-theoretic group theory so that they make sense in a general Markov category.
We do so in Section \ref{sec:group-theory} below, including for groups, homomorphisms, actions, equivariance, semidirect and direct products, orbits, and cosets.
These definitions reduce to the usual ones in classical settings: for example, our general definition of equivariance in Definition \ref{def:equivariance} below recovers the stochastic equivariance condition \eqref{eq:equivariance-stochastic-markov-desideratum} when applied to Markov kernels (Example \ref{ex:stochastic-equivariance}).
Our account aims to be self-contained and accessible, and involves specific considerations that may be of interest for other work that combines groups and Markov categories also.

Throughout the paper, we will make use of various standard concepts from general category theory, including functors, coequalisers, and adjoints.
We have sought to do so sparingly, and only when this provides a large enough conceptual benefit to be justified.
We have also sought to present the overall methodology we obtain in a way that can be applied in practice even without a detailed understanding of these concepts (see e.g.\ Sections \ref{sec:end-to-end-procedure} and \ref{sec:examples}).
For readers unfamiliar with category theory, a highly accessible introduction can be found in \citet{perrone2021notes}.
We will also provide more specific references in various places, as well as examples of these concepts in more concrete settings where appropriate.
For readers who do know category theory, these parts can be skipped over without loss of continuity.

\paragraph{Symmetrising along a homomorphism}

Let $\C$ be a Markov category, and let
\[
	\varphi : \HH \to \G
\]
be a homomorphism between groups in $\C$ (see Definition \ref{def:homomorphism}).
For practical purposes, $\varphi$ may be thought of as a subgroup inclusion.
At a high level, the problem of symmetrisation we consider is to find some mapping that converts $\HH$-equivariant morphisms to $\G$-equivariant morphisms.
Most existing work has considered the specific case where $\HH$ is the trivial group, in which case $\HH$-equivariance always holds vacuously, and a mapping of this kind therefore takes as input an arbitrary morphism in $\C$, which may be regarded as an unconstrained neural network.
However, in the deterministic setting, Section 3.3 of \citet{kaba2023equivariance} also provides a sufficient condition for their canonicalisation procedure also to apply for general subgroup inclusions.
This allows already equivariant models to be made ``more so'' without inadvertently undoing other existing symmetries that are already present.
We take this more general problem as our starting point, and our methodology applies uniformly for all choices of $\varphi$, including in the stochastic setting.
It can also be applied compositionally by symmetrising along multiple homomorphisms in sequence, thereby building up more complex equivariance properties in a structured way (Section \ref{sec:composing-procedures}).

More formally, the group $\G$ canonically gives rise to a Markov category $\C^\G$ whose objects are \emph{$\G$-objects}, that is,  objects of $\C$ equipped with some action of $\G$, and whose morphisms are equivariant with respect to these $\G$-actions (see Definition \ref{def:markov-category-of-equivariant-morphisms}).
We obtain a Markov category $\C^\HH$ of $\HH$-objects and $\HH$-equivariant morphisms in a similar way.
Both categories are always related by a functor $\Res_\varphi : \C^\G \to \C^\HH$ that acts by \emph{restriction} along the homomorphism $\varphi$ (Definition \ref{def:action-restriction-functor}).
For example, when $\varphi$ is a subgroup inclusion, $\Res_\varphi$ maps each $\G$-object to the $\HH$-object obtained simply by restricting its $\G$-action to $\HH$ (Example \ref{ex:subgroup-restricted-action}).
Given $\G$-objects $\X$ and $\Y$, we now formalise the overall goal of symmetrisation as to obtain some function $\sym$ of the form
\begin{equation} \label{eq:symmetrisation-procedure-intro}
	\begin{tikzcd}
		\C^\HH(\Res_\varphi\X, \Res_\varphi\Y) \ar{r}{\sym} & \C^\G(\X, \Y),
	\end{tikzcd}
\end{equation}
which we refer to as a \emph{symmetrisation procedure} (Definition \ref{def:symmetrisation-procedure}).
Here we use the standard notation $\D(\U, \V)$ to denote the set of morphisms $\U \to \V$ in a category $\D$.
Notice that $\sym$ of this form maps $\HH$-equivariant morphisms to $\G$-equivariant ones, which is intuitively what a symmetrisation procedure ought to do.
In particular, when $\varphi$ is a subgroup inclusion, the function $\sym$ ``upgrades'' a morphism that is only equivariant with respect to the subgroup to become equivariant with respect to the whole group.
As an important special case, when $\HH$ is the trivial group, $\sym$ converts completely unconstrained morphisms in $\C$ into ones that are $\G$-equivariant (Example \ref{ex:symmetrisation-along-trivial-homomorphism}).

\paragraph{General framework}

The methodological question now is, how can we actually construct symmetrisation procedures in practice?
In many familiar settings, such as for set-theoretic groups, the functor $\Res_\varphi$ admits a \emph{left adjoint} $\Ext_\varphi : \C^\HH \to \C^\G$ known as \emph{extension} or \emph{induction} (see e.g.\ Chapter I.1 of \citet{may1997equivariant}).
When this holds, for all $\G$-objects $\X$ and $\Y$, we obtain a bijection
\begin{equation} \label{eq:symmetrisation-procedure-intro-3}
	\begin{tikzcd}
		\C^\HH(\Res_\varphi \X, \Res_\varphi \Y) \ar{r}{\cong} & \C^\G(\Ext_\varphi \Res_\varphi \X, \Y).
	\end{tikzcd}
\end{equation}
Notice that the left-hand side here is exactly the left-hand side of \eqref{eq:symmetrisation-procedure-intro}.
This yields a full characterisation of \emph{every} possible symmetrisation procedure $\sym$: each one can be written uniquely as a composition
\begin{equation*} \label{eq:symmetrisation-procedure-intro-2}
	\begin{tikzcd}
		\C^\HH(\Res_\varphi \X, \Res_\varphi \Y) \ar{r}{\cong} & \C^\G(\Ext_\varphi \Res_\varphi \X, \Y) \ar{r} & \C^\G(\X, \Y)
	\end{tikzcd}
\end{equation*}
for some choice of function in the second step (which may be arbitrary).
If we want a method that is general purpose and applies without further assumptions on the components involved, then there is only one obvious choice for this second step: \emph{precomposition} by some morphism $\Pre : \X \to \Ext_\varphi \Res_\varphi \X$ in $\C^\G$.
In other words, given an input $k : \Ext_\varphi \Res_\varphi \X \to \Y$ in $\C^\G$, this step outputs $k \circ \Pre : \X \to \Y$, which is always a morphism of the correct type simply because $\C^\G$ is closed under composition.
End-to-end, by applying the bijection \eqref{eq:symmetrisation-procedure-intro-3} and then precomposing the result by some $\Pre$, we always obtain a symmetrisation procedure $\sym_\Pre$ of the desired form \eqref{eq:symmetrisation-procedure-intro}.
This strategy works for all groups and actions in $\C$ without additional requirements such as compactness, and is therefore highly generic.

A full left adjoint $\Ext_\varphi$ is more than we require here: to obtain a bijection of the form \eqref{eq:symmetrisation-procedure-intro-3}, we in fact only need to construct the composite functor $\Ext_\varphi \Res_\varphi$.
In Theorem \ref{thm:partial-left-adjoint-existence}, we provide a sufficient condition that allows us to do so.
We also show in Theorem \ref{thm:topstoch-has-orbits} that this condition is always satisfied when $\C$ is the Markov category of topological spaces and continuous Markov kernels, which seems more than adequate for most applications in machine learning.
In practice, it is also readily satisfied for many other groups of interest in other contexts, as we demonstrate with various examples throughout.
In Section \ref{sec:concrete-instantiation}, we then show that the existing symmetrisation methods mentioned above can all be recovered as instances of this approach for specific choices of $\C$ and $\varphi$, in some cases when combined with a further \emph{expectation} step that ensures the result is always deterministic (see Section \ref{sec:deterministic-symmetrisation}).
Our framework thereby shows these methods all ``come from'' a common underlying principle based on \eqref{eq:symmetrisation-procedure-intro-3}.

\paragraph{Methodological contributions}

Beyond this streamlined theoretical picture, our framework also provides various novel methodological contributions.
For example, apart from \citet{kaba2023equivariance} (Section 3.3), previous work has only considered the task of symmetrising unconstrained morphisms in $\C$.
Our approach extends these methods to handle morphisms that are already ``partially'' equivariant, which provides additional structure that can be exploited.
It also gives rise to novel techniques for \emph{compositional} and \emph{recursive} symmetrisation, both of which can be useful in practice (Section \ref{sec:composing-procedures} and Remark \ref{rem:recursive-symmetrisation}).

In addition, since our theory is developed entirely in terms of Markov categories, it yields a methodology that applies directly to Markov kernels, the symmetrisation of which has not previously been considered.
We refer to this as \emph{stochastic symmetrisation} to distinguish it from the \emph{deterministic symmetrisation} methods of previous work, which instead apply only to deterministic functions.
Stochastic symmetrisation is very simple to implement, admitting an exact sampling procedure given in Example \ref{ex:symmetrising-markov-kernels}, and resembles a generalised data augmentation procedure that may be learned during training and is applied also at test time.
It moreover avoids the expectation operator $\ave$ required by various deterministic procedures such as \citet{puny2022frame} and \citet{kim2023learning}, which is computationally expensive (or only approximate), and is not even defined when $\Y$ is not a convex space like $\R^d$ (whereas stochastic symmetrisation always is).
We use stochastic symmetrisation to obtain the stochastically equivariant subcomponent required by \citet{kim2023learning}, and show empirically that this yields improved performance on several synthetic numerical examples.

More generally, our work serves as a case study in how Markov categories can be used to reason about real machine learning problems arising in practice.
Markov categories provides a canvas to describe everything from our background theory to our experimental setup using the same coherent language and notations, which allows us to encompass a variety of complex situations in a lightweight and uniform way.
Section \ref{sec:examples} demonstrates this through a variety of examples that show how our symmetrisation framework can be applied across compact groups, translation groups, semidirect products, and even the full general linear group.
Anecdotally, the new methodology we propose was found directly as a result of this greater conceptual clarity, and we believe similar benefits could be obtained by applying categorical ideas to other applications in machine learning also.

\paragraph{Outline}

The rest of the paper is structured as follows.
In Section \ref{sec:markov-kernels} we provide a brief introduction to Markov categories.
We aim in particular to allow readers unfamiliar with \emph{string diagrams} to parse these easily, as they will be used throughout the paper.
In Section \ref{sec:group-theory}, we develop some basic concepts from group theory in the context of a general Markov category.
For readers primarily interested in our methodology, this section can initially be skipped over and referred back to on an as-needed basis.
In Section \ref{sec:symmetrisation}, we formulate the overall goal of symmetrisation, and describe the high-level behaviour of symmetrisation procedures.
In Section \ref{sec:adjunctive-symmetrisation}, we give our overall framework for symmetrisation based on bijections of the form \eqref{eq:symmetrisation-procedure-intro-3}.
In Section \ref{sec:examples}, we provide various examples of how this approach can be applied for different groups and actions of interest.
Finally, in Section \ref{sec:empirical-results}, we consider implementation details and provide empirical results.

\section{Background on Markov categories} \label{sec:markov-categories}

We provide a short introduction to Markov categories here, and refer the reader to \citet{cho2019disintegration} and \citet{fritz2020synthetic} for a more detailed treatment.

\subsection{Markov kernels} \label{sec:markov-kernels}

A Markov category can be understood as an abstraction of the key structural behaviour of \emph{Markov kernels}, also known as stochastic maps.
Given measurable spaces $\X$ and $\Y$, recall that a Markov kernel is a function
\begin{equation} \label{eq:markov-kernel}
	k : \Sigma_Y \times \X \to [0, 1],
\end{equation}
where $\Sigma_\Y$ denotes the $\sigma$-algebra of $\Y$, such that $x \mapsto k(B|x)$ is measurable for each $B \in \Sigma_\Y$, and $B \mapsto k(B|x)$ is a probability measure for each $x \in \X$.
The idea is that Markov kernels formalise the notion of a conditional probability distribution in a rigorous, measure theoretic way, thereby allowing more careful reasoning than is permitted by the usual ``density'' notation $p(y|x)$ often used in applications.
In what follows, we will denote the probability measure $B \mapsto k(B|x)$ more suggestively using ``infinitesimal'' notation.
So, for example,
\[
	\bm{Y} \sim k(dy|x)
\]
denotes a $\Y$-valued random variable sampled from $k$ given the input $x \in \X$.
In addition, rather than writing out the full signature of a Markov kernel as in \eqref{eq:markov-kernel}, we will simply write
\[
	k : \X \to \Y.
\]
This emphasises the interpretation of Markov kernels as stochastic maps, with $\X$ and $\Y$ regarded as the \emph{domain} and \emph{codomain} of $k$ respectively.
Following other work on Markov categories, we will also represent this graphically using \emph{string diagrams} \citep{selinger2010survey} as follows:
\[
	\input{basic-markov-kernel-string-diagram.tikz}
\]
We will follow the same conventions as Section 2 of \citet{cho2019disintegration} when denoting string diagrams, so that in particular these should always be read from bottom to top.
Intuitively, a string diagram represents a \emph{generative process}, with each box denoting some kind of (potentially stochastic) operation or computation, and whose wires track the flow of information, with time flowing upwards.
Such diagrams are used informally throughout the machine learning literature already when describing (for example) neural network architectures, and readers will likely find themselves comfortable with this notation after seeing the examples below.

\subsection{The Markov category $\Stoch$} \label{sec:markov-category-stoch}

The prototypical example of a Markov category is $\Stoch$, the Markov category of measurable spaces and Markov kernels.
We describe the structure of this Markov category now.
Our description will be somewhat informal, emphasising the ``sampling'' or ``generative'' interpretation of Markov kernels rather than their rigorous definition as functions of the form \eqref{eq:markov-kernel}.
A formal treatment can be found in Section 4 of \cite{fritz2020synthetic}.

Markov kernels $k : \X \to \Y$ and $m : \Y \to \Z$ can always be \emph{composed sequentially} to obtain a new Markov kernel $m \circ k : \X \to \Z$.
We sample from $(m \circ k)(dz|x)$ as follows:
\[
	\bm{Y} \sim k(dy|x) \quad\quad \bm{Z} \sim m(dz|{\bm Y}) \qquad \text{return $\bm{Z}$.}
\]
This make sense whenever the codomain of $k$ matches the domain of $m$.
In string diagrams, $m \circ k$ is represented as follows:
\[
	\input{composition-of-markov-kernels.tikz}
\]
For every measurable space $\X$, there is also an \emph{identity} kernel $\id_\X : \X \to \X$ that simply returns its input.
In string diagrams, this is drawn simply as a wire:
\[
	\begin{tikzpicture}
	\begin{pgfonlayer}{nodelayer}
		\node [style=none] (0) at (0, -1) {};
		\node [style=none] (1) at (0, 1) {};
		\node [style=none] (2) at (0, -1.5) {$\X$};
		\node [style=none] (3) at (0, 1.5) {$\X$};
	\end{pgfonlayer}
	\begin{pgfonlayer}{edgelayer}
		\draw (1.center) to (0.center);
	\end{pgfonlayer}
\end{tikzpicture}

\]
However, there is more structure at play here. %
Given two measurable spaces $\X$ and $\Y$, we can always form the product measurable space, denoted $\X \otimes \Y$.
Now suppose we have two Markov kernels $k : \X \to \Y$ and $m : \U \to \V$.
We can then always obtain a new Markov kernel $k \otimes m : \X \otimes \U \to \Y \otimes \V$ between these product spaces by \emph{parallel composition}, or in other words by sampling from each kernel independently.
That is, to sample from $(k \otimes m)(dy, dv|x, u)$, we do the following:
\[
	\bm{Y} \sim k(dy|x) \quad\quad \bm{V} \sim m(dv|u) \qquad \text{return $(\bm{Y}, \bm{V})$,}
\]
where here $\bm{Y}$ and $\bm{V}$ are sampled \emph{independently}.
In string diagrams, $k \otimes m$ is denoted
\[
	\input{parallel-composition-of-markov-kernels.tikz}
\]
In addition, for any two measurable spaces $\X$ and $\Y$, there is also always a kernel $\swap_{\X,\Y} : \X \otimes \Y \to \Y \otimes \X$ that deterministically \emph{swaps} its two inputs, so that given an input $(x, y)$, it simply returns $(y, x)$.
In string diagrams, $\swap_{\X,\Y}$ is denoted suggestively as
\[
	\input{swap-markov-kernel.tikz}
\]
While this may seem like a trivial operation, these kernels perform a fundamental role: they allow us to reorganise the layout of a string diagram in any way we want, provided we do not change its overall topology (so each box is connected to the same inputs and outputs before and afterwards).
A similarly trivial but very useful operation is \emph{copying}: for any measurable space $\X$, there is a Markov kernel $\cop_\X : \X \to \X \otimes \X$ that, when given $x$ as input, returns the pair $(x, x)$.
In string diagrams, we denote this kernel as:
\[
	\input{copy-markov-kernel.tikz}
\]
Finally, let $\I$ denote the trivial measurable space consisting of the singleton set $\{\singleton\}$ equipped with the trivial $\sigma$-algebra.
Then for any measurable space $\X$, there is always a (unique) Markov kernel $\del_\X : \X \to \I$.
We represent this in string diagrams as:
\[
	\begin{tikzpicture}
	\begin{pgfonlayer}{nodelayer}
		\node [style=none] (0) at (0, 0) {};
		\node [style=bn] (1) at (0, 1) {};
		\node [style=none] (2) at (0, -0.5) {$\X$};
	\end{pgfonlayer}
	\begin{pgfonlayer}{edgelayer}
		\draw (0.center) to (1);
	\end{pgfonlayer}
\end{tikzpicture}

\]
and understand this as the Markov kernel that simply \emph{discards} its input.
The trivial space $\I$ also plays another useful role: kernels $p : \I \to \X$ encode a single probability distribution $p(dx|\singleton)$, which allows us to recover \emph{unconditional} distributions as a special case of Markov kernels.
In string diagrams, we denote these kernels without an input wire as follows:
\[
	\begin{tikzpicture}
	\begin{pgfonlayer}{nodelayer}
		\node [style=state] (0) at (0, -0.5) {$p$};
		\node [style=none] (1) at (0, -0.25) {};
		\node [style=none] (2) at (0, 0.5) {};
		\node [style=none] (3) at (0, 1) {$\X$};
	\end{pgfonlayer}
	\begin{pgfonlayer}{edgelayer}
		\draw (1.center) to (2.center);
	\end{pgfonlayer}
\end{tikzpicture}

\]

\subsection{Markov categories}

The general definition of a \emph{Markov category} axiomatises the behaviour of Markov kernels described in the previous section.
We provide a high-level overview here, and refer the reader to Definition 2.1 of \citet{fritz2020synthetic} for a rigorous definition.
Instead of measurable spaces and Markov kernels, the data of a general Markov category $\C$ consists of a collection of \emph{objects} and a collection of \emph{morphisms}.
Like we did for Markov kernels, we denote these morphisms by $k : \X \to \Y$, where $\X$ and $\Y$ are objects in $\C$ referred to as the \emph{domain} and \emph{codomain} of $k$ respectively.
We can compose morphisms sequentially (provided they are compatibly typed), and each object $\X$ comes equipped with an identity morphism $\id_\X$.
For any pair of objects $\X$ and $\Y$, we can form their \emph{monoidal product} $\X \otimes \Y$, which plays the role of the product measurable space.
We can then compose morphisms $k$ and $m$ in parallel to obtain a new morphism $k \otimes m$ between the monoidal products of their domains and codomains.
There is also a distinguished object $\I$ referred to as the \emph{monoidal unit} that plays the role of the singleton set, as well as morphisms for swapping, copying, and deleting information.
We depict all these constructions using string diagrams in just the same way as we did for Markov kernels.

The formal definition of a Markov category includes additional axioms that ensure these constructions behave like Markov kernels do.
For example, it always holds that:
\[
	\input{copy-and-swap.tikz}
\]
The first condition says that copying some input and then swapping the result is the same as just copying the input.
The second condition says that sampling from $k$ given some input and then discarding the result is the same as just discarding the input.
Both are intuitively always true for Markov kernels, at least when these are regarded informally as generative processes.
We will perform similar manipulations throughout the paper, which can also be understood by analogy with Markov kernels in this way.

\subsection{Examples of Markov categories}

The Markov category $\Stoch$ described in Section \ref{sec:markov-category-stoch}, whose objects are measurable spaces and whose morphisms are Markov kernels, will serve as a key example throughout the paper.
However, there are other interesting Markov categories beyond this.
For our purposes, it will also be relevant to consider $\TopStoch$ \citep[Example A.1.4]{fritz2023absolute}, whose objects are topological spaces and whose morphisms are \emph{continuous} Markov kernels, where a Markov kernel between topological spaces $k : \X \to \Y$ is continuous if for every open subset $U \subseteq \Y$, the function $x \mapsto k(U|x)$ is lower semicontinuous \citep[Section 4]{fritz2021probability}.
The monoidal product $\otimes$ returns the product topological space, the monoidal unit $\I$ is the singleton topological space, and the remaining components are defined analogously as for $\Stoch$.
Our main reason for interest in $\TopStoch$ is that it allows us to prove Theorem \ref{thm:topstoch-has-orbits} below, whereas we are not sure whether the analogous result holds for $\Stoch$.
For practical purposes, $\TopStoch$ is still very general, since neural networks used in practice are almost invariably continuous.

Beyond these, more basic examples of Markov categories include: $\Set$, whose objects are sets and whose morphisms are functions; $\Meas$, whose objects are measurable spaces and whose morphisms are measurable functions; and $\Top$, whose objects are topological spaces and whose morphisms are continuous functions.
For $\Set$, the monoidal product is just the cartesian product and the monoidal unit is the singleton set.
For $\Meas$ and $\Top$ these are similar, but are now equipped with suitable $\sigma$-algebras and topologies respectively.
That these are indeed Markov categories follows from Remark 2.4 of \cite{fritz2020synthetic}.
For the purposes of modelling stochastic phenomena, these examples are clearly less interesting than $\Stoch$ and $\TopStoch$, but they will be useful for examples and intuition in what follows.

\begin{remark}
	All the preceding examples are \emph{positive} Markov categories \citep[Definition 11.22]{fritz2020infinite}.
	This is a mild technical condition that, roughly speaking, ensures \emph{determinism} (see below) has various implications that we would expect.
	That $\Stoch$ is positive is shown in Example 11.25 of \citet{fritz2020infinite}, and this in turn also implies that $\TopStoch$ is positive \citep[Remark 11.26]{fritz2020infinite}.
	It is also straightforward to show that $\Set$, $\Meas$, and $\Top$ are all positive too.
	We will only make use of this condition in a few places.
	The point of this remark is that, when we do, we will not have sacrificed much generality for practical purposes.
\end{remark}

\subsection{Determinism} \label{sec:determinism}

Following Definition 10.1 of \citet{fritz2020synthetic}, we will say that a morphism $f : \X \to \Y$ in a Markov category $\C$ is \emph{deterministic} if it holds that
\[
	\input{deterministic-definition.tikz}
\]
Intuitively, this says that repeated independent samples from $f$ always produce the same result.
For example, if $f$ is a deterministic Markov kernel and we independently sample $\bm{Y}, \bm{Y}' \sim f(dy|x)$,
then we have
\[
	(\bm{Y}, \bm{Y}) \eqd (\bm{Y}, \bm{Y}'),
\]
where notice that the random variable $\bm{\Y}$ appears twice on the left-hand side.
By Remark 10.13 of \citet{fritz2020synthetic}, the deterministic morphisms themselves always form a Markov category, denoted $\Cdet$.

\begin{remark} \label{rem:measurable-function-as-markov-kernel}
	By Example 10.4 of \citet{fritz2020synthetic}, every measurable function $f : \X \to \Y$ gives rise to a deterministic Markov kernel $k_f : \X \to \Y$, where $k_f(dy|x)$ is Dirac on $f(x)$.
	More tersely, morphisms in $\Meas$ may always be \emph{lifted} to become deterministic morphisms in $\Stoch$.
	A similar story is also true for $\Top$ and $\TopStoch$, where it is straightforward to show that $k_f$ is a lower semicontinuous Markov kernel whenever $f$ is continuous.
	In these cases, we will often slightly abuse notation by writing the lifted morphism $k_f$ simply as $f$.
\end{remark}

\section{Group theory in Markov categories} \label{sec:group-theory}

We now show how various concepts from basic group theory can be considered in the context of a general Markov category.
This section can be read somewhat nonlinearly and referred back to as definitions appear in later sections.

\subsection{Groups, homomorphisms, and actions}

We begin by providing the basic definitions of group theory internal to a general Markov category.
A feature of our approach is that we require the various group operations all to be \emph{deterministic}.
This allows us to recover standard results from set-theoretic group theory that might otherwise not hold in a general Markov category.
We discuss this in more detail in Remark \ref{rem:translating-classical-theory} below.

\begin{definition} \label{def:group}
	A \emph{group} in a Markov category $\C$ consists of an object $\G$ and deterministic morphisms $\mul : \G \otimes \G \to \G$, $\e : \I \to \G$, and $\inv : \G \to \G$ in $\C$ that satisfy the following:
	\[
		\input{group-object-definition.tikz}
	\]
\end{definition}

To streamline notation, we will refer to the overall group simply as $\G$, leaving its operations implicit.
We will also mostly reuse the same symbols $\mul$, $\e$, and $\inv$ to denote these operations across all groups, even when these may be distinct.
When we need to disambiguate, we will use a subscript, e.g.\ $\mul_\G$ for the multiplication operation of $\G$.

\begin{example} \label{ex:group-in-set}
We spell out how this relates to the classical definition of a group.
Under Definition \ref{def:group}, a group in $\Set$ is a set $\G$ equipped with functions $\mul$, $\e$ and $\inv$.
Writing $\mul(g, h)$ simply as $g h$, the first axiom here says
\[
	(g h) n = g (h n)
\]
for all $g, h, n \in \G$.
In other words, $\mul$ is \emph{associative}.
Similarly, recall that $\I$ in $\Set$ is the singleton set $\{\singleton\}$, and so $\e : \I \to \G$ encodes a unique value $\e(\singleton)$.
Denoting $\e(\singleton)$ simply as $e$, the second axiom says
\[
	eg = ge = g
\]
for all $g \in \G$,
and hence $e$ serves as the \emph{unit}.
Finally, writing $\inv(g)$ as $g^{-1}$, the third axiom says
\[
	g^{-1} g = g g^{-1} = e
\]
for all $g \in \G$.
In other words, every element in $\G$ has an \emph{inverse}.
This shows that groups in $\Set$ in the sense of Definition \ref{def:group} are precisely groups in the classical sense.

In what follows, we will denote the multiplication, unit, and inversion operations of groups in $\Set$ in the same traditional way as we do here.
We will do the same for groups in $\Meas$ and $\Top$ (discussed next), where this notation also makes sense.
\end{example}

\begin{example}
	For other Markov categories whose morphisms are functions, such as $\Meas$ and $\Top$, a similar story to Example \ref{ex:group-in-set} applies.
	However, in these cases, the group operations are more restricted than in $\Set$, since they are required to be morphisms of the ambient category.
	For example, group operations in $\Meas$ are required to be measurable functions, and group operations in $\Top$ are required to be continuous functions.
	This is very natural: for example, in probability theory it is standard to study \emph{measurable groups} (e.g.\ Page 15 of \citet{kallenberg1997foundations}), which exactly correspond to groups in $\Meas$ here.

	In $\Stoch$ and $\TopStoch$, the situation is slightly different, since the group operations here are technically Markov kernels rather than functions.
	However, by Proposition \ref{prop:lifting-via-monad} below, groups in $\Meas$ and $\Top$ can always be regarded as groups in $\Stoch$ and $\TopStoch$ by lifting their operations to deterministic Markov kernels as in Remark \ref{rem:measurable-function-as-markov-kernel}.
	All the concrete examples of groups in $\Stoch$ and $\TopStoch$ that we consider will be obtained in this way.
\end{example}

\begin{example} \label{sec:trivial-group}
	In any Markov category $\C$, the unit $\I$ is always a group, regarded as the \emph{trivial group}.
	This is clear in $\Set$: the trivial group is simply the singleton set $\{\singleton\}$, and there is only one possible choice for each group operation, and the group axioms hold immediately.
	For general $\C$, the same idea holds because $\I$ is a terminal object \citep[Remark 2.3]{fritz2020synthetic}.
	Alternatively, this also follows from the case of $\Set$ by Remark \ref{rem:translating-classical-theory} below.

\end{example}

\begin{definition} \label{def:homomorphism}
	A \emph{homomorphism} in a Markov category $\C$ is a deterministic morphism $\varphi : \HH \to \G$ between groups $\G$ and $\HH$ in $\C$ that satisfies
	\begin{equation} \label{eq:homomorphism-definition}
		\input{homomorphism-definition.tikz}
	\end{equation}
\end{definition}

\begin{example}
	In $\Set$, as well as in $\Meas$ and $\Top$, the condition \eqref{eq:homomorphism-definition} translates as saying
	\[
		\varphi(h) \, \varphi(n) = \varphi(h n)
	\]
	for all $h, n \in \HH$, which recovers the usual definition of a set-theoretic homomorphism.
\end{example}

\begin{remark}
It is straightforward to check that the composition of homomorphisms is again a homomorphism, and that identities are homomorphisms.
In this way, we naturally obtain a category whose objects are groups in $\C$ and whose morphisms are homomorphisms between these.
\end{remark}

\begin{example} \label{ex:trivial-homomorphism}
	For every group $\G$ in any Markov category $\C$, there always exists a unique homomorphism $\I \to \G$, namely the unit operation $\e : \I \to \G$.
	In $\Set$, where $\I$ becomes the singleton set $\{\singleton\}$, this is a standard exercise to show.
	In turn, the same then holds for a general $\C$ by Remark \ref{rem:translating-classical-theory} below.
	In categorical terms, this says that $\I$ is an \emph{initial object} in the category of groups in $\C$.
	(It is also a \emph{terminal object} in this same category.)
\end{example}

\begin{definition} \label{def:action}
	Let $\C$ be a Markov category.
	An \emph{action} of a group $G$ on an object $\X$ in $\C$ (or simply a \emph{$\G$-action}) is a deterministic morphism $\act : \G \otimes \X \to \X$ that satisfies
	\begin{equation} \label{eq:action-of-group-object-definition}
		\input{action-of-group-object-definition.tikz}
	\end{equation}
\end{definition}

\begin{example}
	In $\Set$, denoting $\act(g, x)$ as $g \cdot x$, the first equation in \eqref{eq:action-of-group-object-definition} says that
	\[
		g \cdot (g' \cdot x) = (g g') \cdot x \qquad \text{for all $g, g' \in \G$ and $x \in \X$},
	\]
	so $\act$ is associative.
	The second axiom says
	\[
		e \cdot x = x \qquad \text{for all $x \in \X$},
	\]
	so $\act$ is unital.
	As a result, actions in $\Set$ are precisely group actions in the classical sense.

	In what follows, we will often use the notation $g \cdot x$ to denote group actions in $\Set$ more generally, as well as in $\Meas$ and $\Top$, where this notation also makes sense.
	The specific action this refers to ($\act$ here) will always be clear from context.
\end{example}

\begin{example} \label{ex:trivial-action}
	A basic example that will nevertheless be useful for us is the \emph{trivial action}.
	For any group $\G$ and object $\X$ in a Markov category $\C$, this is defined as follows:
	\[
		\input{trivial-action-definition.tikz}
	\]
	When $\C = \Set$, this becomes simply $g \cdot x = x$ for all $g \in \G$ and $x \in \X$.
	In what follows, we will use the same notation $\triv$ to denote all trivial actions, even though these are technically distinct morphisms for different choices of $\G$ and $\X$.
\end{example}

\begin{example} \label{ex:diagonal-action}
	Let $\act_\X : \G \otimes \X \to \X$ and $\act_\Y : \G \otimes \Y \to \Y$ be actions in a Markov category $\C$.
	Then we always obtain a \emph{diagonal action} defined as follows:
	\begin{equation} \label{eq:diagonal-action-definition-1}
		\input{diagonal-action-definition-1.tikz}
	\end{equation}
	When $\C = \Set$, this may be written as $g \cdot (x, y) = (g \cdot x, g \cdot y)$, where $g \in \G$, $x \in \X$, and $y \in \Y$, and some straightforward algebra shows that this is indeed an action.
	By Remark \ref{rem:translating-classical-theory} below, this then also implies that \eqref{eq:diagonal-action-definition-1} is an action for all $\C$, which is otherwise tedious to prove in terms of string diagrams directly.
\end{example}

\begin{example} \label{ex:group-action-on-itself}
	By comparing Definition \eqref{def:group} and Definition \eqref{def:action}, it can be seen that for any group $\G$ in a Markov category $\C$, its multiplication operation $\mul : \G \otimes \G \to \G$ is always an action.
	In other words, every group acts on itself by left multiplication.
	Additionally, every group acts on itself by the following action, which corresponds to right multiplication by the inverse:
	\[
		\input{right-multiplication-action.tikz}
	\]
	For example, in $\Set$, this is associative because
	\[
		g \cdot (h \cdot n) = n h^{-1} g^{-1} = n(gh)^{-1} = (gh) \cdot n
	\]
	for all $g, h, n \in \G$.
	Unitality is similarly straightforward.
	This in fact also shows that $\rmul$ is an action for all Markov categories $\C$, as we explain in Remark \ref{rem:translating-classical-theory} below.
\end{example}

\begin{remark} \label{rem:translating-classical-theory}
	All the morphisms involved in the previous definitions are deterministic. %
	At a high level, we do this in order to ensure that all the familiar results from set-theoretic group theory apply in our context also.
	In more technical terms (which are not necessary to understand in detail), we have defined groups, homomorphisms, and actions internal to the subcategory $\Cdet$ of deterministic morphisms in $\C$.
	Since $\Cdet$ is cartesian monoidal \citep[Remark 10.13]{fritz2020synthetic}, equations involving these constructions in $\Set$ can be lifted directly to equations in $\Cdet$ (and hence $\C$) also.
	This is a standard technique that makes use of the Yoneda embedding: the idea is that a diagram in $\C$ commutes if and only if its image under the Yoneda embedding does, which reduces to a question about set-theoretic groups, homomorphisms, and actions.
	Note however that this approach relies on the fact that
	\[
		\Cdet(\X, \Y \otimes \Z) \cong \Cdet(\X, \Y) \times \Cdet(\X, \Z),
	\]
	which holds because $\Cdet$ is cartesian monoidal.
	The corresponding statement is not always true for $\C$ itself, and more care is required when reasoning by analogy with $\Set$ in that case.

	As an example of this technique, recall that for all groups $\G$ in $\Set$ and all actions $\act : \G \otimes \X \to \X$ of $\G$, where $\X$ may be any set, we have
	\begin{equation} \label{eq:action-by-inverse-is-identity}
		g \cdot (g^{-1} \cdot x) = (g g^{-1}) \cdot x = e \cdot x = x \qquad \text{for all $g \in \G$ and $x \in \X$.}
	\end{equation}
	By applying the Yoneda embedding in this way, the corresponding equation in $\Cdet$ then follows automatically, namely:
	\begin{equation} \label{eq:action-by-inverse-is-identity-in-C}
		\input{action-of-inverse-is-identity.tikz}
	\end{equation}
	We will often use this technique to simplify our presentation, providing set-theoretic justifications like \eqref{eq:action-by-inverse-is-identity} for equations like \eqref{eq:action-by-inverse-is-identity-in-C} that hold in a general Markov category.
	If this appears too opaque, then our set-theoretic proofs may alternatively be thought of as a convenient shorthand for the ``real'' proofs in terms of string diagrams (which can be more tedious to write out).\footnote{The cartesian monoidal structure of $\Cdet$ in fact allows this correspondence to be made rigorous also. See e.g.\ \citet{leinster2008doing} for a discussion.}
	For example, \eqref{eq:action-by-inverse-is-identity} translates directly to string diagrams as follows:
	\[
		\input{string-diagrams-example-proof.tikz}
	\]
	Here we apply the associativity axiom from Definition \ref{def:action}, then the inversion axiom from Definition \ref{def:group}, and finally the unital axiom from Definition \ref{def:action}.
\end{remark}

\subsection{Equivariance and invariance}

The previous definitions suggest an obvious notion of equivariance given next.
Notice that, whereas before we defined everything in terms of deterministic morphisms, the notion of equivariance here now applies to all morphisms in $\C$.

\begin{definition} \label{def:equivariance}
	Suppose $\act_\X : \G \otimes \X \to \X$ and $\act_\Y : \G \otimes \Y \to \Y$ are actions in a Markov category $\C$.
	A morphism $k : \X \to \Y$ in $\C$ is \emph{$\G$-equivariant} (or simply \emph{equivariant}) with respect to $\act_\X$ and $\act_\Y$ if it holds that
	\begin{equation} \label{eq:equivariance-definition}
		\input{equivariance-definition.tikz}
	\end{equation}
	We will say that $k$ is \emph{invariant} if this holds when $\act_\Y$ is the trivial action $\triv$.
\end{definition}

\begin{example}
	In $\Set$, the morphism $k$ is just a function $\X \to \Y$, and $k$ is equivariant if
	\[
		k(g \cdot x) = g \cdot k(x)
	\]
	for all $g \in \G$ and $x \in \X$.
	For invariance, this becomes instead $k(g \cdot x) = k(x)$.
	In both cases, we recover the classical definitions.
\end{example}

\begin{example} \label{ex:stochastic-equivariance}
	Importantly for our purposes, Definition \ref{def:equivariance} also applies directly to the stochastic setting.
	In $\Stoch$ (for example), the condition \eqref{eq:equivariance-definition} says that
	\[
		\int k(B|x') \, \act_\X(dx'|g, x) = \int \act_Y(B|g, y) \, k(dy|x)
	\]
 	for all measurable $B \subseteq \Y$, $g \in \G$, and $x \in \X$.
	Our primary case of interest obtains $\act_\X$ and $\act_\Y$ by lifting the actions of some group in $\Meas$ via Proposition \ref{prop:lifting-via-monad} below.
	In this case, the condition \eqref{eq:equivariance-definition} becomes
	\[
		k(dy|g \cdot x) = g \cdot k(dy|x)
	\]
	for all $g \in \G$ and $x \in \X$, where $g \cdot k(dy|x)$ denotes the pushforward of $k(dy|x)$ by the function $y \mapsto g \cdot y$.
	This recovers the original equivariance condition \eqref{eq:equivariance-stochastic-markov-desideratum} that we gave for stochastic neural networks at the beginning of the paper.
\end{example}

In the machine learning literature, stochastic equivariance is often defined in a slightly different way.
In particular, a conditional \emph{density} $p(y|x)$ is said to be equivariant if 
\begin{equation} \label{eq:stochastic-equivariance-density-definition}
	p(g \cdot y | g \cdot x) = p(y|x)
\end{equation}
for all $g \in \G$, $x \in \X$, and $y \in \Y$ (see e.g.\ Proposition 2.1 of \citet{xu2022geodiff}).
This is a special case of Definition \ref{def:equivariance}, as the following result shows.

\begin{proposition} \label{prop:stochastic-equivariance-density}
	Let $p(y|x)$ be conditional density given $x \in \X$ with respect to some base measure $\mu$ on $\Y$, and denote by $k : \X \to \Y$ the Markov kernel this induces, namely
	\[
		k(B|x) \coloneqq \int_B p(y|x) \, \mu(dy)
	\]
 	where $B \subseteq \Y$ is measurable and $x \in \X$.
	Suppose $\G$ is a group acting on $\X$ and $\Y$ in $\Meas$ such that \eqref{eq:stochastic-equivariance-density-definition} holds, and that moreover $g \cdot \mu = \mu$ for all $g \in \G$, where $g \cdot \mu$ denotes the pushforward of $\mu$ by the function $y \mapsto g \cdot y$.
	Then $k$ is equivariant in the sense of Definition \ref{def:equivariance}, where $\act_\X$ and $\act_\Y$ are obtained by lifting the actions of $\G$ to become Markov kernels via Proposition \ref{prop:lifting-via-monad} below. 
\end{proposition}

\begin{proof}
	See Section \ref{sec:proof-stochastic-equivariance-density} of the Appendix.
\end{proof}

The condition $g \cdot \mu = \mu$ used here is often not made explicit in the machine learning literature.
This is automatically satisfied in many cases that have appeared (e.g.\ \citet{xu2022geodiff,hoogeboom2022equivariant}), which take $\mu$ to be the Lebesgue measure and $\G$ to be some group of transformations with unit Jacobian, such as the orthogonal group.
However, this condition may not hold for more general group actions such as scale transformations, in which case equivariance in the sense of Definition \ref{def:equivariance} need not follow from \eqref{eq:stochastic-equivariance-density-definition}.

\subsection{Compactness}

Compact groups serve as a particularly useful class of groups in many applications.
For practical purposes, the most important feature of these is that they admit a \emph{Haar measure}, which is thought of as a uniform distribution over elements of the group.
We use this as the basis for our definition of compactness in a general Markov category.

\begin{definition} \label{def:compact-group}
	A group $\G$ in a Markov category $\C$ is \emph{compact} if there exists a morphism $\lambda : \I \to \G$ in $\C$, referred to as its \emph{Haar measure}, that satisfies
	\begin{equation} \label{eq:haar-measure-definition}
		\input{haar-measure-definition.tikz}
	\end{equation}
\end{definition}

\begin{example}
	Let $\G$ be a group in $\Meas$, and lift this to a group in $\Stoch$ in the usual way by Proposition \ref{prop:lifting-via-monad} below.
	The condition \eqref{eq:haar-measure-definition} says that if $\bm{G} \sim \haar$, then
	\[
		g\bm{G} \eqd \bm{G}
	\]
	for all $g \in \G$. %
	This recovers the usual definition of a (normalised) Haar measure.
	A standard result shows that a unique such $\haar$ exists for every compact second-countable Hausdorff topological group \citep[Theorem 2.27]{kallenberg1997foundations}.
	In particular, this is true for all finite groups, as well as the classical compact matrix groups (such as the orthogonal group).
\end{example}

\subsection{Semidirect products}

Many groups of interest in applications arise as semidirect products of simpler groups.
The following definition allows us to consider these in an abstract Markov category without requiring further assumptions.
We give the definition before showing how it corresponds to the usual set-theoretic definition in Example \ref{ex:outer-semidirect-product-in-set} below.

\begin{definition} \label{def:semidirect-product-definition}
	Let $\C$ be a Markov category, $\N$ and $\HH$ groups in $\C$, and $\actr : \HH \otimes \N \to \N$ some action of $\HH$ on $\N$ such that
	\begin{equation} \label{eq:semidirect-product-conditions}
		\input{equivariant-group-mul-1.tikz}
	\end{equation}
	The \emph{semidirect product} of $\N$ and $\HH$, denoted $\N \rtimes_\actr \HH$, is the group in $\C$ whose underlying object is $\N \otimes \HH$, and whose group operations $\mul$, $\e$, and $\inv$ are defined respectively as
	\begin{equation} \label{eq:semidirect-product-group-operations}
		\input{semidirect-product-structure-maps-minimal.tikz}
	\end{equation}
\end{definition}

Technically, we should check that these operations always satisfies the group axioms.
By Remark \ref{rem:translating-classical-theory}, since all the morphisms involved are deterministic, it suffices to do so in $\Set$.
The following example shows how Definition \ref{def:semidirect-product-definition} in $\Set$ corresponds to the usual set-theoretic definition of outer semidirect products, after which this becomes standard to prove.

\begin{example} \label{ex:outer-semidirect-product-in-set}
	In $\Set$, an (outer) semidirect product is usually defined in terms of a homomorphism $\varphi : \HH \to \Aut(\N)$, $h \mapsto \varphi_h$, where $\Aut(\N)$ denotes the automorphism group of $\N$.
	(Recall that this is the group of invertible homomorphisms $\N \to \N$ equipped with function composition as its multiplication operation.)
	The idea behind Definition \ref{def:semidirect-product-definition} is that every such $\varphi$ canonically gives rise to an action $\actr$ that satisfies \eqref{eq:semidirect-product-conditions} and vice versa.
	For example, in one direction we may uncurry $\varphi$ to obtain
	\begin{equation} \label{eq:semidirect-product-action-from-varphi}
 		\actr(h, n) \coloneqq \varphi_h(n)
	\end{equation}
	for $h \in \HH$ and $n \in \N$.
	That $\actr$ is indeed an action follows from the fact that $\varphi$ is a homomorphism, which means $\varphi_{e_\HH} = \id_\N$ and $\varphi_h \circ \varphi_{h'} = \varphi_{h h'}$, and moreover $\actr$ satisfies \eqref{eq:semidirect-product-conditions} because each $\varphi_h : \N \to \N$ is itself a homomorphism.
	Conversely, by flipping \eqref{eq:semidirect-product-action-from-varphi} left-to-right, each action $\actr$ that satisfies \eqref{eq:semidirect-product-conditions} also defines a homomorphism $\varphi : \HH \to \Aut(\N)$, as may be checked.
	Under this correspondence, the group operations \eqref{eq:semidirect-product-group-operations} then also correspond to their usual set-theoretic definitions.
	For example, multiplication may be written as
	\[
		(n, h) \, (n', h') = (n \, \actr(h, n'), h h')
	\]
	where $n, n' \in \N$ and $h, h' \in \HH$, which recovers the usual definition by replacing $\actr(h, n')$ with $\varphi_h(n')$.
	The usual argument for outer semidirect products in $\Set$ now translates directly to show that the group axioms hold, although we omit the details.

\end{example}

\begin{example} \label{ex:euclidean-groups}

	The Euclidean groups $\Euc(d)$ and $\SE(d)$ are well-known examples of semidirect products.
	We show how these can be expressed as instances of Definition \ref{def:semidirect-product-definition} in $\Set$ (although this applies equally in $\Meas$ and $\Top$).
	Let $\T^d$ be the group of translations, or in other words $\R^d$ equipped with vector addition, and let $\Orth(d)$ be the orthogonal group.
	We then obtain an action $\actr : \Orth(d) \otimes \T^d \to \T^d$ in the obvious way with $Q \cdot t \coloneqq Q t$.
	The condition \eqref{eq:semidirect-product-conditions} becomes $Q \cdot (t + t') = Q \cdot t + Q \cdot t'$, which is clearly satisfied.
	This allows us to obtain the Euclidean group as follows:
	\[
	\Euc(d) \coloneqq \T_d \rtimes_\actr \Orth(d).
	\]
	For example, the group multiplication in \eqref{eq:semidirect-product-group-operations} becomes
	\[
		(t, Q) \, (t', Q') = (t + Q t', Q Q'),
	\]
	which recovers the usual definition.
	By substituting the special orthogonal group $\SO(d)$ for $\Orth(d)$, we also obtain the special Euclidean group $\SE(d)$ in the same way.
\end{example}

\begin{remark}
	The condition \eqref{eq:semidirect-product-conditions} amounts to saying that $\N$ is a group in $\C$ whose multiplication operation is $\HH$-equivariant, where $\HH$ acts on $\N$ diagonally as in Example \ref{ex:diagonal-action}.
	For example, in $\Set$, writing $\mul_\N$ as $m$, this condition becomes
	\[
		m(h \cdot n, h \cdot n') = h \cdot m(n, n')
	\]
 	for all $h \in \HH$ and $n, n' \in \N$.
	By some straightforward manipulations, this in turn implies that the inversion and identity operations of $\N$ are also $\HH$-equivariant in a particular sense: we obtain $h \cdot n^{-1} = (h \cdot n)^{-1}$ and $h \cdot e_\N = e_\N$ in $\Set$, and similar equations more generally.
	This gives rise to a particularly streamlined way to describe semidirect products: in terminology from Definition \ref{def:markov-category-of-equivariant-morphisms} below, each semidirect product $\N \rtimes_\actr \HH$ in $\C$ corresponds to a group $(\N, \actr)$ in the Markov category $\C^\HH$ of $\HH$-equivariant morphisms, and vice versa.
	This is known for example in the context of algebraic topology (see Lemma 2.1.4 of \citet{sati2022equivariant}), and shows that Definition \ref{def:semidirect-product-definition} arises very naturally, despite its apparently complex structure.

\end{remark}

A familiar result says that actions of a classical set-theoretic semidirect product $\N \rtimes_\actr \HH$ can always be computed as some $\HH$-action followed by some $\N$-action.
This translates directly to Markov categories as follows.

\begin{proposition} \label{prop:semidirect-product-actions}
	Let $\C$ be a Markov category.
	Every action of a semidirect product $\N \rtimes_{\actr} \HH$ on an object $\X$ in $\C$ can be written in the form
	\begin{equation} \label{eq:semidirect-product-action-correspondence}
		\input{semidirect-product-action-correspondence.tikz}
	\end{equation}
	where $\act_\HH : \HH \otimes \X \to \X$ and $\act_\N : \N \otimes \X \to \X$ are actions that satisfy
	\[
		\input{semidirect-product-action-condition.tikz}
	\]
	Conversely, for any actions $\act_\HH$ and $\act_\N$ in $\C$ with these properties, \eqref{eq:semidirect-product-action-correspondence} defines an action of $\N \rtimes_{\actr} \HH$ on $\X$.
\end{proposition}

\begin{proof}
	After applying the correspondence from Example \ref{ex:outer-semidirect-product-in-set}, this becomes standard to show in $\Set$.
	The result then applies more generally by Remark \ref{rem:translating-classical-theory}.
\end{proof}

\begin{example} 
	Continuing Example \ref{ex:euclidean-groups}, suppose we now have an action of $\SE(d) = \T_d \rtimes \SO(d)$ on an arbitrary object in $\C$.
	Proposition \ref{prop:semidirect-product-actions} says this can always be computed as an action of $\SO(d)$ followed by an action of $\T_d$.
	Adopting the usual interpretation of $\SO(d)$ as rotation matrices, this amounts to a rotation followed by a translation, which recovers the usual result.
\end{example}

As for the classical case, equivariance with respect to the action of a semidirect product is equivalent to equivariance with respect to its two induced actions separately.

\begin{proposition} \label{prop:semidirect-product-equivariance}
	Let $\act_\X$ and $\act_\Y$ be actions of a semidirect product $\N \rtimes_\actr \K$ on objects $\X$ and $\Y$ in a Markov category $\C$.
	Then $k : \X \to \Y$ in $\C$ is equivariant with respect to $\act_\X$ and $\act_\Y$ if and only if it is equivariant with respect to both 1) $\act_{\X,\HH} $ and $\act_{\Y,\HH}$, and 2) $\act_{\X,\N}$ and $\act_{\Y,\N}$, where here $\act_{\X,\HH}$ and $\act_{\X,\N}$ denote the $\HH$- and $\N$-actions obtained from $\act_\X$ via Proposition \ref{prop:semidirect-product-actions}, and similarly for $\act_\Y$.
\end{proposition}

\begin{proof}
	We sketch the main idea.
	That $\N$- and $\K$-equivariance implies $(\N \rtimes_\actr \K)$-equivariance follows immediately from two applications of the definition of equivariance.
	Conversely, since group actions are unital, every action of the form \eqref{eq:semidirect-product-action-correspondence} becomes an $\N$-action when $e_\K$ is attached its $\K$-input, and a $\K$-action when $e_\N$ is attached to its $\N$-input.
	As such, $(\N \rtimes_\actr \K)$-equivariance implies both $\N$- and $\K$-equivariance.
\end{proof}

\subsection{Direct products}

An important special case of semidirect products is the direct product.
We give this its own definition as follows.

\begin{definition} \label{def:direct-product}
	Let $\G$ and $\HH$ be groups in a Markov category $\C$.
	The \emph{direct product} of $\G$ and $\HH$, denoted $\G \times \HH$, is the group in $\C$ consisting of $\G \otimes \HH$ equipped with group operations $\mul$, $\e$, and $\inv$ defined respectively as
	\[
		\input{direct-product-operations.tikz}
	\]
\end{definition}

\begin{remark} \label{rem:direct-product-as-semidirect-product}
	As classically, direct products can be regarded as semidirect products where the action of $\HH$ on $\G$ is trivial.
	That is, $\G \times \HH = \G \rtimes_\triv \HH$.
	The equivariance condition \eqref{eq:semidirect-product-conditions} is then immediate, which shows $\G \times \HH$ is indeed always a group since $\G \rtimes_\triv \HH$ is.
\end{remark}

\begin{example}
	Suppose $\G$ and $\HH$ are groups in $\Set$.
	Then $\G \times \HH$ consists of pairs $(g, h)$ with $g \in \G$ and $h \in \HH$.
	Multiplication, identities, and inversions are all computed componentwise, so that for example
	\[
		(g, h) \, (g', h') = (g g', h h')
	\]
	where $g, g' \in \G$ and $h, h' \in \HH$, which recovers the usual definition.
\end{example}

\begin{remark} \label{rem:direct-product-actions}
	By specialising Proposition \ref{prop:semidirect-product-actions} to this context, it holds that actions of a direct product $\G \times \HH$ on an object $\X$ correspond bijectively to pairs of actions $\act_\G : \G \otimes \X \to \X$ and $\act_\HH : \HH \otimes \X \to \X$ such that
	\[
		\input{direct-product-action-condition.tikz}
	\]
	That is, an action of $\G \times \HH$ is simply a pair of an $\G$- and a $\HH$-action that commute with each other.
	Likewise, by specialising Proposition \ref{prop:semidirect-product-equivariance}, it holds that equivariance with respect to $\G \times \HH$ is equivalent to equivariance with respect to $\G$ and $\HH$ separately.
\end{remark}

\subsection{Orbits} \label{sec:orbits}

A classical result from group theory is that the action of a group $\G$ on a set $\X$ always induces an equivalence relation on $\X$, with $x \sim x'$ when $x = g \cdot x'$ for some $g \in \G$.
The equivalence classes $[x]$ obtained in this way are referred to as \emph{orbits}.
The following definition allows us to talk about orbits in the context of Markov categories also.

\begin{definition} \label{def:orbit-map}
	Let $\act : \G \otimes \X \to \X$ be an action in a Markov category $\C$.
	An \emph{orbit map} is a deterministic morphism $q : \X \to \X/\G$, where $\X/\G$ may be any object of $\C$, such that the following is a coequaliser diagram in $\C$
	\begin{equation} \label{eq:orbit-map-coequaliser}
		\begin{tikzcd}%
			\G \otimes \X \ar[shift left=1.5]{r}{\act} \ar[shift right=1.5,swap]{r}{\triv} & \X \ar{r}{q} & \X/\G
		\end{tikzcd}
	\end{equation}
	that is moreover preserved by the functor $(-) \otimes \Y$ for every $\Y$ in $\C$.
\end{definition}

Coequalisers are a standard construction in category theory, and are often used as a way to talk about \emph{quotients} in an abstract setting.
We provide the main idea here.
For \eqref{eq:orbit-map-coequaliser} to be a coequaliser diagram, two conditions must hold.
First, $q$ must satisfy $q \circ \act = q \circ \triv$, where recall that $\triv$ is the trivial action from Example \ref{ex:trivial-action}.
In string diagrams this becomes
\[
	\input{orbit-map-invariance.tikz}
\]
which just says that $q$ is \emph{invariant} with respect to $\act$.
In addition, $q$ must be \emph{universal} with this property in the sense that if $k : \X \to \Z$ is any other invariant morphism in $\C$, then there is a unique morphism $k/\G : \X/\G \to \Z$ such that the following diagram commutes:
\[
	\begin{tikzcd}
		\X \ar{r}{q} \ar[swap]{dr}{k} & \X/\G \ar[dashed]{d}{k/\G} \\
		& \Z
	\end{tikzcd}
\]
Intuitively, this means that $q$ is the most information-preserving invariant morphism out of $\X$, since any other invariant morphism can be computed in some way in terms of its output.
In practice, this amounts to saying that $q$ takes a unique value on each orbit induced by the action $\act$ (see e.g.\ Example \ref{ex:orbit-map-example} and Remark \ref{rem:orbit-map-need-not-be-equivalence-classes} below).
For a more detailed discussion of coequalisers, we refer the reader to Section 3.4.2 of \citet{perrone2021notes}.

\begin{remark}
	The requirement that $q$ is deterministic is a fairly natural one, as the examples below will show.
	Our main technical reason for imposing it is for the proof of Proposition \ref{prop:induced-action-on-cosets} below.
	In practice, determinism is not a major restriction.
	For example, if $\C$ is positive and \eqref{eq:orbit-map-coequaliser} is a coequaliser diagram, then $q$ is automatically deterministic, and so this does not need to be checked separately.
	This is true more generally for all coequalisers of deterministic morphisms: see Proposition \ref{prop:coequalisers-in-deterministic-category} in the Appendix.
\end{remark}

\begin{remark} \label{rem:product-of-orbit-maps-is-orbit-map}
	We briefly comment on the preservation condition in Definition \ref{def:orbit-map}.
	For practical purposes, this may be understood as a mild technical requirement that ensures parallel compositions of orbit maps are well-behaved.
	For various $\C$ of interest we can even avoid checking this condition altogether (see Remark \ref{rem:orbit-map-existence-implies-preservation} below).
	Its key implication is that orbit maps are closed under parallel composition with identity morphisms (which are themselves always orbit maps by Example \ref{ex:trivial-orbit-map} below).
	In other words, if $q$ is an orbit map with respect to $\act : \G \otimes \X \to \X$, then $q \otimes \id_\Y$ is also an orbit map, where $\G$ acts on $\X$ via $\act$ and on $\Y$ trivially.
	See Proposition \ref{prop:orbit-map-preserved-action} in the Appendix for a proof.
\end{remark}

\begin{example} \label{ex:orbit-map-example}
	Consider what Definition \ref{def:orbit-map} means in $\Set$.
	Let $\X/\G \coloneqq \{[x] \mid x \in \X\}$ be the set of orbits induced by $\act$, and let $q : \X \to \X/\G$, $x \mapsto [x]$ map each element to its orbit.
	We claim that $q$ is an orbit map.
	Certainly it is deterministic.
	We need to show that it is also a coequaliser, or in other words that it is invariant with respect to $\act$, and universal.
	Invariance is straightforward since
	\[
		q(g \cdot x) = [g \cdot x] = [x] = q(x)
	\]
	by the definition of orbits.
	For universality, if $f$ is any other invariant function, then $f/\G : \X/\G \to \Y$ defined as $[x] \mapsto f(x)$ is well-defined and invariant, and moreover
	\[
		f(x) = (f/\G)([x]) = (f/G)(q(x)),
	\]
	which shows $f = f/\G \circ q$.
	Moreover, since $q$ is surjective, $f/\G$ is the only function that can satisfy this, and so it follows that $q$ is a coequaliser.

	Finally, to be an orbit map, we must show that $q$ is preserved by every functor $(-) \otimes \Y$.
	Recall that this functor sends each object $\Z$ in $\C$ to $\Z \otimes \Y$, and each morphism $f$ to $f \otimes \id_\Y$.
	The image of \eqref{eq:orbit-map-coequaliser} under this functor is therefore
	\begin{equation}
		\begin{tikzcd}[column sep=3em]
			(\G \otimes \X) \otimes \Y \ar[shift left=1.5]{r}{\act \otimes \id_\Y} \ar[shift right=1.5,swap]{r}{\triv \otimes \id_\Y} & \X \otimes \Y \ar{r}{q \otimes \id_\Y} & \X/\G \otimes \Y
		\end{tikzcd}
	\end{equation}
	and we want to show that this is also a coequaliser diagram.
	This can be done using a similar argument as we just gave.
	(In fact, what we showed is really a special case of this more general statement with $\Y \coloneqq \I$ the singleton set.)
	As a result, $q$ is an orbit map.
\end{example}

\begin{example} \label{ex:trivial-orbit-map}
	For any group $\G$ activing trivially on an object $\X$ in a Markov category $\C$, the identity $\id_\X$ is always an orbit map.
	Indeed, it is easily checked that the following is a coequaliser diagram:
	\begin{equation} \label{eq:trivial-orbit-map}
		\begin{tikzcd}
			\G \otimes \X \ar[shift left=1.5]{r}{\triv} \ar[shift right=1.5,swap]{r}{\triv} & \X \ar{r}{\id_\X} & \X
		\end{tikzcd}
	\end{equation}
	Moreover, this is preserved by every functor $(-) \otimes \Y$ since we have $\id_{\X} \otimes \id_{\Y} = \id_{\X \otimes \Y}$. %
	Although trivial, this example will nevertheless be useful for us in what follows.
\end{example}

\begin{remark} \label{rem:orbit-map-need-not-be-equivalence-classes}
	Often there will exist more than one possible choice of orbit map for the same group action.
	For example, consider the rotation group $\G \coloneqq \SO(2)$ acting on the plane $\R^2$ in $\Set$.
	By checking the various requirements, it follows that the function $r : \R^2 \to [0, \infty)$, $x \mapsto \norm{x}$ is an orbit map.
	This is a different construction to the orbit map $q : \X \to \X/\G$ from Example \ref{ex:orbit-map-example}, where $\X/\G$ was defined as a set of equivalence classes.
	However, in another sense, these two orbit maps are equivalent: a straightforward argument shows that there exists a unique bijection $\X/\G \to [0, \infty)$ such that the following commutes:
	\[
		\begin{tikzcd}
			X \ar{r}{q} \ar[swap]{rd}{r} & \X/\G \ar[dashed]{d}{\cong} \\
			& \left[0, \infty\right)
		\end{tikzcd}
	\]
	In this way, $r$ and $q$ describe the same information, just with different ``encodings''.
	More generally, as a consequence of the Yoneda Lemma, if $q$ and $r$ are orbit maps with respect to the same group action in a Markov category $\C$, then there always exists a unique isomorphism $\phi$ such that $r = \phi \circ q$ holds.
	For practical purposes, this means we are free to choose whichever orbit map is most convenient, and will not sacrifice generality in doing so.
\end{remark}

It is not guaranteed that a Markov category will admit all orbit maps: there may exist some $\act$ for which a coequaliser diagram \eqref{eq:orbit-map-coequaliser} does not exist, or is not preserved by every functor $(-) \otimes \Y$.
We therefore consider examples in which orbit maps are always available.

\begin{example} \label{ex:set-has-orbits}
	$\Set$ admits all orbit maps.
	This is shown directly by Example \ref{ex:orbit-map-example}.
	More abstractly, this can be seen by noting that $\Set$ has coequalisers and is cartesian closed, which means $(-) \otimes \Y$ is a left adjoint and hence cocontinuous (see e.g.\ Corollary 4.3.2 of \citet{perrone2021notes}).
	The same idea applies more generally: a Markov category $\C$ admits all orbit maps if it has coequalisers (or even reflexive coequalisers) and is monoidal closed.
\end{example}

\begin{example} \label{ex:top-has-orbits}
	$\Top$ admits all orbit maps.
	Unlike with $\Set$, a quick abstract proof is less obvious to us here, but a standard lower-level argument can be used instead.
	For completeness, we include this in Proposition \ref{prop:top-has-orbits} in the Appendix.
\end{example}

In both $\Set$ and $\Top$, every morphism is deterministic.
For reasoning about stochastic neural networks, we would like a Markov category that admits all orbit maps, but also more general morphisms.
The following provides a convenient and very general example.

\begin{theorem} \label{thm:topstoch-has-orbits} 
	Every action in $\TopStoch$ admits an orbit map.
\end{theorem}

\begin{proof}
	See Section \ref{sec:top-has-orbits} of the Appendix.
\end{proof}

We do not know whether $\Meas$ or $\Stoch$ admit all orbit maps.
In $\Meas$, it is true that every action gives rise to a coequaliser \eqref{eq:orbit-map-coequaliser} obtained from the usual measure-theoretic quotient space. %
A similar statement also holds in $\Stoch$ as a consequence of Proposition 3.7 of \citet{moss2023category}.
However, in both cases, we were not able to prove that these coequalisers are always preserved by the functor $(-) \otimes \Y$, which our theory below requires.

\begin{remark} \label{rem:orbit-map-existence-implies-preservation}
	For our symmetrisation methodology below, we will need to construct orbit maps more explicitly than these existence results provide.
	However, it is still practically useful to know that some orbit map does indeed exist.
	A straightforward argument shows that if an action $\act$ admits some orbit map, then every coequaliser of the form \eqref{eq:orbit-map-coequaliser} automatically satisfies the preservation condition and is therefore an orbit map.
	See Proposition \ref{prop:orbit-map-from-coequaliser} in the Appendix.
	As such, to show a morphism is an orbit map in $\TopStoch$ (for example), we just need to establish the coequaliser condition: the preservation condition does not need to be checked separately.

\end{remark}

\subsection{Cosets} \label{sec:coset-maps}

In set-theoretic group theory, a subgroup $\HH$ of a group $\G$ gives rise to a set of (left) cosets, whose elements have the form $g\HH \coloneqq \{gh \mid h \in \HH\}$ where $g \in \G$. %
The following definition allows us to talk about cosets in general Markov categories also.

\begin{definition} \label{def:coset-map}
	Let $\varphi : \HH \to \G$ be a homomorphism in a Markov category $\C$.
	A \emph{$\varphi$-coset map} is an orbit map with respect to the following action of $\HH$ on $\G$:
	\begin{equation} \label{eq:rmul-definition}
		\input{coset-map-definition-action.tikz}
	\end{equation}
\end{definition}

In terminology introduced in Definition \ref{def:pullback-action-definition} below, the morphism \eqref{eq:rmul-definition} is just the restriction via $\varphi$ of the action $\rmul$ from Example \ref{ex:group-action-on-itself}, and is therefore indeed always an action.

\begin{example} \label{ex:coset-map-example}
	In $\Set$, Definition \ref{def:coset-map} recovers cosets in the usual sense by letting $\varphi : \HH \to \G$ be a subgroup inclusion, so that $\varphi(h) = h$.
	The action \eqref{eq:rmul-definition} becomes
	\[
		h \cdot g = g \, \varphi(h)^{-1} = gh^{-1}
	\]
 	for $h \in \HH$ and $g \in \G$.
	The cosets of $\HH$ in $\G$ are then exactly the orbits obtained from this action, since we have
	\[
		gH = \{gh \mid h \in H\} = \{gh^{-1} \mid h \in H\} = \{h \cdot g \mid h \in H\},
	\]
	and the right-hand side is seen to be the orbit of $g$.
	Letting $\G/\HH$ be the set of these cosets, it now follows from Example \ref{ex:orbit-map-example} that the function $\G \to \G/\HH$ defined as $g \mapsto g\HH$ is an orbit map, and hence a $\varphi$-coset map.
\end{example}

\begin{remark}
	We will often use the same suggestive notation $\G/\HH$ from Example \ref{ex:coset-map-example} to denote the codomain of a $\varphi$-coset map, where $\varphi : \HH \to \G$ is a homomorphism in a general Markov category $\C$.
	However, we emphasise that this is not meant to imply $\G/\HH$ is a quotient \emph{group} in any sense.
	In particular, recall that $\G/\HH$ from Example \ref{ex:coset-map-example} does not inherit a group structure unless $\HH$ is a \emph{normal} subgroup.
	Instead, $\G/\HH$ should be thought of simply as a space of orbits (or more specifically, cosets) without associated group operations.
	This makes sense to consider regardless of any additional properties $\HH$ may have.
\end{remark}

\begin{example} \label{ex:trivial-coset-map}
	Let $\G$ be a group in a Markov category $\C$, and $\varphi : \I \to \G$ the unique homomorphism out of the trivial group from Example \ref{sec:trivial-group}.
	Then \eqref{eq:rmul-definition} becomes trivial.
	By Example \ref{ex:trivial-orbit-map}, it follows that $\id_\G : \G \to \G$ is an orbit map, and hence a $\varphi$-coset map.
\end{example}

\begin{example} \label{ex:semidirect-product-coset-maps}
	Let $\N \rtimes_\actr \HH$ be a semidirect product in a Markov category $\C$.
	Recall from Definition \ref{def:semidirect-product-definition} that this has $\N \otimes \HH$ as its underlying object.
	We therefore always obtain inclusion and projection morphisms
	\[
		\begin{tikzcd}[column sep=4em]
				\N \ar[shift left=1.5]{r}{i_\N} & \N \rtimes_\actr \HH \ar[swap,shift left=1.5,swap]{l}{p_\N} \ar[shift right=1.5,swap]{r}{p_\HH} & \HH \ar[shift right=1.5,swap]{l}{i_\HH} 
		\end{tikzcd}
	\]
	where for example
	\[
		\input{semidirect-product-inclusions.tikz}
	\]
	and $i_\HH$ and $p_\HH$ are similar.
	By a straightforward argument, both $i_\N$ and $i_\HH$ are homomorphisms. %
	Moreover, $p_\HH$ is a $i_\N$-coset map, and $p_\N$ is a $i_\HH$-coset map.
	See Proposition \ref{prop:semidirect-product-coset-maps} in the Appendix.
	Heuristically, this says that $(\N \rtimes_\actr \HH) / \N \cong \HH$ and $(\N \rtimes_\actr \HH) / \HH \cong \N$, so that after ``quotienting out'' one factor, we are left with the other. 
\end{example}

\begin{example} \label{ex:glnr-coset-map}
	Let $\varphi : \Orth(d) \to \GL(d, \R)$ be the inclusion homomorphism of the orthogonal group into the general linear group in $\Set$ (for example).
	Let $\PD(d)$ denote the set of $d$-dimensional positive-definite matrices, and define $q : \GL(d, \R) \to \PD(d)$ by $A \mapsto A A^T$.
	Then $q$ is a $\varphi$-coset map.
	Indeed, this is invariant because
	\[
		q(AQ^{-1}) = AQ^TQA^T = AA^T
	\]
	whenever $Q$ is orthogonal.
	Now suppose $r$ is also invariant.
	For $A \in \GL(d, \R)$ a standard matrix decomposition allows us always to write $A = RQ$ for some orthogonal $Q$ and upper triangular $R$, which yields
	\[
		r(A) = r(RQ) = r(R) = r(\mathrm{Chol}(RR^T)) = r(\mathrm{Chol}(AA^T)) = r(\mathrm{Chol}(q(A))),
	\]
	where $\mathrm{Chol}$ denotes the Cholesky decomposition.
	This shows that $r(A)$ can be expressed as a function of $q(A)$, and this function must be unique since $q$ is surjective.
	By Remark \ref{rem:orbit-map-existence-implies-preservation}, the preservation condition now holds automatically, and so $q$ is a $\varphi$-coset map.
\end{example}

\subsection{Coset actions}

If $\varphi : \HH \to \G$ is a homomorphism, then by definition every $\varphi$-coset map is $\HH$-invariant.
However, importantly, a $\varphi$-coset map also becomes $\G$-equivariant in a canonical way.

\begin{proposition} \label{prop:induced-action-on-cosets}
	Let $q : \G \to \G/\HH$ be a $\varphi$-coset map.
	There exists a unique action $\mul/\HH : \G \otimes \G/\HH \to \G/\HH$ that makes $q$ equivariant in the sense that
	\[
		\input{induced-equivariance-of-coset-map.tikz}
	\]
\end{proposition}

\begin{proof}
	See Section \ref{sec:proof-of-induced-action-on-cosets} of the Appendix.
\end{proof}

\begin{example}
	Continuing Example \ref{ex:coset-map-example}, the action of $\G$ on $\G/\HH$ is given by
	\[
		g \cdot (g'H) = (gg')H.
	\]
	It is straightforward to check that this is well-defined and satisfies the axioms of an action.
	Moreover,
	\[
		q(gg') = (gg')H = g \cdot (g'H) = g \cdot q(g'),
	\]
	which shows that $q$ is $\G$-equivariant with respect to this action.
\end{example}

\begin{example} \label{ex:trivial-coset-map-action}
	Continuing Example \ref{ex:trivial-coset-map}, the action induced by $\id_\G$ is simply the action of $\G$ on itself by left multiplication.
	This is always an action by Example \ref{ex:group-action-on-itself}, and it holds trivially that $\id_\G$ is equivariant with respect to this action on its domain and codomain.
\end{example}

\begin{example} \label{ex:semidirect-product-coset-maps-induced-actions}
	Continuing Example \ref{ex:semidirect-product-coset-maps} on semidirect products, recall that the projection $p_\HH : \N \rtimes_\actr \HH \to \HH$ is always an $i_\N$-coset map.
	Proposition \ref{prop:induced-action-on-cosets} then induces the following action of $\N \rtimes_\actr \HH$ on $\HH$:
	\begin{equation*} %
		\input{semidirect-product-H-coset-projection-action.tikz}
	\end{equation*}
	Indeed, it can be seen from Proposition \ref{prop:semidirect-product-actions} that this is always an action.
	We therefore only need to show that $p_\HH$ is equivariant with respect to $\mul$ and $\mul/\N$.
	By Remark \ref{rem:translating-classical-theory}, this may be done in $\Set$, where for $n, n' \in \N$ and $h, h' \in \HH$ we have
	\[
		p_\HH((n, h) \, (n', h'))
			= p_\HH(n \, \actr(h, n'), h h')
			= h h'
			= (n, h) \cdot p_\HH(n', h'),
	\]
	where we use the definition of $\mul/\HH$ in the last step.
	By a similar argument, the action on $\N$ induced by the projection $p_\N : \N \rtimes_\actr \HH \to \N$, which is an $i_\HH$-coset map, is given by
	\begin{equation*} %
		\input{semidirect-product-N-coset-projection-action.tikz}
	\end{equation*}
\end{example}

\begin{example} \label{ex:glnr-coset-map-action}
	Continuing Example \ref{ex:glnr-coset-map}, the unique action of $\GL(d, \R)$ on $\PD(d)$ that makes $q$ equivariant is given by
	\[
		A \cdot P = A P A^T,
	\]
	where $A \in \GL(d, \R)$ and $P \in \PD(d)$.
	It is straightforward to check that this indeed satisfies the axioms of an action, and moreover
	\[
		q(AB) = (AB)(AB)^T = ABB^TA^T = A \cdot (BB^T) = A \cdot q(B)
	\]
	for all $A, B \in \GL(d, \R)$, so that $q$ is equivariant with respect to this action.
\end{example}

\subsection{Lifting group constructions} \label{sec:lifting-group-constructions}

It is usually tedious to specify groups, actions, and so on in $\Stoch$ or $\TopStoch$ directly, because the morphisms of these categories are Markov kernels, which are not as straightforward to write down as ordinary functions are.
The following result provides a convenient way around this.

\begin{proposition} \label{prop:lifting-via-monad}
	The lifting operation described in Remark \ref{rem:measurable-function-as-markov-kernel} sends groups, homomorphisms, actions, equivariant morphisms, orbit maps, and sections (see Remark \ref{rem:transpose-via-section} below) from $\Meas$ to $\Stoch$, and from $\Top$ to $\TopStoch$.
\end{proposition}

\begin{proof}
	See Section \ref{sec:proof-of-lifting-via-monad} in the Appendix.
\end{proof}

This allows us to specify these various constructions in terms of functions in $\Meas$ or $\Top$, and then lift these to $\Stoch$ or $\TopStoch$ automatically.
For example, the following may be checked to be an action of the orthogonal group $\Orth(n)$ on $\R^{n \times n}$ in $\Top$, and hence in $\TopStoch$:
\[
	Q \cdot A \coloneqq Q A Q^T.
\]
Similarly, orbit maps always exist for actions defined in $\Top$ (see Example \ref{ex:top-has-orbits}), and these lift directly to $\TopStoch$ also.

\section{Symmetrisation procedures} \label{sec:symmetrisation}

We now formalise the overall problem of symmetrisation in Markov categories.
At a high level, our goal is to construct \emph{symmetrisation procedures}, defined in Section \ref{sec:symmetrisation-procedures}.
Intuitively, these take a morphism that is unconstrained or only ``partially'' equivariant and ``upgrade it'' to become fully equivariant with respect to some specified group actions.
We also consider desirable properties of these procedures and their overall usage in general terms.
Our framework for actually constructing concrete procedures is given in Section \ref{sec:adjunctive-symmetrisation} below.

\subsection{Markov categories of equivariant morphisms}

It is easily checked that the composition of equivariant morphisms is again equivariant, and that the identity morphism is equivariant.
In this way, the equivariant morphisms in $\C$ form a category.
By Proposition \ref{prop:markov-category-of-equivariant-morphisms} in the Appendix, this is in fact also a Markov category, whose components are defined as follows.

\begin{definition} \label{def:markov-category-of-equivariant-morphisms}
	Given a group $\G$ in a Markov category $\C$, we denote by $\C^\G$ the \emph{Markov category of $\G$-equivariant morphisms}, whose components are as follows:
	\begin{itemize}
		\item Objects are pairs $(\X, \act_\X)$ of an object $\X$ and an action $\act_\X : \G \otimes \X \to \X$ in $\C$
		\item Morphisms $(\X, \act_\X) \to (\Y, \act_\Y)$ are morphisms $\X \to \Y$ in $\C$ that are equivariant with respect to $\act_\X$ and $\act_\Y$, with composition and identities inherited from $\C$
		\item The monoidal product $(\X, \act_\X) \otimes (\Y, \act_\Y)$ is $(\X \otimes \Y, \act_{\X \otimes \Y})$, where $\act_{\X \otimes \Y}$ is the diagonal action defined in Example \ref{ex:diagonal-action}
		\item The monoidal unit is $(\I, \triv)$, where $\triv$ is the trivial action defined in Example \ref{ex:trivial-action}
		\item The structure morphisms (e.g.\ copy and discard) are inherited from $\C$
	\end{itemize}
\end{definition}

It can be cumbersome to write objects in $\C^\G$ always as pairs like $(\X, \act_\X)$ and $(\Y, \act_\Y)$.
To streamline notation, we will sometimes denote these simply by $\X$, $\Y$, etc.
We will always make clear when we are doing so.

\begin{example} \label{ex:trivial-equivariant-markov-category}
	It is straightforward to check that in any Markov category $\C$, the only actions of the trivial group $\I$ are trivial.
	Moreover, a morphism $k : \X \to \Y$ in $\C$ is always equivariant with respect to the trivial actions on $\X$ and $\Y$.
	As a result, $\C^\I$ consists of objects of the form $(\X, \triv)$, where $\X$ may be any object in $\C$, and its morphisms $(\X, \triv) \to (\Y, \triv)$  are all the morphisms $\X \to \Y$ in $\C$.
	In this way, we may regard $\C^\I$ and $\C$ as being the same (although more precisely they are isomorphic as categories).
\end{example}

\begin{remark}
	In technical terms (which are not essential for what follows), $\C^\G$ is closely related to the Eilenberg-Moore category of the action monad $\G \otimes (-)$ on $\C$.
	Precisely, it is the full subcategory spanned by the algebras of this monad that are deterministic.
	In the context of deep learning, the action monad in $\Set$ has also featured in \citet{gavranovic2024categorical} (Example 2.6), who proposed using other monads in its place as a way to describe more general architectural properties than group equivariance.
\end{remark}

\subsection{Action restriction} \label{sec:action-restriction}

A homomorphism $\varphi : \HH \to \G$ always gives rise to a functor $\C^\G \to \C^\HH$.
This provides a link between $\G$-equivariant and $\HH$-equivariant morphisms, and so is highly relevant for symmetrisation.
This functor acts on objects via \emph{restriction}, which we define first.

\begin{definition} \label{def:pullback-action-definition}
	Let $\varphi: \HH \to \G$ be a homomorphism in a Markov category $\C$.
	Given an action $\act : \G \otimes \X \to \X$ in $\C$, its \emph{restriction} via $\varphi$ is the following action of $\HH$:
	\begin{equation} \label{eq:restricted-action-definition}
		\input{pullback-action-definition.tikz}
	\end{equation}
\end{definition}

Technically, we should check that \eqref{eq:restricted-action-definition} indeed always defines an action.
The following example establishes this in $\Set$, and hence in general by Remark \ref{rem:translating-classical-theory}.

\begin{example} \label{ex:subgroup-restricted-action}
	In $\Set$, the action \eqref{eq:restricted-action-definition} may be written as
	\begin{equation} \label{eq:restricted-action-definition-set}
		h \cdot x = \varphi(h) \cdot x \qquad \text{where $h \in \HH$ and $x \in \X$}.
	\end{equation}
	Notice that the original $\G$-action appears here on the right-hand side.
	This is associative because
	\[
		\varphi(h) \cdot (\varphi(h') \cdot x) = (\varphi(h) \varphi(h')) \cdot x = \varphi(hh') \cdot x,
	\]
	where we use the fact that $\varphi$ is a homomorphism in the second step.
	Unitality is similar.

	An important special case occurs when $\HH \subseteq \G$ is a subgroup and $\varphi : \HH \hookrightarrow \G$ is the inclusion homomorphism, so that $\varphi(h) = h$.
	The action \eqref{eq:restricted-action-definition-set} is then just the restriction of the original $\G$-action to the subgroup $\HH$, which motivates the terminology ``restriction''.
\end{example}

With Definition \ref{def:pullback-action-definition} given, we can now define the functor $\C^\G \to \C^\HH$ mentioned above.

\begin{definition} \label{def:action-restriction-functor}
	Let $\varphi : \HH \to \G$ be a homomorphism in a Markov category $\C$.
	The \emph{restriction functor}
	\[
		\Res_\varphi : \C^\G \to \C^\HH
	\]
	is defined on objects as $\Res_\varphi(\X, \act) \coloneqq (\X, \act_\varphi)$, where $\act_\varphi$ denotes the action \eqref{eq:restricted-action-definition}.
	On morphisms, $\Res_\varphi$ is just the identity, so that $\Res_\varphi(k) \coloneqq k$.
\end{definition}

To see that $\Res_\varphi$ is indeed a well-defined functor, just observe that, by definition, a morphism $k : \X \to \Y$ in $\C$ is a morphism $\Res_\varphi(\X, \act_\X) \to \Res_\varphi(\Y, \act_\Y)$ in $\C^\HH$ if it holds that
\begin{equation} \label{eq:restricted-equivariance-definition}
	\input{restricted-equivariance-definition.tikz}
\end{equation}
This follows immediately when $k$ is equivariant with respect to $\act_\X$ and $\act_\Y$, and hence for all $k : (\X, \act_\X) \to (\Y, \act_\Y)$ in $\C^\G$.
The other functor axioms are immediate.

\subsection{Symmetrisation procedures} \label{sec:symmetrisation-procedures}

Our primary methodological aim in this paper is to obtain procedures for \emph{symmetrisation}.
At a high-level, given some homomorphism $\varphi : \HH \to \G$ that we specify, we would like a mapping that sends $\HH$-equivariant morphisms to $\G$-equivariant ones.
We formalise this precisely as follows.
Recall that we use the standard notation $\D(\U, \V)$ to denote the set of morphisms $\U \to \V$ in a category $\D$.

\begin{definition} \label{def:symmetrisation-procedure}
	Let $\C$ be a Markov category.
	A \emph{symmetrisation procedure} is a function of the form
	\begin{equation} \label{eq:symmetrisation-procedure-signature}
		\C^\HH(\Res_\varphi(\X, \act_\X), \Res_\varphi(\Y, \act_\Y)) \to \C^\G((\X, \act_\X), (\Y, \act_\Y)) %
	\end{equation}
	where $\varphi : \HH \to \G$ is a homomorphism in $\C$, and $(\X, \act_\X)$ and $(\Y, \act_\Y)$ are objects in $\C^\G$.
\end{definition}

Observe that the left-hand side of \eqref{eq:symmetrisation-procedure-signature} consists of morphisms satisfying \eqref{eq:restricted-equivariance-definition}, whereas the right-hand side consists of morphisms that satisfy this condition when $\varphi$ is removed.
For this reason, we intuitively think of a symmetrisation procedure as transporting ``less equivariant'' morphisms to ``more equivariant'' ones.

\begin{example} \label{ex:subgroup-inclusion-symmetrisation-procedure}
	Let $\C \coloneqq \Set$, and suppose $\varphi : \HH \hookrightarrow \G$ is a subgroup inclusion.
	By Example \ref{ex:subgroup-restricted-action},
	a morphism $f : \Res_\varphi(\X, \act_\X) \to \Res_\varphi(\Y, \act_\Y)$ in $\Set^\HH$ is a function $f : \X \to \Y$ that satisfies
	\[
		f(\act_\X(h, x)) = \act_\Y(h, f(x))
	\]
	for all $h \in \HH$ and $x \in \X$.
	In other words, $f$ is only equivariant with respect to the subgroup $\HH$.
	A symmetrisation procedure along $\varphi$ then ``upgrades'' $f$ to become equivariant with respect to $\act_\X$ and $\act_\Y$ over the full group $\G$.
\end{example}

\begin{example} \label{ex:symmetrisation-along-trivial-homomorphism}
	In a general Markov category $\C$, an important special case takes $\varphi : \I \to \G$ to be the unique homomorphism out of the trivial group from Example \ref{ex:trivial-homomorphism}.
	Since we can identify $\C^\I$ with $\C$ by Example \ref{ex:trivial-equivariant-markov-category}, a symmetrisation procedure for this $\varphi$ is equivalently a function
	\[
		\C(\X, \Y) \to \C^\G((\X, \act_\X), (\Y, \act_\Y)).
	\]
	In other words, this produces $\G$-equivariant morphisms from morphisms in $\C$ that are not subject to any equivariance constraints at all.

\end{example}

\begin{example} \label{ex:semidirect-product-symmetrisation}
	From Example \ref{ex:semidirect-product-coset-maps}, every semidirect product $\N \rtimes_\actr \HH$ in a Markov category $\C$ admits an inclusion homomorphism
	\[
		i_\N : \N \to \N \rtimes_\actr \HH
	\]
	Also recall from Proposition \ref{prop:semidirect-product-actions} that every $(\N \rtimes_\actr \HH)$-action $\act_\X$ can be decomposed as some $\HH$-action $\act_{\X,\HH}$ followed by some $\N$-action $\act_{\X,\N}$.
	A symmetrisation procedure along $i_\N$ is then a function
	\[
		\C^\N((\X, \act_{\X, \N}), (\Y, \act_{\Y, \N})) \to \C^{\N \rtimes_\actr \HH}((\X, \act_{\X}), (\Y, \act_{\Y})).
	\]
	This is the case because
	\[
		\input{restriction-functor-of-iN.tikz}
	\]
	which shows $\Res_{i_\N}(\X, \act_\X) = (\X, \act_{\X, \N})$, and likewise we have $\Res_{i_\N}(\Y, \act_\Y) = (\Y, \act_{\Y, \N})$.
	The situation for the other inclusion homomorphism $i_\HH : \HH \to \N \rtimes_\actr \HH$ is completely analogous.
\end{example}

\subsection{Potential desiderata} \label{sec:symmetrisation-desiderata}

Definition \ref{def:symmetrisation-procedure} is very minimal, and imposes essentially no structure on the function involved: it must simply map $\HH$-equivariant morphisms to $\G$-equivariant ones.
In practice, there are various additional properties that we might want a symmetrisation procedure to satisfy.
Recall from Section \ref{sec:action-restriction} that every morphism $\X \to \Y$ in $\C^\G$ is also a morphism of the form $\Res_\varphi \X \to \Res_\varphi \Y$ in $\C^\HH$.
In other words, it holds that
\[
	\C^\G(\X, \Y) \subseteq \C^\HH(\Res_\varphi \X, \Res_\varphi \Y).
\]
As such, a reasonable requirement of a symmetrisation procedure is that it restricts to the identity on $\C^\G(\X, \Y)$, and so does not modify its input unless strictly necessary.
Precisely:

\begin{definition}
	A symmetrisation procedure $\sym : \C^\HH(\Res_\varphi \X, \Res_\varphi \Y) \to \C^\G(\X, \Y)$ is \emph{stable} if for $k : \X \to \Y$ in $\C^\G$ we have
	\[
		\sym(k) = k.
	\]
\end{definition}

Beyond its intuitive appeal, stability has other desirable consequences.
For example, it follows essentially by definition that a stable symmetrisation procedure is \emph{surjective}.
This is clearly of interest for machine learning applications, since it provides a basic guarantee of the overall expressiveness of the procedure: it is possible to obtain every $\G$-equivariant morphism given some appropriate input.
Additionally, a stable symmetrisation procedure $\sym$ is always \emph{idempotent}, so that $\sym(\sym(k)) = \sym(k)$.
This is useful to know at implementation time: we can run all our experiments with $\sym$ applied just once, confident that we have not sacrificed any performance by doing so.

\subsection{Composing procedures} \label{sec:composing-procedures}

Symmetrisation procedures may be applied in sequence.
For example, suppose we have two homomorphisms
\[
	\begin{tikzcd}
		\K \ar{r}{\phi} & \HH \ar{r}{\varphi} & \G,
	\end{tikzcd}
\]
and two symmetrisation procedures of the form
\begin{equation} \label{eq:symmetrisation-procedure-schematic}
	\begin{tikzcd}[column sep=4em]
		\C^K(\Res_{\phi} \Res_\varphi \X, \Res_{\phi} \Res_\varphi \Y) \ar{r}{\sym_\phi} & \C^\HH(\Res_\varphi \X, \Res_\varphi \Y) \ar{r}{\sym_\varphi} & \C^\G(\X, \Y).
	\end{tikzcd}
\end{equation}
We may then compose these in the obvious way by first applying $\sym_\phi$ and then applying $\sym_\varphi$.
This allows us to start with a morphism that is $\K$-equivariant and end up with one that is $\G$-equivariant.
It follows immediately from inspection of Definition \ref{def:pullback-action-definition} that $\Res_\phi \Res_\varphi = \Res_{\varphi \circ \phi}$, and so this composition is a function
\[
	\begin{tikzcd}
		\C^K(\Res_{\varphi \circ \phi} \X, \Res_{\varphi \circ \phi} \Y) \ar{r} & \C^\G(\X, \Y).
	\end{tikzcd}
\]
This is still a symmetrisation procedure, now along the homomorphism $\varphi \circ \phi : \K \to \G$.

\begin{remark}
	It is clear that stability and surjectivity are both preserved under composition of symmetrisation procedures.
	However, idempotence is not in general.
\end{remark}

\subsection{Deterministic symmetrisation} \label{sec:deterministic-symmetrisation}

In many situations, it is desirable to obtain an equivariant neural network that is also deterministic.
For this, one approach is simply to work in a Markov category whose morphisms are always deterministic, such as $\Set$ or $\Top$ rather than $\Stoch$ or $\TopStoch$.
This ensures any symmetrisation procedure will return a deterministic morphism by construction.
However, recent work suggests that there is an advantage to working in a probabilistic setting even when determinism is ultimately sought \citep{kim2023learning,dym2024equivariant}.
An alternative approach is instead to symmetrise probabilistically, and then to make the result deterministic by computing its expectation.
We describe this now in $\Stoch$.

\begin{definition} \label{def:expectation-operator}
	Let $\Y \coloneqq \R^d$ for some $d \in \Nat$.
	The \emph{expectation operator}, denoted $\ave$, sends a Markov kernel $k : \X \to \Y$ to the measurable function $\ave(k) : \X \to \Y$, defined as
	\begin{equation} \label{eq:deterministic-symmetrised-function}
		\ave(k)(x) \coloneqq \int y \, k(dy|x),
	\end{equation}
	provided this integral exists for all $x \in \X$.
	We will also regard this interchangeably as a deterministic Markov kernel by Remark \ref{rem:measurable-function-as-markov-kernel}, where $\ave(k)(dy|x)$ is Dirac on \eqref{eq:deterministic-symmetrised-function}.
\end{definition}

Here the integral is meant componentwise.
A more general definition, where $\R^d$ is replaced by a Banach space, also appears possible.
That \eqref{eq:deterministic-symmetrised-function} is indeed measurable in $x$ follows by a standard argument using Fubini's theorem.

In general, $\ave(k)$ may not be equivariant even if $k$ is.
However, $\ave$ does preserve equivariance with respect to \emph{affine} group actions.
This occurs very often in practice, since many actions are defined in terms of elementary matrix operations.
The following result makes this more precise.

\begin{proposition} \label{prop:expectation-preserves-equivariance}
	Let $\G$ be a group in $\Meas$ acting on measurable spaces $\X$ and $\Y \coloneqq \R^d$, where the action on $\Y$ is affine, and suppose $k : \X \to \Y$ is a Markov kernel that is equivariant when these actions are lifted to $\Stoch$ by Proposition \ref{prop:lifting-via-monad}.
	Then $\ave(k)$ is also equivariant (provided the integral \eqref{eq:deterministic-symmetrised-function} is everywhere defined).

\end{proposition}

\begin{proof}
	In the notation of Example \ref{ex:stochastic-equivariance}, for all $x \in \X$ and $g \in \G$ we have
	\begin{align*}
		\int y \, k(dy|g \cdot x)
			= \int y \, (g \cdot k(dy| x))
			= \int (g \cdot y) \, k(dy| x)
			= g \cdot \int y \, k(dy| x),
	\end{align*}
	where the second step uses the law of the unconscious statistician, and the third uses the assumption that the action on $\Y$ is affine.
	This shows $\ave(k)$ is equivariant when regarded as a measurable function, and Proposition \ref{prop:lifting-via-monad} then implies that it is equivariant when regarded as a deterministic Markov kernel also.
\end{proof}

In this way, provided these conditions are met, a general strategy for obtaining a deterministic $\G$-equivariant neural network is to apply the following composition:
\begin{equation} \label{eq:deterministic-symmetrisation-procedure}
	\begin{tikzcd}[column sep=3em]
		\Stoch^\HH(\Res_\varphi \X, \Res_\varphi \Y) \ar{r}{\sym} & \Stoch^\G(\X, \Y) \ar{r}{\ave} & \Stoch_\Det^\G(\X, \Y),
	\end{tikzcd}
\end{equation}
where $\sym$ may be any symmetrisation procedure of the type shown.

\begin{remark}
	Technically we should we should regard the composition \eqref{eq:deterministic-symmetrisation-procedure} as a partial function that is defined only on the subset of its domain for which the required integral exists.
	However, in practice, this is not a major issue.
	One reason for this is that $\ave(k)$ is always defined when $k$ is deterministic, since then $k(dy|x)$ is Dirac on \eqref{eq:deterministic-symmetrised-function} (see Example 10.5 of \citet{fritz2020synthetic}).
	As a result, even factoring in integrability caveats, it follows that $\ave \circ \sym$ is always surjective if $\sym$ is, which provides a basic guarantee of its overall expressiveness.
\end{remark}

\begin{remark}
	The preceding discussion works very concretely in the category $\Stoch$.
	A more abstract treatment along the same lines seems possible for a general Markov category that is \emph{representable} in the sense of \citet{fritz2023representable}. %
	However, we will not require this greater generality in what follows.
\end{remark}

\section{A general methodology for symmetrisation} \label{sec:adjunctive-symmetrisation}

\subsection{Motivating idea} \label{sec:adjunctive-symmetrisation-motivation}

Suppose the restriction functor $\Res_\varphi$ admits a left adjoint $\Ext_\varphi$.
For all objects $\X$ and $\Y$ in $\C^\G$, this yields a bijection between the morphisms in $\C^\HH$ and $\C^\G$ as follows:
\begin{equation} \label{eq:adjunction-isomorphism}
	\begin{tikzcd}[column sep=scriptsize]
		\C^\HH(\Res_\varphi \X, \Res_\varphi \Y)
			\ar{r}{\cong} & \C^\G(\Ext_\varphi \Res_\varphi \X, \Y).
	\end{tikzcd}
\end{equation}
This allows us to map $\HH$-equivariant morphisms directly to $\G$-equivariant ones, which seems promising for symmetrisation.
However, \eqref{eq:adjunction-isomorphism} is not yet a symmetrisation procedure in the sense of Definition \ref{def:symmetrisation-procedure} since its output does not have the desired type $\X \to \Y$.
To address this, we can add a second step that simply precomposes by some arbitrary morphism
\[
	\Pre : \X \to \Ext_\varphi \Res_\varphi \X \qquad \text{in $\C^\G$}.
\]
That is, given $k : \Ext_\varphi \Res_\varphi \X \to \Y$ in $\C^\G$ obtained from \eqref{eq:adjunction-isomorphism}, we return $k \circ \Pre : \X \to \Y$, which always has the desired type just by definition of composition.
In this way, when a left adjoint $\Ext_\varphi$ exists, every $\Pre$ gives rise to a symmetrisation procedure defined as follows:
\begin{equation} \label{eq:symmetrisation-high-level-procedure}
	\begin{tikzcd}[column sep=8em]
		\C^\HH(\Res_\varphi \X, \Res_\varphi \Y) \ar{r}{\textnormal{Apply \eqref{eq:adjunction-isomorphism}}} & \C^\G(\Ext_\varphi \Res_\varphi \X, \Y) \ar[description]{r}{\textnormal{Precompose by $\Pre$}} & \C^\G(\X, \Y).
	\end{tikzcd}
\end{equation}
To obtain the best results empirically, we want to choose $\Pre$ ``well'' in some sense.
For this, we will parameterise $\Pre$ using a neural network that we will then optimise, as discussed later.

Overall, this strategy is fully generic and applies without further assumptions on the groups $\G$ and $\HH$ (such as compactness) or their actions on $\X$ and $\Y$ (such as linearity).
Moreover, it is also the only obvious way to construct symmetrisation procedures at this level of generality.
Indeed, since \eqref{eq:adjunction-isomorphism} is a bijection, every symmetrisation procedure can be obtained by first applying \eqref{eq:adjunction-isomorphism} and then mapping the result through \emph{some} function $\C^\G(\Ext_\varphi \Res_\varphi \X, \Y) \to \C^\G(\X, \Y)$.
But without additional structure to exploit, it is unclear how to write down such a function that is \emph{not} precomposition.
In this sense, \eqref{eq:symmetrisation-high-level-procedure} is the only strategy for obtaining symmetrisation procedures that is fully ``general purpose''.\footnote{The Yoneda lemma gives a more precise (but more technical) characterisation than this: the symmetrisation procedures that are \emph{natural} in $Y$ are precisely those of the form \eqref{eq:symmetrisation-high-level-procedure}.}

\subsection{General approach}

The requirement of a full left adjoint is stronger than necessary: rather than $\Ext_\varphi$ itself, the procedure \eqref{eq:symmetrisation-high-level-procedure} only requires us to construct the composition $\Ext_\varphi \Res_\varphi$.\footnote{In technical terms, rather than a full left adjoint to $\Res_\varphi$, we require only a \emph{left $\Res_\varphi$-relative adjoint}.}
In classical settings such as $\Set$, this can be obtained as
\begin{equation} \label{eq:relative-left-adjoint}
	\Ext_\varphi \Res_\varphi \X \cong \G/\HH \otimes \X,
\end{equation}
where $\G/\HH$ is a coset space (see e.g.\ equation (1.6) of \citet{may1997equivariant}).
The following result in effect generalises this fact to an arbitrary Markov category.

\begin{theorem} \label{thm:partial-left-adjoint-existence}
	Let $\C$ be a Markov category and $\varphi : \HH \to \G$ a homomorphism in $\C$.
	Suppose a $\varphi$-coset map $q : \G \to \G/\HH$ exists, and let $\mul/\HH$ be the action of $\G$ on $\G/\HH$ induced by Proposition \ref{prop:induced-action-on-cosets}.
	Then for every $(\X, \act_\X)$ and $(\Y, \act_\Y)$ in $\C^\G$, there is a natural bijection
	\begin{equation} \label{eq:transpose-characterisation-hom-sets}
		\begin{tikzcd}
			\C^\HH(\Res_\varphi(\X, \act_\X), \Res_\varphi(\Y, \act_\Y)) \ar{r}{\cong} & \C^\G((\G/\HH, \mul/\HH) \otimes (\X, \act_\X), (\Y, \act_\Y))
		\end{tikzcd}
	\end{equation}
	that sends $k : \Res_\varphi(\X, \act_\X) \to \Res_\varphi(\Y, \act_\Y)$ in $\C^\HH$ to the unique $k^\transpose : \G/\HH \otimes \X \to \Y$ in $\C$ such that
	\begin{equation} \label{eq:transpose-characterisation}
		\input{transpose-characterisation.tikz}
	\end{equation}
	which always exists, and is always a morphism $(\G/\HH, \mul/\HH) \otimes (\X, \act_\X) \to (\Y, \act_\Y)$ in $\C^\G$.
\end{theorem}

\begin{proof}
	See Section \ref{sec:existence-of-left-adjoint-proof} of the Appendix.
\end{proof}

When it applies, Theorem \ref{thm:partial-left-adjoint-existence} gives rise to symmetrisation procedures via the same two-step process as in \eqref{eq:symmetrisation-high-level-procedure}, making the substitution shown in \eqref{eq:relative-left-adjoint}.
Explicitly, in the first step, we now apply the bijection \eqref{eq:transpose-characterisation-hom-sets}, and in the second step we now precompose by some
\begin{equation} \label{eq:symmetrisation-high-level-procedure-with-reader-comonad}
	\Pre : (\X, \act_\X) \to (\G/\HH, \mul/\HH) \otimes (\X, \act_\X) \qquad \text{in $\C^\G$}.
\end{equation}
For our purposes, this approach is more convenient than constructing a full left adjoint, and we will use it as the basis for the methodology we develop in what follows.

\begin{remark} \label{rem:transpose-via-section}
	A $\varphi$-coset map $q : \G \to \G/\HH$ often admits a \emph{section} (or \emph{right-inverse}), which is namely a morphism $s : \G/\HH \to \G$ in $\C$ such that
	\[
		q \circ s = \id_{\G/\HH}.
	\]
	For example, a section of the coset map $g \mapsto g\HH$ from Example \ref{ex:coset-map-example} selects a representative element of each coset in $\G/\HH$.
	This is useful computationally, since by attaching $s$ to the $\G$-input on both sides of \eqref{eq:transpose-characterisation}, we may write $k^\transpose$ explicitly as follows:
	\begin{equation} \label{eq:transpose-plus-precomposition-result-with-section}
		\input{transpose-plus-precomposition-result-with-section.tikz}
	\end{equation}
	In general there may be many choices of $s$, but each one will produce the same $k^\transpose$.
	Moreover, $s$ need not be equivariant in any sense, so need not be a morphism in $\C^\G$ or $\C^\HH$.
\end{remark}

\begin{remark}
	Should it be required, a sufficient condition for $\Res_\varphi$ to have a full left adjoint is that $\C$ admits all orbit maps.
	We omit the proof of this, but the result follows using very similar ideas as the proof of Theorem \ref{thm:partial-left-adjoint-existence}.
	Compared with \eqref{eq:adjunction-isomorphism}, this yields a natural bijection 
	\[
		\C^\HH(\Z, \Res_\varphi \Y) \cong \C^\G(\Ext_\varphi \Z, \Y)
	\]
	for all $\Z$ in $\C^\HH$, rather than just $\Z = \Res_\varphi \X$.
	This may yield interesting possibilities for symmetrisation: for instance, we could generalise \eqref{eq:symmetrisation-high-level-procedure} to start with a morphism $\Z \to \Res_\varphi \Y$ in $\C^\HH$, where now $\Z$ may be any object in $\C^\HH$, and in the second step precompose instead by a morphism $\Pre : \X \to \Ext_\varphi \Z$ in $\C^\G$.
	We leave this for future work to explore.
\end{remark}

\subsection{Obtaining a precomposition morphism} \label{sec:precomposition-morphism}

Suppose Theorem \ref{thm:partial-left-adjoint-existence} applies.
To obtain an overall symmetrisation procedure, it remains to select a precomposition morphism $\Pre$ of the form \eqref{eq:symmetrisation-high-level-procedure-with-reader-comonad}.
For this, we will set
\begin{equation} \label{eq:precomposition-morphism-2}
	\input{precomposition-morphism.tikz}
\end{equation}
where $\pre$ may be any morphism $(\X, \act_\X) \to (\G/\HH, \mul/\HH)$ in $\C^\G$.
By definition of composition in $\C^\G$, this means $\Pre$ has the required type from \eqref{eq:symmetrisation-high-level-procedure-with-reader-comonad} also.

\begin{remark}
We will denote by $\sym_\pre$ the overall symmetrisation procedure that first applies \eqref{eq:transpose-characterisation-hom-sets} and then precomposes by \eqref{eq:precomposition-morphism-2}.
End-to-end, this has the following type:
\[
	\begin{tikzcd}[column sep=3em]
		\C^\HH(\Res_\varphi(\X, \act_\X), \Res_\varphi(\Y, \act_\Y)) \ar{r}{\sym_\pre} & \C^\G((\X, \act_\X), (\Y, \act_\Y)).
	\end{tikzcd}
\]
However, notice that same choice of $\pre$ may be reused across more than one $(\Y, \act_\Y)$.
We will abuse notation slightly by denoting every such procedure using the same symbol $\sym_\pre$, even though technically these are distinct functions when their codomains differ.
\end{remark}

The choice of \eqref{eq:precomposition-morphism-2} sacrifices some generality, since not every morphism in $\C^\G$ of the form \eqref{eq:symmetrisation-high-level-procedure-with-reader-comonad} can be expressed in this way.
Our reason for this restriction is that it is sufficient, and for positive Markov categories necessary, to ensure that the overall procedure is stable, as the next result shows.
As discussed in Section \ref{sec:symmetrisation-desiderata}, stability in turn means the procedure is surjective and idempotent, both of which are desirable for machine learning applications.

\begin{proposition} \label{prop:sym-gamma-is-stable}
	Assuming Theorem \ref{thm:partial-left-adjoint-existence} applies, the procedure $\sym_\pre$ is stable for every choice of $\pre$.
	Conversely, if $\C$ is positive, then every instance of \eqref{eq:symmetrisation-high-level-procedure} that is stable can be obtained as $\sym_\pre$ for some choice of $\pre$.
\end{proposition}

\begin{proof}
	See Section \ref{sec:proof-of-sym-gamma-is-stable} of the Appendix.
\end{proof}

\begin{remark} \label{rem:recursive-symmetrisation}
As a morphism in $\C^\G$, we require $\pre$ already to be $\G$-equivariant.
This in effect pushes back the problem of symmetrisation to the choice of $\pre$, which mirrors the situation for deterministic symmetrisation approaches also \citep{puny2022frame,kaba2023equivariance,kim2023learning}.
Previous work has assumed $\pre$ is given directly in the form of an intrinsically equivariant neural network.
However, our formalism suggests an alternative, \emph{recursive} procedure for obtaining $\pre$, which adds a new layer of flexibility.
In particular, we can set
\[
	\pre \coloneqq \sym_{\pre_0}(\pre_1)
\]
for some $\pre_1 : \Res_\varphi(\X, \act_\X) \to \Res_\varphi(\G/\HH, \mul/\HH)$ in $\C^\HH$, where now $\pre_0 : (\X, \act_\X) \to (\G/\HH, \mul/\HH)$ in $\C^\G$.
If desired, $\pre_0$ could itself be the result of recursive symmetrisation, and so on.
Of course, this still ultimately requires a base case, for which an intrinsically equivariant neural network could be used.
In Section \ref{sec:examples}, we also provide various examples of very simplistic choices that provide a default option in cases where a more complex choice is not available or not desired.
We expect that in many cases the recursive approach will lead to a more expressive $\pre$ than the base case alone, just as previous work on symmetrisation has shown improved performance compared with intrinsically equivariant baselines.
We provide empirical evidence of this in Section \ref{sec:empirical-results} below.

\end{remark}

\subsection{End-to-end procedure} \label{sec:end-to-end-procedure}

We now summarise the complete steps required to symmetrise a morphism $k : \Res_\varphi(\X, \act_\X) \to \Res_\varphi(\Y, \act_\Y)$ in $\C^\HH$ along a homomorphism $\varphi : \HH \to \G$ as described in this section.
\begin{mdframed}
	\begin{enumerate}[wide, labelindent=0pt]
		\item Obtain a $\varphi$-coset map $q : \G \to \G/\HH$ in $\C$
		\item Obtain a section $s : \G/\HH \to \G$ of $q$ in $\C$
		\item Determine the action $\mul/\HH : \G \otimes \G/\HH \to \G/\HH$ induced by Proposition \ref{prop:induced-action-on-cosets}
		\item Obtain a morphism $\pre : (\X, \act_\X) \to (\G/\HH, \mul/\HH)$ in $\C^\G$, either recursively or from some base case
		\item Return $\sym_\pre(k) : (\X, \act_\X) \to (\Y, \act_\Y)$ in $\C^\G$ computed as follows:
		\begin{equation} \label{eq:fully-symmetrised-morphism}
			\input{fully-symmetrised-morphism.tikz}
		\end{equation}
	\end{enumerate}
\end{mdframed}
Here the last diagram is obtained simply by attaching our precomposition morphism \eqref{eq:precomposition-morphism-2} to $k^\sharp$ as obtained using $s$ in \eqref{eq:transpose-plus-precomposition-result-with-section}.
More generally, $k^\sharp$ is still defined even without $s$, although then only implicitly by the expression \eqref{eq:transpose-characterisation} which makes its computation more difficult.

\begin{remark}
	By Examples \ref{ex:set-has-orbits} and \ref{ex:top-has-orbits} and Theorem \ref{thm:topstoch-has-orbits}, a $\varphi$-coset map exists for every homomorphism $\varphi$ in $\Set$, $\Top$, and $\TopStoch$.
	This procedure is therefore always applicable in all of these cases, provided a suitable $\pre$ can be found.
	For other Markov categories such as $\Stoch$, the procedure still applies whenever we can find a $\varphi$-coset map.
	This may not always be possible, although we suspect that any such cases would be somewhat pathological and of less interest practically.
\end{remark}

\begin{example} \label{ex:symmetrising-from-base-markov-category}
	Suppose $\HH \coloneqq \I$ is the trivial group and $\varphi$ is the unique homomorphism
	\[
		\I \to \G.
	\]
	By identifying $\C^\I$ with $\C$ as in Example \ref{ex:trivial-equivariant-markov-category}, symmetrising along $\varphi$ allows us to convert arbitrary morphisms in $\C$ to ones that are $\G$-equivariant.
	Example \ref{ex:trivial-coset-map} shows that $\id_\G$ is always a $\varphi$-coset map, and this is trivially its own section, which lets us take $s \coloneqq \id_\G$.
	Moreover, by Example \ref{ex:trivial-coset-map-action}, the canonical action $\mul/\I$ is just the action of $\G$ on itself by left multiplication.
	As such, all that is required here is a base case $\pre : (\X, \act_\X) \to (\G, \ast)$ in $\C^\G$.
	We give some default options for specific groups and actions in Section \ref{sec:examples}.
\end{example}

\subsection{Concrete instantiation} \label{sec:concrete-instantiation}

We now show how our framework above may be instantiated concretely in various ways.
In doing so, we recover previous methods that have been proposed for symmetrising deterministic functions.
This accords with our discussion in Section \ref{sec:adjunctive-symmetrisation-motivation} about the generality of our approach.
Our framework also extends to provide a methodology for symmetrising \emph{Markov kernels} (Example \ref{ex:symmetrising-markov-kernels}), a task that has not previously been considered in this literature.

\begin{example}
	Take $\C \coloneqq \Set$, so that all the morphisms involved are just functions, and let $\varphi : \I \to \G$ be the trivial homomorphism.
	By Example \ref{ex:symmetrising-from-base-markov-category} we may let $s$ be the identity, and so \eqref{eq:fully-symmetrised-morphism} becomes
	\begin{equation} \label{eq:canonicalisation}
		\sym_\pre(k)(x) = \pre(x) \cdot k(\pre(x)^{-1} \cdot x).
	\end{equation}
	This recovers the \emph{canonicalisation} approach of \citet{kaba2023equivariance}.
	If $\varphi : \HH \to \G$ is a general homomorphism, then \eqref{eq:fully-symmetrised-morphism} becomes instead
	\[
		\sym_\pre(k)(x) = s(\pre(x)) \cdot k(s(\pre(x))^{-1} \cdot x).
	\]
	By letting $\varphi$ be a subgroup inclusion, this recovers the more general \emph{partial canonicalisation} approach from Theorem 3.1 of \citet{kaba2023equivariance}, where our $s \circ \pre$ plays the role of their $h$.

\end{example}

\begin{example} \label{ex:averaging-approach}
	Let $\G$ be a group acting on $\X$ and $\Y \coloneqq \R^d$ in $\Meas$.
	Moreover, let $\varphi : \I \to \G$ be the trivial homomorphism, and $f : \X \to \Y$ a measurable function.
	Now lift these components to $\C \coloneqq \Stoch$ via Proposition \ref{prop:lifting-via-monad}.
	By computing \eqref{eq:fully-symmetrised-morphism} with $k \coloneqq f$ and then applying the expectation operator from Section \ref{sec:deterministic-symmetrisation}, we obtain a measurable function
	\begin{equation} \label{eq:kim-et-al-approach}
		\ave \circ \sym_\pre(f)(x) = \int y \, \sym_\pre(f)(dy|x) = \int g \cdot f(g^{-1} \cdot x) \, \pre(dg|x)
	\end{equation}
	by the law of the unconscious statistician.
	Provided the action on $\Y$ is affine, Proposition \ref{prop:expectation-preserves-equivariance} then implies that this is equivariant.
	This recovers the approach of \citet{kim2023learning}, who consider the case where $\G$ is compact and $\pre$ admits an equivariant (in the sense of \eqref{eq:stochastic-equivariance-density-definition}) conditional density with respect to the Haar measure on $\G$.
	Our framework immediately extends this to the case of a general homomorphism $\varphi : \HH \to \G$.

	In turn, Proposition 1 of \citet{kim2023learning} shows that \emph{frame averaging} \citep[(3)]{puny2022frame} is an instance of their approach, where $\pre(dg|x)$ is uniform on the set produced by the frame at the value $x$.
	The right-hand side of \eqref{eq:kim-et-al-approach} also recovers the (equivariant) \emph{weighted frames} approach from Remark 3.4 of \citet{dym2024equivariant}, explicating its relationship to \citet{kim2023learning}. 
	In addition, by letting $\G \coloneqq \Sn_n$ be the group of permutations of $\{1, \ldots, n\}$, with $\pre(dg|x)$ uniform on $\Sn_n$, and equipping $\Y$ with the trivial action, \eqref{eq:kim-et-al-approach} becomes
	\begin{equation} \label{eq:janossy-pooling}
		\frac{1}{n!} \sum_{\sigma \in \Sn_n} f(\sigma^{-1} \cdot x),
	\end{equation}
	which recovers \emph{Janossy pooling} \citep[Definition 2.1]{murphy2018janossy}, an early example of a symmetrisation procedure in the machine learning literature.
\end{example}

\begin{example} \label{ex:symmetrising-markov-kernels}
	By taking $\C \coloneqq \Stoch$, our framework gives a novel methodology for symmetrising Markov kernels $k : \X \to \Y$ directly.
	We refer to this as \emph{stochastic symmetrisation} to distinguish it from previous methods discussed above, which apply instead to deterministic functions.
	Stochastic symmetrisation is straightforward to implement: given $x \in \X$, an exact sampling procedure for $\sym_\pre(k)(dy|x)$ may be read off directly from \eqref{eq:fully-symmetrised-morphism} as
	\begin{equation} \label{eq:stochastic-symmetrisation-sampling-procedure}
		\bm{C} \sim \pre(dc|x) \qquad \bm{\G} \sim s(dg|\bm{C}) \qquad \bm{\Y} \sim k(dy|{\bm \G}^{-1} \cdot x) \qquad \text{return $\bm{\G} \cdot \bm{\Y}$.}
	\end{equation}
	It follows from the results of this section that this always describes a stochastically equivariant Markov kernel.
	Note here that, for simplicity, we are technically assuming the group and actions are lifted from $\Meas$ via Proposition \ref{prop:lifting-via-monad}, as will be the case in our examples below.
	This allows us to write e.g.\ $\bm{G} \cdot \bm{Y}$ rather than the more verbose $\act_\Y(dy|\bm{G}, \bm{Y})$.

	Some intuition for stochastic symmetrisation is obtained from the special case where $\act_\Y$ is trivial and the setup of Example \ref{ex:symmetrising-from-base-markov-category} applies.
	The sampling procedure \eqref{eq:stochastic-symmetrisation-sampling-procedure} then becomes
	\[
		\bm{\G} \sim \pre(dg|x) \qquad \bm{\Y} \sim k(dy|{\bm \G}^{-1} \cdot x) \qquad \text{return $\bm{\Y}$},
	\]
	so that $\pre$ behaves essentially as a \emph{data augmentation}.
	However, in our framework, this augmentation is considered a part of the overall model, rather than something that is applied externally.
	In particular, the augmentation may be learned during training (Section \ref{sec:empirical-results}), and continues to be applied when the model is deployed at test time.

	Stochastic symmetrisation has various appealing practical properties.
	For instance, it avoids the need for averaging over $\Y$ as in \eqref{eq:kim-et-al-approach}, which may be costly and approximate, and which may not even be defined at all if $\Y$ is nonconvex.
	Likewise, compared with \eqref{eq:canonicalisation}, it does not require a deterministic $\G$-equivariant $\pre : \X \to \G$, which previous work has shown must be discontinuous for many group actions of interest \citep{dym2024equivariant}.
	In contrast, stochastic symmetrisation takes $\pre$ to be a Markov kernel, permitting natural and flexible choices (including for all compact groups) as the examples in the next section show.
\end{example}

Before proceeding, it is worth considering how \eqref{eq:stochastic-symmetrisation-sampling-procedure} could be obtained using a classical measure-theoretic approach.
To do so rigorously would require writing down a function $\Sigma_\Y \times \X \to [0, 1]$ that satisfies the usual Markov kernel axioms, and then showing that this is equivariant in the sense of Example \ref{ex:stochastic-equivariance}.
Attempted directly, this is quite intractable, especially due to the presence of the coset space $\G/\HH$ here.
Moreover, a successful proof of this kind would offer very little conceptual benefit: it would not explain where this approach ``comes from'' or whether it is the only possibility.
In contrast, Markov categories take care of this tedious bookkeeping automatically, abstracting away these details without sacrificing mathematical precision.
In turn, it becomes much easier to identify the higher-level principles that give rise to this approach, such as the bijection from Theorem \ref{thm:partial-left-adjoint-existence}.

\section{Examples} \label{sec:examples}

We now show how implement the steps described in Section \ref{sec:end-to-end-procedure} for a variety of homomorphisms $\varphi : \HH \to \G$.
In each case, we will provide a suitable $\varphi$-coset map, section, and coset action.
We will also give examples of $\pre$ that could be used as base cases in the recursive procedure described in Remark \ref{rem:recursive-symmetrisation}.
These are by no means the only choice, and for instance an intrinsically equivariant neural network could be used instead for this purpose if desired.

Several of the examples involve an abstract Markov category $\C$.
In some other cases, we take $\C \coloneqq \Set$.
However, we do so mainly to simplify our presentation: $\Meas$ or $\Top$ would work just as well, and in turn the various components we describe (coset maps, sections, etc.)\ lift to $\Stoch$ or $\TopStoch$ by Proposition \ref{prop:lifting-via-monad}, and could therefore be used for stochastic symmetrisation also.

\begin{example} \label{ex:compact-group-symmetrisation}
	Let $\G$ be a compact group in a Markov category $\C$, and consider the unique homomorphism
	\[
		\I \to \G
	\]
	out of the trivial group.
	By Example \ref{ex:symmetrisation-along-trivial-homomorphism}, symmetrising along this gives a procedure
	\[
		\C(\X, \Y) \to \C^\G((\X, \act_\X), (\Y, \act_\Y))
	\]
	that sends arbitrary morphisms in $\C$ to ones that are $\G$-equivariant.
	From Example \ref{ex:symmetrising-from-base-markov-category}, the identity $\id_\G$ is both a coset map and a section, and the coset action is just the multiplication operation $\mul$.
	By substituting these into \eqref{eq:fully-symmetrised-morphism}, the result of symmetrising $k : \X \to \Y$ in $\C$ becomes as follows:
	\begin{equation} \label{eq:trivial-symmetrised-morphism}
		\input{trivial-symmetrised-morphism.tikz}
	\end{equation}
	To implement this, all we require here is a choice of base case in $\C^\G$ of form
	\[
		\pre : (\X, \act_\X) \to (\G, \mul)
	\]
	A default choice that works for any $(\X, \act_\X)$ is
	\begin{equation} \label{eq:haar-measure-as-base-case}
		\input{haar-measure-as-base-case.tikz}
	\end{equation}
	where $\haar : \I \to \G$ is the Haar measure of $\G$.
	It follows straightforwardly from Definition \ref{def:compact-group} that this is equivariant with respect to $\act_\X$ and $\mul$.
	In this way, we obtain a strategy that works for all actions of compact groups.

	Since $\haar$ behaves as a uniform distribution on $\G$, the choice of \eqref{eq:haar-measure-as-base-case} may not be very performant when used in \eqref{eq:trivial-symmetrised-morphism} directly, especially in high dimensional settings.
	However, it can yield much stronger results when used as the base case of the recursive procedure from Remark \ref{rem:recursive-symmetrisation}.
	We demonstrate this empirically in Section \ref{sec:empirical-results} below.
\end{example}

\begin{example} \label{ex:translation-group-example}
	In $\Set$, let $\T_d$ denote the translation group, namely $\R^d$ equipped with vector addition.
	We consider how to symmetrise along the unique homomorphism
	\[
		\I \to \T_d.
	\]
	Here we can reuse the same approach as Example \ref{ex:compact-group-symmetrisation}, although we now need a base case in $\Set^{\T_d}$ of the form
	\[
		\pre : (\X, \act_\X) \to (\T_d, +).
	\]
	A suitable choice will depend on $(\X, \act_\X)$.
	A common situation takes $\X \coloneqq \R^{d \times n}$, and obtains $\act_\X$ by columnwise addition, so that
	\[
		t \cdot (x_1, \ldots, x_n) \coloneqq (x_1 + t, \ldots, x_n + t)
	\]
 	where $t \in \T_d$ and $x_i \in \R^d$.
	An obvious $\pre$ is then the columnwise mean
	\begin{equation} \label{eq:columnwise-mean}
		\pre(x_1, \ldots, x_n) \coloneqq \frac{1}{n} \sum_{i=1}^n x_i,
	\end{equation}
	which is easily verified to be equivariant as required.
\end{example}

\begin{example} \label{ex:semidirect-product-symmetrisation-example-iN}
	Let $\N \rtimes_\actr \HH$ be a semidirect product in a Markov category $\C$.
	From Example \ref{ex:semidirect-product-coset-maps}, this always comes equipped with an inclusion homomorphism:
	\[
		i_\N : \N \to \N \rtimes_\actr \HH.
	\]
	By Example \ref{ex:semidirect-product-symmetrisation}, symmetrising along $i_\N$ gives a procedure
	\[
		\C^\N((\X, \act_{\X,\N}), (\Y, \act_{\Y,\N})) \to \C^{\N \rtimes_\actr \HH}((\X, \act_\X), (\Y, \act_\Y))
	\]
	that converts $\N$-equivariant morphisms to ($\N \rtimes_\actr \HH$)-equivariant ones, where here we have used Proposition \ref{prop:semidirect-product-actions} to decompose $\act_\X$ into an $\HH$-action $\act_{\X,\HH}$ followed by an $\N$-action $\act_{\X,\N}$, and similarly for $\act_\Y$.
	In this context:
	\begin{itemize}
		\item The projection $p_\HH : \N \rtimes_\actr \HH \to \HH$ is an $i_\N$-coset map by Example \ref{ex:semidirect-product-coset-maps} 
		\item The other inclusion $i_\HH : \HH \to \N \rtimes_\actr \HH$ is a section of $p_\HH$ (as may be checked)
		\item The coset action is $\mul/\N$ from Example \ref{ex:semidirect-product-coset-maps-induced-actions}
	\end{itemize}
	Substituting these into \eqref{eq:fully-symmetrised-morphism} gives the result of symmetrising $k : (\X, \act_{\X,\N}) \to (\Y, \act_{\Y,\N})$ in $\C^\N$ as follows:
	\begin{equation} \label{eq:semidirect-N-symmetrised-morphism}
		\input{semidirect-N-symmetrised-morphism.tikz}
	\end{equation}
	where the equality follows here from the definitions of $i_\HH$ and the inversion operation for the semidirect product (Definition \ref{def:semidirect-product-definition}), together with some basic group theoretic manipulations.

	To implement this approach, all that remains is to find a base case in $\C^{\N \rtimes_\actr \HH}$ of the form
	\begin{equation} \label{eq:semidirect-product-sym-via-N-gamma}
		\pre : (\X, \act_\X) \to (\HH, \mul/\N).
	\end{equation}
	Inspecting Example \ref{ex:semidirect-product-coset-maps-induced-actions}, we see that $\mul/\N$ decomposes as the $\HH$-action $\mul_\HH$ followed by the trivial $\N$-action.
	By Proposition \ref{prop:semidirect-product-equivariance}, we therefore want $\pre$ to be equivariant with respect to $\act_{\X,\HH}$ and $\mul_\HH$, and invariant with respect to $\act_{\X,\N}$.
	When $\C = \Set$, this recovers Theorem 3.2 of \citet{kaba2023equivariance}.
	
	Whenever $\HH$ is compact, a default choice of $\pre$ here is given by the Haar measure $\haar$ in a similar way to Example \ref{ex:compact-group-symmetrisation},\footnote{It holds that $\G \cong \I \rtimes_{\triv} \G$ as groups in $\C$, and so Example \ref{ex:compact-group-symmetrisation} is really a special case of this example.} namely
	\[
		\input{haar-measure-as-base-case-1.tikz}
	\]
	which is easily checked to satisfy the required equivariance conditions.
	This works without assumptions on $\N$, which may be noncompact.
	As a result, symmetrisation can be performed in a fully compositional way in this case: if we know how to obtain $\N$-equivariant morphisms and $\HH$ is compact, then we know how to obtain ($\N \rtimes_\actr \HH$)-equivariant morphisms also.
	Note however that this approach is not possible when $\C = \Set$ as in \citet{kaba2023equivariance}, since then a Haar measure does not exist unless $\G$ is trivial.
	By allowing a more general $\C$, we obtain greater flexibility here.

\end{example}

\begin{example} \label{ex:semidirect-product-symmetrisation-example-iH}
	Let $\N \rtimes_\actr \HH$ be a semidirect product in a Markov category $\C$.
	In addition to the inclusion $i_\N$ from the previous example, we can also symmetrise along the other inclusion homomorphism from Example \ref{ex:semidirect-product-coset-maps}
	\[
		i_\HH : \HH \to \N \rtimes_\actr \HH.
	\]
	By Example \ref{ex:semidirect-product-symmetrisation}, this gives a procedure
	\[
		\C^\HH((\X, \act_{\X,\HH}), (\Y, \act_{\Y,\HH})) \to \C^{\N \rtimes_\actr \HH}((\X, \act_\X), (\Y, \act_\Y))
	\]
	that sends $\HH$-equivariant morphisms to ($\N \rtimes_\actr \HH$)-equivariant ones, where here we have again decomposed $\act_\X$ into an $\HH$-action $\act_{\X,\HH}$ and an $\N$-action $\act_{\X,\N}$ by Proposition \ref{prop:semidirect-product-actions}, and similarly for $\act_\Y$.
	The situation is now dual to Example \ref{ex:semidirect-product-symmetrisation-example-iN}:
	\begin{itemize}
		\item The projection $p_\N : \N \rtimes_\actr \HH \to \N$ is an $i_\HH$-coset map by Example \ref{ex:semidirect-product-coset-maps}
		\item The other inclusion $i_\N : \N \to \N \rtimes_\actr \HH$ is a section of $p_\N$ (as may be checked)
		\item The coset action is $\mul/\HH$ from Example \ref{ex:semidirect-product-coset-maps-induced-actions}
	\end{itemize}
	By substituting these into \eqref{eq:fully-symmetrised-morphism} and applying some basic manipulations, the result of symmetrising $k : (\X, \act_{\X,\HH}) \to (\Y, \act_{\Y,\HH})$ in $\C^\HH$ has the same form as \eqref{eq:semidirect-N-symmetrised-morphism}, but with each occurrence of the symbol $\HH$ replaced by $\N$ instead.
	To implement this procedure, we now only need a base case morphism in $\C^{\N \rtimes_\actr \HH}$ of the form
	\[
		\pre : (\X, \act_\X) \to (\N, \mul/\HH).
	\]
	Unlike when symmetrising along the other inclusion $i_\N$, a general choice of $\pre$ (that leverages, say, compactness) seems less forthcoming here, but case-by-case choices are still possible.
	For example, in $\Set$, consider the Euclidean group $\Euc(d) = \T_d \rtimes_\actr \Orth(d)$ from Example \ref{ex:euclidean-groups} (the case of $\SE(d)$ is similar).
	From its definition in Example \ref{ex:semidirect-product-coset-maps-induced-actions}, we see that $\mul/\HH$ can be written here as 
	\[
		(t, Q) \cdot t' = t + Qt'
	\]
	for $t, t' \in \T_d$ and $Q \in \Orth(d)$.
	A common situation takes $\X \coloneqq \R^{d \times n}$ and obtains $\act_\X$ in a columnwise fashion, so that
	\[
		(t, Q) \cdot (x_1, \ldots, x_n) \coloneqq (Qx_1 + t, \ldots, Qx_n + t).
	\]
	Intuitively, $(x_1, \ldots, x_n)$ is thought of as a cloud of $n$ points $x_i \in \R^d$, and $\Euc(d)$ acts by rotating and then translating this cloud rigidly.
	It is straightforward to check that the columnwise mean \eqref{eq:columnwise-mean} from Example \ref{ex:translation-group-example} also provides a suitable $\pre$ here.\footnote{It holds that $\T_d \cong \T_d \rtimes_{\triv} \I$ as groups in $\C$, and so Example \ref{ex:translation-group-example} is really a special case of this example.}
	The resulting model $\sym_\gamma(k)$ first subtracts the centroid from its input point cloud, then maps this through $k$, and finally adds the centroid back on to the result.
	This recovers the trick described in Section 2.2 of \citet{kim2023learning} to obtain $\SE(d)$-equivariance, showing it arises from the same underlying framework as the other examples we consider.
\end{example}

\begin{example} \label{ex:semidirect-product-full-symmetrisation}
	Given a semidirect product $\N \rtimes_\actr \HH$ in a Markov category $\C$, it is also possible to symmetrise directly along the unique homomorphism
	\[
		I \to \N \rtimes_\actr \HH,
	\]
	as opposed to the $\N$- or $\HH$-inclusions from the previous examples.
	From Example \ref{ex:symmetrisation-along-trivial-homomorphism}, a procedure of this kind allows us to obtain ($\N \rtimes_\actr \HH$)-equivariant morphisms from arbitrary morphisms in $\C$.
	Here we can reuse the same overall approach as in Example \ref{ex:compact-group-symmetrisation}, and only need to select a base case in $\C^{\N \rtimes_\actr \HH}$ of the form
	\begin{equation} \label{eq:semidirect-product-full-symmetrisation-gamma-type}
		\pre : (\X, \act_\X) \to (\N \rtimes_\actr \HH, \mul).
	\end{equation}
	By inspection of the semidirect product multiplication in Definition \ref{def:semidirect-product-definition}, it follows straightforwardly that
	\[
		(\N \rtimes_\actr \HH, \mul) = (\N, \mul/\HH) \otimes (\HH, \mul/\N),
	\]
	where $\mul/\HH$ and $\mul/\N$ are the coset actions of $\N \rtimes_\actr \HH$ from Example \ref{ex:semidirect-product-coset-maps-induced-actions}.
	A general strategy is therefore to take
	\[
		\input{semidirect-gamma-from-marginals.tikz}
	\]
	where $\pre_\N$ and $\pre_\HH$ are morphisms in $\C^{\N \rtimes_\actr \HH}$ of the types shown, which in turn ensures $\pre$ has the desired type \eqref{eq:semidirect-product-full-symmetrisation-gamma-type} simply because $\C^{\N \rtimes_\actr \HH}$ is a Markov category.
	In $\Set$, this recovers the approach used by \citet{kaba2023equivariance} for the Euclidean group, who obtain $\pre_\N$ and $\pre_\K$ using intrinsically equivariant neural networks (see their equations (9) and (10)).

	In a sense, this approach is the easiest possible here, since any other $\pre$ of the form \eqref{eq:semidirect-product-full-symmetrisation-gamma-type} immediately gives rise to suitable choices of $\pre_\N$ and $\pre_\HH$ by projecting onto $\N$ and $\HH$ respectively.
	Notice that this also is a stronger requirement than in Examples \ref{ex:semidirect-product-symmetrisation-example-iN} and \ref{ex:semidirect-product-symmetrisation-example-iH}, which required $\pre_\N$ \emph{or} $\pre_\HH$ of this form, but not both.
	To some extent, this additional complexity is to be expected given we are now symmetrising arbitrary morphisms in $\C$, whereas previously we were symmetrising morphisms that were already $\N$- or $\HH$-equivariant.
\end{example}

\begin{example}
	Interestingly, we can even obtain equivariance with respect to the full general linear group $\GL(d, \R)$.
	In $\Set$, consider the inclusion homomorphism
	\[
		\Orth(d) \hookrightarrow \GL(d, \R).
	\]
	By Example \ref{ex:subgroup-inclusion-symmetrisation-procedure}, symmetrising along this converts $\Orth(d)$-equivariant morphisms to $\GL(d, \R)$-equivariant ones.
	Here a coset map $\GL(d, \R) \to \PD(d)$ is given by $A \mapsto AA^T$ from Example \ref{ex:glnr-coset-map}, and its induced coset action $\mul/\Orth(d)$ is computed as $A \cdot P = A P A^T$ by Example \ref{ex:glnr-coset-map-action}.
	A section of the coset map may also be obtained straightforwardly by the Cholesky decomposition.
	We therefore only need a base case
	\[
		\pre : (\X, \act_\X) \to (\PD(d), \mul/\Orth(d))
	\]
	in $\Set^{\GL(d, \R)}$. 
	Consider the specific example of $\X \coloneqq \R^{d \times n}$, where $\act_\X$ is obtained by left-multiplication, namely $A \cdot B \coloneqq A B$.
	For this we may take $\pre(B) \coloneqq B B^T$, which is equivariant since
	\[
		\pre(A \cdot B) = A B B^T A^T = A \pre(B) A^T = A \cdot \pre(B).
	\]
	As for other symmetrisation procedures, this could be used compositionally.
	For example, we could symmetrise along the homomorphisms
	\[
		\I \to \Orth(d) \hookrightarrow \GL(d, \R)
	\]
	in sequence to obtain an $\GL(d, \R)$-equivariant morphism starting from an arbitrary morphism in $\Set$ (as opposed to one that is already $\Orth(d)$-equivariant).
	This approach may be of interest in its own right for applications such as affine-invariant image classification \citep{macdonald2022enabling,mironenco2024lie}.
	More generally, it shows that our framework can encompass the actions of even very complex groups in a natural way.
\end{example}

\section{Application and numerical results}  \label{sec:empirical-results}

We now describe a concrete application of our stochastic symmetrisation approach.
In particular, we apply it to the method of \citet{kim2023learning}, which has obtained state-of-the-art results for deterministic symmetrisation across a variety of tasks.
Although their overall model is deterministic, \citet{kim2023learning} require a stochastically equivariant neural network as a crucial subcomponent, and use an intrinsically equivariant neural network for this purpose.
We show how this component can instead be obtained using our methodology, which allows us apply more flexible off-the-shelf architectures that are not subject to equivariant constraints.
Empirically, this leads to improved performance over the intrinsic approach on several synthetic examples.
More generally, we believe this case study also demonstrates the conceptual and notational precision that Markov categories provide for describing complex machine learning systems, which may be useful in other applications beyond ours.

Other use-cases for stochastic symmetrisation appear possible beyond the one we consider here.
In particular, our methodology could also be used more directly to obtain an equivariant model that is overall stochastic, rather than as a component of a deterministic symmetrisation procedure.
This may be of interest in applications such as deep generative modelling, or where uncertainty quantification is required.
We leave this for future work.

\subsection{Architecture}

We describe the approach of \citet{kim2023learning} within the context of our framework.
For concreteness, we formalise this entirely in $\Stoch$, which means we think of all model components as Markov kernels, including the deterministic neural networks that we use.
Given i.i.d.\ samples from some distribution $p(dx, dy)$ on $\X \otimes \Y$, the overall goal is to learn a deterministic $f : \X \to \Y$ that we will use as a predictor.
We assume that $\Y$ is real-valued (and possibly multidimensional), and that both $\X$ and $\Y$ are equipped with the actions of some group $\G$ in $\Stoch$, and we would like $f$ to be equivariant with respect to these actions.
We think of $f$ as depending on some additional parameters that correspond to neural network weights, although to streamline notation we keep these implicit in what follows.

\paragraph{Baseline}

The approach of \citet{kim2023learning} in effect applies the strategy described in Section \ref{sec:deterministic-symmetrisation} above: they first symmetrise an unconstrained neural network in $\Stoch$, and then average over the output to obtain a deterministic predictor.
More succinctly, they obtain $f$ via:
\begin{align}
	k &\coloneqq \sym_\pre(\nn_k) \notag \\
	f &\coloneqq \ave(k), \label{eq:kim-et-al-model-experiments}
\end{align}
where $\ave$ is the expectation operator from Definition \ref{def:expectation-operator}, and $\nn_k : \X \to \Y$ is some unconstrained neural network.
In both \citet{kim2023learning} and in our own experiments, this component is deterministic, although this is not strictly necessary.
Likewise, $\pre : (\X, \act_\X) \to (\G, \mul)$ is a morphism in $\Stoch^\G$.
For this component, \citet{kim2023learning} use the following architecture:
\begin{equation} \label{eq:kim-gamma-architecture}
	\input{kim-gamma-architecture.tikz}
\end{equation}
Here, at a high level:
\begin{itemize}
	\item $\nn_\pre$ is some deterministic neural network that forms the ``backbone'' of $\pre$
	\item $\eta$ is some noise distribution that allows $\pre$ overall to be stochastic
	\item $\proj$ projects its input onto $\G$, which is necessary because many groups live on a manifold, whereas neural networks typically output values in some Euclidean space.
\end{itemize}
If these components are all suitably equivariant, it follows immediately that their composition $\pre$ is too. %
To ensure this, \citet{kim2023learning} obtain $\nn_\pre$ by using some intrinsically equivariant architecture off-the-shelf.
They also provide choices of $\eta$ and $\proj$ suitable for several specific groups of interest \citep[Section 2.2]{kim2023learning}.

\paragraph{Our approach}

Rather than relying on an intrinsically equivariant neural network, we obtain $\pre$ itself through stochastic symmetrisation.
All up, our architecture is as follows:
\begin{align}
	\pre &\coloneqq \sym_{\pre_0}(\pre_1) \notag \\
	k &\coloneqq \sym_\pre(\nn_k) \notag \\
	f &\coloneqq \ave(k) \label{eq:our-approach-in-experiments}
\end{align}
This corresponds to symmetrising $\nn_k$ using the recursive approach described in Remark \ref{rem:recursive-symmetrisation}, and then again averaging the result as in Section \ref{sec:deterministic-symmetrisation}.
Here $\nn_k : \X \to \Y$ is again some unconstrained neural network, and $\pre_0 : (\X, \act_\X) \to (\G, \mul)$ is again some morphism in $\Stoch^\G$, which means that $\pre$ is a morphism in this category also.
However, now $\pre_1 : \X \to \G$ is an unconstrained morphism in $\Stoch$, and so can be chosen more freely.
For example, $\pre_1$ can have the same form as \eqref{eq:kim-gamma-architecture}, but where its ``backbone'' neural network, $\eta$, and $\proj$ are now completely unconstrained rather than $\G$-equivariant.

\subsection{Training objective}

We discuss how to train models of this kind.
Overall, we would like to learn the parameters of our predictor $f$ by stochastic gradient descent.
Recall from Definition \ref{def:expectation-operator} that to sample from $f(d\hat{y}|x) = \ave(k)(d\hat{y}|x)$ requires computing the integral $\int \hat{y} \, k(d\hat{y}|x)$.
This is usually intractable, which poses a challenge for obtaining unbiased gradient estimates of the expected loss.
However, for a loss function $\ell : \Y \otimes \Y \to \R$ that is convex in its second argument, Jensen's inequality yields the following upper bound for any $m \in \Nat$:
\begin{align}
	\int \ell(y, \hat{y}) \, f(d\hat{y}|x) \, p(dx, dy) \, 
		&= \int \ell\left(y, \int \hat{y} \, k(d\hat{y}|x) \right) \, p(dx, dy) \notag \\
		&\leq \int \int \ell\left(y, \frac{1}{m} \sum_{i=1}^m \hat{y}_i\right) \, k(d\hat{y}_1|x) \cdots k(d\hat{y}_m|x) \, p(dx, dy). \label{eq:expected-loss-upper-bound}
\end{align}
Provided $k$ is reparameterisable \citep[Section 2.4]{kingma2022autoencoding}, we may use Monte Carlo to estimate gradients of the right-hand side in the parameters of the model, allowing this bound to be used as a surrogate training objective in place of the left-hand side.
For both the baseline and our approach, reparameterisability holds if $\eta$ (which is the only source of randomness here) is some fixed noise distribution that does not depend on the parameters of the model, as we assume in our experiments.
This approach with $m = 1$ has previously appeared in equation (55) of \citet{kim2023learning}, as well as equations (10) of \citet{murphy2018janossy} and (9) of \citet{murphy2019relational}.

\begin{remark}
	We briefly explain why \eqref{eq:expected-loss-upper-bound} is reasonable to use as a surrogate objective.
	First, by a similar argument to Theorem 1 of \citet{burda2016importance}, it is straightforward to show that this bound becomes tighter as $m$ increases, and exact in the limit $m \to \infty$ under mild regularity conditions.
	In addition, for any fixed $m$, the bound is also exact if $k(d\hat{y}|x)$ is always Dirac, or equivalently if $k$ is deterministic \citep[Example 10.5]{fritz2020synthetic}.
	It follows that if \eqref{eq:expected-loss-upper-bound} is globally optimised over the parameters of the model, the resulting $k$ (which may be stochastic) will perform at least as well as the best-performing \emph{deterministic} $k$ the model can express.
	But now $k$ has at least two opportunities to become deterministic, which occurs if either:
	\begin{itemize}
		\item $\nn_k$ is deterministic and $\G$-equivariant.
		In this case $k = \sym_\pre(\nn_k) = \nn_k$ is also deterministic, where the second equality holds because $\sym_\pre$ is stable (Proposition \ref{prop:sym-gamma-is-stable}). This generalises an observation made in Section 2.3 of \citet{murphy2018janossy}.

		\item $\nn_k$ and $\pre$ are both deterministic.
		In this case, $k = \sym_\pre(\nn_k)$ is deterministic because $\sym_\pre$ restricts to a symmetrisation procedure in the determinsitic subcategory $\Stoch_\Det$.
		This in effect reduces to the canonicalisation approach of \citet{kaba2023equivariance}. %
	\end{itemize}
	This ``double robustness'' suggests that, under typical circumstances, it will be fairly easy for a model to approximate a rich family of deterministic $k$, which gives reason to anticipate good performance even when optimising the bound \eqref{eq:expected-loss-upper-bound} rather than the true expected loss.
\end{remark}

\subsection{Numerical examples}

We applied this setup to several synthetic problems, described now.
We chose examples that involve difficult computations and interesting equivariance constraints that can be easily considered across a range of dimensions.
However, our results are mainly intended to serve as a proof of concept that demonstrates the flexibility of our approach across a variety of complex group actions.
Further empirical work is needed to establish that the improvements reported here also carry over to other machine learning tasks at larger scales.

In each case, we obtained $p(dx, dy)$ as the distribution of some random variables $(\bm{X}, \bm{Y})$, whose definitions we give below along with the choice of $\G$ and its actions on $\X$ and $\Y$.
In all cases, it is straightforward to check that the following holds for all $g \in \G$:
\[
	(g \cdot \bm{X}, g \cdot \bm{Y}) \eqd (\bm{X}, \bm{Y}).
\]
This ensures the conditional of $\bm{Y}$ given $\bm{X}$ is also equivariant \citep[Proposition 1]{bloem2020probabilistic}, which motivates the use of an equivariant predictor $f$ in each case.
Each problem was also parameterised by a problem dimension $d$, which we varied to assess scalability.

\paragraph{Covariance estimation}

Our goal here was to learn to estimate the covariance matrix of a sample of i.i.d.\ Gaussian vectors.
In particular, we set $\X \coloneqq \R^{d \times n}$ and $\Y \coloneqq \R^d$, where $n \coloneqq 25$.
We then sampled $\bm{Y} \sim \mathrm{Wishart}_d(I_d, d)$, and obtained $\bm{X}$ consisting of $n$ i.i.d.\ columns, each a $d$-dimensional Gaussian with mean zero and covariance $\bm{Y}$.
We took $\G \coloneqq \Orth(d)$ to be the orthogonal group, and let this act on $\X$ and $\Y$ respectively via
\[
	Q \cdot A \coloneqq QA \qquad\qquad Q \cdot A \coloneqq QAQ^T.
\]

\paragraph{Linear regression}

Here the goal was to learn to estimate the coefficients of a linear regression model from some input data.
For this, we set $\X \coloneqq \R^{d \times n}$ and $\Y \coloneqq \R^d$, where again we took $n \coloneqq 25$.
We sampled regression coefficients $\bm{Y}$ from a standard $d$-dimensional Gaussian.
We then sampled a random $(d \times n)$-dimensional design matrix and $\bm{U}$ a $d$-dimensional noise vector with i.i.d.\ standard Gaussian entries, and set $\bm{X} \coloneqq (\bm{Z}, \bm{Y}^T \bm{Z}  + \bm{U})$ to be the pair of the design matrix and the observed values.
We took $\G \coloneqq \Orth(d)$, and let this act on $\X$ and $\Y$ respectively via
\[
	Q \cdot (z, y) \coloneqq (Qz, y) \qquad Q \cdot y \coloneqq Qy,
\]
where $z \in \R^{d \times n}$ and $y \in \R^d$.

\paragraph{Matrix exponentiation}

Our goal here was to learn the matrix exponential.
We took both $\X$ and $\Y$ to be $\R^{d \times d}$, with $\bm{X}$ having i.i.d.\ standard Gaussian entries, and $\bm{Y} \coloneqq e^{\bm{X}}$.
We set $\G \coloneqq \Orth(d)$ to be the orthogonal group, and let this act on both $\X$ and $\Y$ by
\[
	Q \cdot A \coloneqq Q A Q^T.
\]

\paragraph{Matrix inversion}

Our goal here was to learn the matrix inversion map.
For this, we took both $\X$ and $\Y$ to be the general linear group $\R^{d \times d}$, with $\bm{X}$ having i.i.d.\ standard Gaussian entries, and $\bm{Y} \coloneqq \bm{X}^{-1}$.
We set $\G \coloneqq \Orth(d) \times \Orth(d)$ to be the product of the orthogonal group with itself, and let this act on $\X$ and $\Y$ respectively by
\[
	(Q, R) \cdot A \coloneqq Q A R^T \qquad\qquad (Q, R) \cdot A \coloneqq R A Q^T.
\]
For the models involving EMLP architectures (see below), we could not use this action directly, since the current official implementation of EMLP\footnote{\url{https://github.com/mfinzi/equivariant-MLP}} does not support product groups where both components have the same type.
Accordingly, we took $\G \coloneqq \Orth(d)$, and let this act on $\X$ and $\Y$ via $Q \cdot A \coloneqq QA$ and $Q \cdot A \coloneqq AQ^T$ respectively instead.

\paragraph{Baseline}

We compared against the method of \citet{kim2023learning} in \eqref{eq:kim-et-al-model-experiments}.
We took $\nn_k : \X \to \Y$ to be an MLP with two hidden layers of 250 hidden units and $\tanh$ activations. %
Likewise, $\eta$ was a $d$-dimensional standard Gaussian, and $\proj$ was the Gram-Schmidt procedure, matching the choices in Section 2.2 of \citet{kim2023learning} for the orthogonal group.
For $\nn_\pre$ in \eqref{eq:kim-gamma-architecture}, we used an intrinsically equivariant EMLP with the same default architecture involving bilinear and gated nonlinear layers as proposed by \citet{finzi2021practical} (see their Figure 4).
This component had one hidden layer with 250 hidden units, following \citet{kim2023learning} in making this smaller than $\nn_k$.

\paragraph{Our method}

For our model \eqref{eq:our-approach-in-experiments}, we took $\nn_k$ to be the same as for the baselines.
We obtained the equivariant base case $\pre_0 : \X \to \G$ from the Haar measure as in \eqref{eq:haar-measure-as-base-case}, and for the ``backbone'' $\pre_1 : \X \to \G$ used the architecture \eqref{eq:kim-gamma-architecture}, with $\eta$ and $\proj$ the same as for the baselines, but with $\nn_{\pre_0}$ now an unconstrained MLP with one hidden layer of 250 hidden units.
To sample from the Haar measure on $\Orth(d)$ during training and testing, we used the method from Section 5 of \citet{mezzadri2007generate}.
We used $\tanh$ activations in all cases.

\paragraph{Additional baselines}

As a further baseline, we trained $\nn_k$ directly, without applying any symmetrisation.
We also trained a model in which $\nn_k$ was an intrinsically equivariant EMLP, as well as a version of \eqref{eq:kim-et-al-model-experiments} where $\pre$ was directly obtained using the Haar measure as in \eqref{eq:haar-measure-as-base-case}.
Since these models did not require an additional neural network for $\pre$, we made $\nn_k$ deeper than before, using three hidden layers of 250 hidden units instead of two.
We used $\tanh$ activations other than for the ELMP model, which has its own custom nonlinearities.

\paragraph{Training and testing details}

We did not use a finite dataset, but instead sampled new training and testing examples on the fly.
In this way, we sought to determine the overall expressiveness of the models without concerns about overfitting.
We used the mean squared error loss for the covariance estimation and linear regression problems, and the relative sum of squared errors loss for the matrix exponentiation and inversion problems, i.e.
\[
	\ell(y, \hat{y}) \coloneqq \frac{\norm{y - \hat{y}}^2}{\norm{y}^2},
\]
where $\norm{\cdot}$ denotes the $L^2$ norm, which is easily checked to be convex in $\hat{y}$.
We trained our symmetrised models using the upper bound \eqref{eq:expected-loss-upper-bound} as described earlier, taking $m \coloneqq 10$.
For the baseline unsymmetrised MLP, we optimised the expected loss directly.
In all cases we used the Adam optimiser \citep{kingma2014adam} with default hyperparameters and a learning rate of $10^{-4}$, and performed $10^5$ gradient steps with a batch size of 100.
At test time, given each test input $x \in \X$, we estimated the prediction of each symmetrised model $f(d\hat{y}|x) = \ave(k)(d \hat{y}|x)$ using Monte Carlo, averaging over $100$ i.i.d.\ samples from $k(dy|x)$.

\paragraph{Results}

Figure \ref{fig:matrix-inversion-results} shows the average test loss after training each model across the range of dimension $d$ that we considered.
Our method consistently performed on par or better than all other models, including those that used an intrinsically equivariant neural network for $\pre$.
This indicates that the greater flexibility obtained by obtaining $\pre$ via symmetrisation, rather than an intrinsically equivariant neural network, can lead to improved performance.
The EMLP models performed strongly in low dimensions, but did not scale well to higher dimensions.
Finally, the remaining two additional baselines performed the worst, which is somewhat to be expected given their more simplistic architectures.

\begin{figure}
	\centering
	\begin{subfigure}[b]{0.49\textwidth}
		\includegraphics[width=\textwidth]{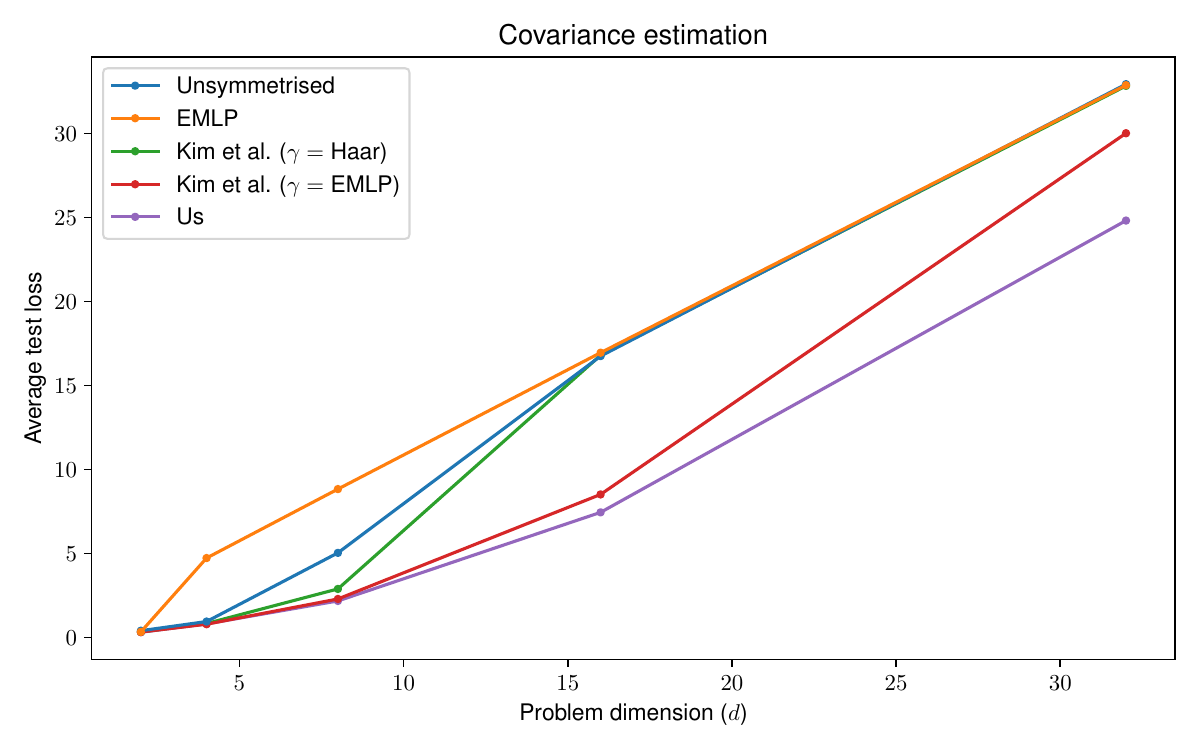}
	\end{subfigure}
	\hfill
	\begin{subfigure}[b]{0.49\textwidth}
		\includegraphics[width=\textwidth]{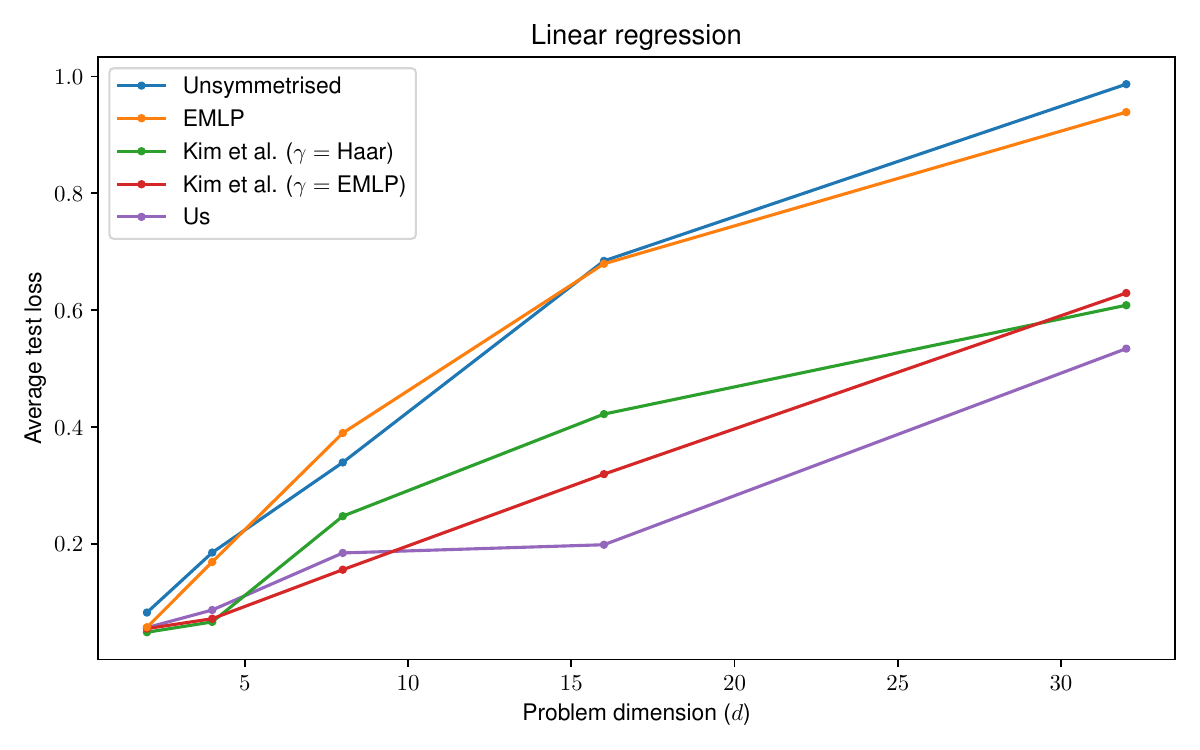}
	\end{subfigure}

	\begin{subfigure}[b]{0.49\textwidth}
		\includegraphics[width=\textwidth]{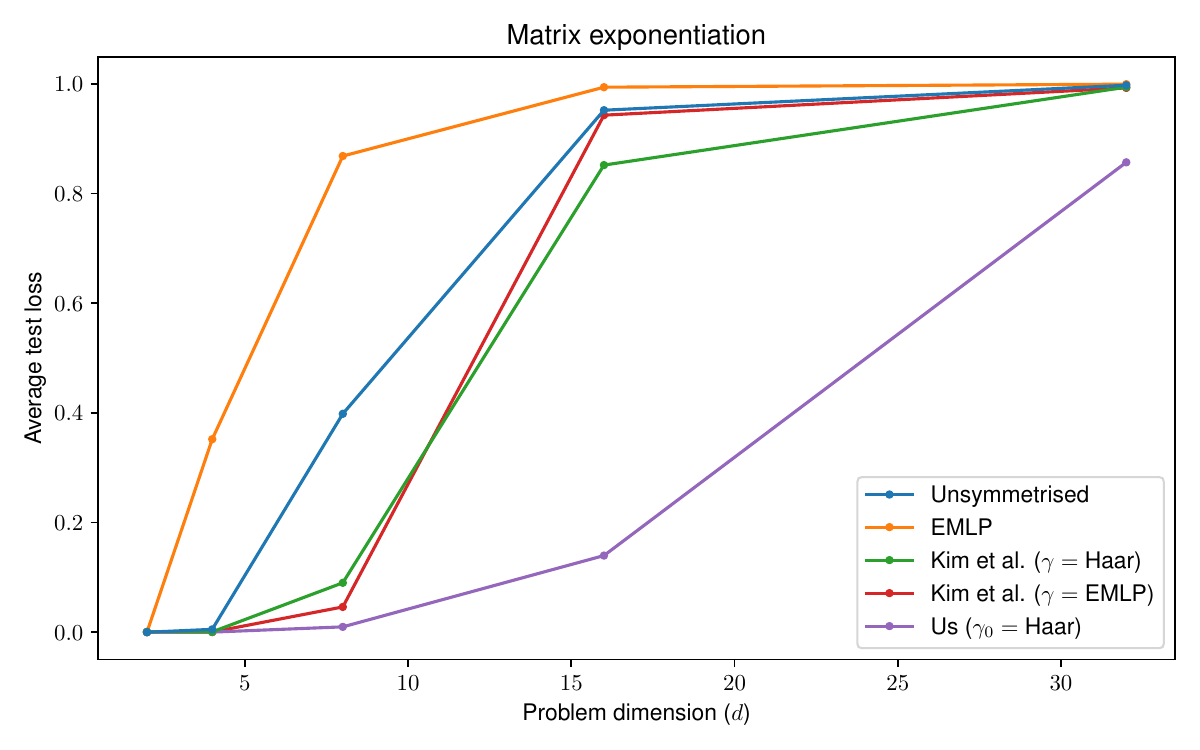}
	\end{subfigure}
	\hfill
	\begin{subfigure}[b]{0.49\textwidth}
		\includegraphics[width=\textwidth]{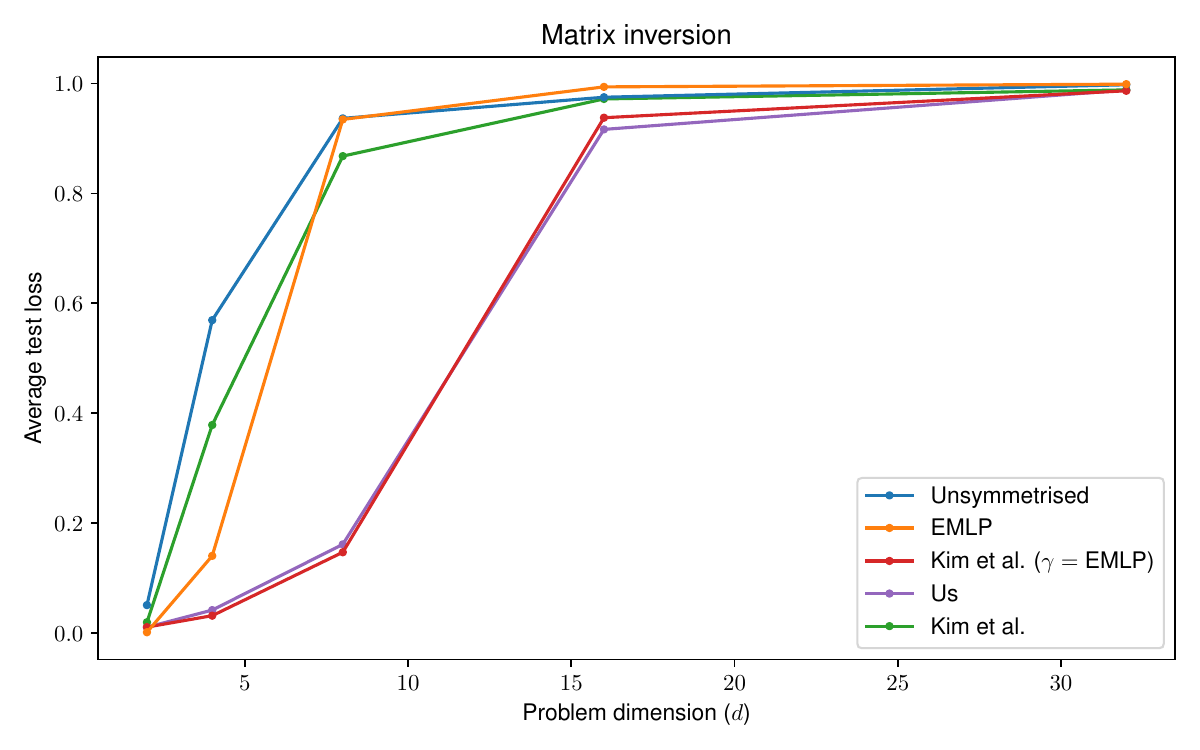}
	\end{subfigure}

	\caption{Average test loss obtained after training each model across the range of problem dimensions $d$ considered.} \label{fig:matrix-inversion-results}
\end{figure}

\section*{Acknowledgements}

This work would not have been possible without my parents John and Cathy and my partner Rachael, who supported me throughout and kept me focussed and on track.
I am also very grateful to Paolo Perrone for his encouragement, enthusiasm, and helpful conversations, and for the useful feedback he, Pedro Amorim, and Dario Stein gave on the first version of this paper.
My many thanks also to Sam Staton and his research group more generally for creating a welcoming and friendly community that has only made my enthusiasm for this topic grow.

\begin{small}
\bibliography{main}
\end{small}

\newpage

\appendix

{%
\addtocontents{toc}{\protect\setcounter{tocdepth}{0}} %

\section{Proofs: Group theory in Markov categories}

\subsection{Proof of Proposition \ref{prop:stochastic-equivariance-density}} \label{sec:proof-stochastic-equivariance-density}

\begin{proof}
	For all measurable $B \subseteq \Y$, $g \in \G$, and $x \in \X$, it follows that
	\begin{align*}
		k(B|g \cdot x) &= \int_{B} p(y|g \cdot x) \, \mu(dy) \\
		&= \int_B p(y|g \cdot x) \, (g \cdot \mu)(dy) \\
		&= \int_{g^{-1} \cdot B} p(g \cdot y|g \cdot x) \, \mu(dy) \\
		&= \int_{g^{-1} \cdot B} p(y|x) \, \mu(dy) \\
		&= k(g^{-1} \cdot B | x) \\
		&= g \cdot k(B|x),
	\end{align*}
	where the third step uses the law of the unconscious statistician.
	(Here $g^{-1} \cdot B \coloneqq \{g^{-1} \cdot y : y \in B\}$.)
	By Example \ref{ex:stochastic-equivariance}, this shows that $k$ is equivariant.
\end{proof}

\subsection{Results on orbit maps}

\begin{proposition} \label{prop:coequalisers-in-deterministic-category}
	Let $\C$ be a Markov category, and suppose the following is a coequaliser diagram in $\C$:
	\[
		\begin{tikzcd}
			X \ar[shift left=1.5]{r}{f} \ar[shift right=1.5,swap]{r}{g} & Y \ar{r}{h} & Z
		\end{tikzcd}
	\]
	If $f$, $g$, and $h$, are deterministic, then this is also a coequaliser diagram in $\Cdet$.
	Additionally, if $\C$ is positive and $f$ and $g$ are deterministic, then $h$ is deterministic.
\end{proposition}

\begin{proof}
	For the first statement, suppose $f$, $g$, and $h$ are deterministic, and let $k : Y \to W$ be a morphism in $\Cdet$ such that $k \circ f = k \circ g$.
	Then there is a unique morphism $k' : Z \to W$ in $\C$ such that $k = k' \circ h$.
	Since $h$ is a coequaliser, it is an epimorphism.
	By Lemma 10.9 of \citet{fritz2020synthetic}, it follows that $k'$ is deterministic, and hence a morphism in $\Cdet$ also.

	For the second statement, suppose $\C$ is positive and $f$ and $g$ are deterministic.
	Then we have
	\[
		\input{orbit-projections-are-deterministic-proof-4.tikz}
	\]
	where we use the fact that $f$ and $g$ are deterministic in the first and third steps.
	Since $h$ is a coequaliser, it follows that there exists a unique $k : \Z \to \Z \otimes \Z$ in $\C$ such that
	\begin{equation} \label{eq:orbit-projection-deterministic-proof-1}
		\input{orbit-projections-are-deterministic-proof-1.tikz}
	\end{equation}
	Marginalising out each output in turn and using the fact that $h$ is epi (since it is a coequaliser), we obtain
	\[
		\input{orbit-projections-are-deterministic-proof-2.tikz}
	\]
	Since $\C$ is positive, Theorem 2.8 of \citet{fritz2023dilations} now implies that $k = \cop_\Z$.
	Substituting this into \eqref{eq:orbit-projection-deterministic-proof-1}, it follows that $h$ is deterministic.
\end{proof}

\begin{proposition} \label{prop:orbit-map-preserved-action}
	Let $\C$ be a Markov category, and $\act : \G \otimes \X \to \X$ an action and $\Y$ an object in $\C$.
	If $q : \X \to \X/\G$ is an orbit map with respect to $\act$, then $q \otimes \id_\Y : \X \otimes \Y \to \X/\G \otimes \Y$ is an orbit map with respect to the action
	\begin{equation} \label{eq:orbit-map-preserved-action}
		\input{orbit-map-preserved-action.tikz}
	\end{equation}
\end{proposition}

\begin{proof}
	To see that \eqref{eq:orbit-map-preserved-action} is indeed an action, observe that it is just the diagonal action (see Example \ref{ex:diagonal-action}) obtained from $\act$ and the trivial action $\triv$.
	We now show the result in the case that $\C$ is strictly monoidal, with the general case being similar but notationally more complex.
	Since $q$ is an orbit map, it is a coequaliser.
	Moreover, it is preserved by the functor $(-) \otimes \Y$, and so the following is also a coequaliser diagram:
	\begin{equation} \label{eq:orbit-coequaliser-preserved-by-product}
		\begin{tikzcd}[column sep=4em]
			\G \otimes \X \otimes \Y \ar[shift left=1.5]{r}{\act \otimes \id_\Y} \ar[shift right=1.5, swap]{r}{\triv \otimes \id_\Y} & \X \otimes \Y \ar{r}{q \otimes \id_\Y} & \X/\G \otimes \Y.
		\end{tikzcd}
	\end{equation}
	Observe that $\act \otimes \id_\Y$ is just \eqref{eq:orbit-map-preserved-action}, and $\triv \otimes \id_\Y$ is just the trivial action on $\X \otimes \Y$.
	This shows that $q \otimes \id_\Y$ is a coequaliser of the form required to be an orbit map with respect to \eqref{eq:orbit-map-preserved-action}.
	It remains to show that this is preserved by every functor $(-) \otimes \Z$.
	But this holds because we have $\id_\Y \otimes \id_\Z = \id_{\Y \otimes \Z}$, and so the image of \eqref{eq:orbit-coequaliser-preserved-by-product} under $(-) \otimes \Z$ is just
	\[
		\begin{tikzcd}[column sep=5em]
			\G \otimes \X \otimes \Y \otimes \Z \ar[shift left=1.5]{r}{\act \otimes \id_{\Y \otimes \Z}} \ar[shift right=1.5, swap]{r}{\triv \otimes \id_{\Y \otimes \Z}} & \X \otimes \Y \ar{r}{q \otimes \id_{\Y \otimes \Z}} & \X/\G \otimes \Y \otimes \Z,
		\end{tikzcd}
	\]
	which is again a coequaliser diagram since $q$ is an orbit map.
\end{proof}

\subsection{$\Top$ has orbits}

\begin{proposition} \label{prop:top-has-orbits}
	Every action in $\Top$ admits an orbit map.
\end{proposition}

\begin{proof}
	We sketch the argument here, which uses standard ideas.
	Suppose $\act : \G \otimes \X \to \X$ is any action in $\Top$.
	Let $q : \X \to \X/\G$ be as defined in Example \ref{ex:orbit-map-example}, where now $\X/\G$ is equipped with the final topology with respect to $q$, which is namely the finest topology that makes $q$ continuous.
	Then $q$ becomes a coequaliser of $\act$ and $\triv$ by a similar argument as was given in Example \ref{ex:orbit-map-example}.
	(Slightly more care is needed in this case, because the final topology is in general not the only possible topology on $\X/\G$ that makes $q$ continuous.)

	For the preservation condition, given another topological space $\Y$, we must show that the following is also a coequaliser diagram:
	\[
		\begin{tikzcd}[column sep=3em]
			(\G \otimes \X) \otimes \Y \ar[shift left=1.5]{r}{\act \otimes \id_\Y} \ar[shift right=1.5,swap]{r}{\triv} & \X \otimes \Y \ar{r}{q \otimes \id_\Y} & \X/\G \otimes \Y
		\end{tikzcd}
	\]
	Now, certainly $q$ is a surjection, and it is also an open map \citep[11.1.2]{brown2006topology}.
	Likewise, $\id_\Y$ is an open surjection, and so it holds that $q \otimes \id_\Y$ is an open surjection too \citep[4.2, Exercise 7]{brown2006topology}.
	Consequently $\X/\G \otimes \Y$ must be equipped with the final topology with respect to $q \otimes \id_\Y$ \citep[4.2.4]{brown2006topology}.
	From here, it may be shown that $q \otimes \id_\Y$ is a coequaliser using essentially the same argument as for $q$ itself.
\end{proof}

\subsection{Proof of Theorem \ref{thm:topstoch-has-orbits}} \label{sec:top-has-orbits}

Our goal in this section is to prove that $\TopStoch$ admits all orbit maps.
Since $\TopStoch$ is the Kleisli category of a certain monad on $\Top$ \citep{fritz2021probability}, this \emph{almost} follows from Proposition \ref{prop:top-has-orbits}.
In particular, some straightforward diagram chasing shows that an orbit map exists for every group action in $\TopStoch$ that is obtained by lifting some group action in $\Top$ via Proposition \ref{prop:lifting-via-monad}.
However, we are not sure whether all actions in $\TopStoch$ arise in this way.
This is because a deterministic Markov kernel is only required to be zero-one \citep[Example 10.4]{fritz2020synthetic}, and in general not all zero-one Markov kernels can be obtained by lifting some measurable function.
(See Example 3.9 of \citet{moss2022probability} for an example in the case of $\Meas$ and $\Stoch$; a similar idea appears to hold for $\Top$ and $\TopStoch$ also.)
We therefore provide a direct proof of the result here.
Our approach takes its inspiration from Proposition 3.7 of \citet{moss2023category}, which shows the analogous statement in $\Stoch$.
We translate their construction to the topological setting and show it plays nicely with products, which in turn yields the preservation condition required by the definition of orbit maps.

Throughout this section, we use the usual notation $\ind_A$ to denote the indicator function of a measurable set $A$.
We also recall that, for $\X$ a topological space, a function $f : \X \to [0, 1]$ is \emph{lower semicontinuous} if for every $t \in [0, 1]$ the following set is open:
\[
	f^{-1}((t, 1]) = \{x \in \X \mid f(x) > t\}.
\]

\subsubsection{The invariant topology} \label{sec:invariant-topology}

Let $\act : \G \otimes \X \to \X$ be an action in $\TopStoch$.
We will keep $\act$ fixed throughout this subsection.
Given an open $U \subseteq \X$, we will say that $U$ is \emph{invariant} if it holds that
\[
	\act(U|g, x) = 1 \qquad \text{for all $x \in \U$ and $g \in \G$.}
\]
We then have the following.

\begin{proposition}
	The invariant open subsets of $\X$ form a topology.
\end{proposition}

\begin{proof}
	Certainly $X$ and $\emptyset$ are invariant ($\emptyset$ vacuously so).
	Given a collection of invariant open $U_i \subseteq \X$, letting $U \coloneqq \bigcup_i U_i$, we have that $U$ is open, and moreover if $g \in \G$ and $x \in U_i$ for some $i$, then
	\[
	    \alpha(U | g, x) \geq \alpha(U_i | g, x) = 1,
	\]
	which means $U$ is invariant.
	Similarly, given invariant open sets $U_1$ and $U_2$, their intersection $U_1 \cap U_2$ is open.
	Moreover, it satisfies $\alpha(U_1 \cap U_2 | g, x) = 1$ since $\alpha(U_1| g, x) = \alpha(U_2| g, x) = 1$ and is hence invariant.
\end{proof}

We may therefore define $\X/\G$ to be the topological space whose underlying set of points is the same as $\X$, but whose topology consists of the invariant open subsets of $\X$.
The following result shows that the Borel $\sigma$-algebra generated by $\X/\G$ is then a sub-$\sigma$-algebra of the invariant $\sigma$-algebra considered in Definition 3.6 of \citet{moss2023category}.

\begin{proposition} \label{prop:invariant-indicator-function}
	For all invariant open $U \subseteq \X$, it holds that $\act(U|g, x) = \ind_U(x)$ for all $g \in \G$ and $x \in \X$.
\end{proposition}

\begin{proof}
	Let $U$ be open and invariant.
	By definition, $\act(U|g, x) = 1 = \ind_U(x)$ if $x \in U$.
	On the other hand, if $x \not \in U$, it holds for all $g \in \G$ that (denoting the Markov kernel $\inv$ by $i$)
	\begin{align*}
		0 = \id_\X(U|x) &= \int \act(U|g', x') \, i(dg'|g) \, \act(dx'|g, x) \\
			&\geq \int \ind_U(x') \, \act(U|g', x') \, i(dg'|g) \, \act(dx'|g, x) \\
			&= \int \ind_U(x') \, i(dg'|g) \, \act(dx'|g, x) \\
			&= \act(U|g, x).
	\end{align*}
	Here the second step applies the group axioms, and the fourth step uses the fact that $U$ is invariant, so that $\act(U|g', x') = 1$ when $x' \in U$.
	It follows that $\act(U|g, x) = 0 = \ind_U(x)$ in this case as well.

\end{proof}

Suppose $U \subseteq \X$ is open, but not necessarily invariant.
The following construction provides a canonical way to obtain an invariant open set containing $U$.
In particular, we will define
\[
	\G \cdot U \coloneqq \{x \in \X \mid \text{$\act(U|g, x) = 1$ for some $g \in \G$}\}.
\]
We then have the following.

\begin{proposition} \label{prop:GU-properties}
	For all open $U \subseteq \X$, it holds that $G \cdot U$ is open and invariant, and $U \subseteq G \cdot U$.
\end{proposition}

\begin{proof}
	Let $U \subseteq \X$ be open.
	Given $g \in \G$, denote $\act_g(x) \coloneqq \act(U|g, x)$.
	Since $\act$ is deterministic, it is zero-one \citep[Example 10.4]{fritz2020synthetic}, and so
	\[
		\G \cdot U = \bigcup_{g \in \G} \act_g^{-1}(\{1\}) = \bigcup_{g \in \G} \act_g^{-1}((0, 1]).
	\]
	A standard argument shows $x \mapsto \act_g(x)$ is lower semicontinuous (since $(x, g) \mapsto \act_g(x)$ is).
	It follows that each $\act_g^{-1}((0, 1])$ is open, and hence so is $\G \cdot U$, being the union of open sets.

	To see that $\G \cdot U$ is invariant, let $x \in \G \cdot U$ and $g \in \G$.
	By definition, there exists $h \in \G$ such that (denoting the Markov kernels $\inv$ and $\mul$ by $i$ and $m$)
	\begin{align*}
		1 = \act(U|h, x) %
			&= \int \act(U|h, x'') \, \act(x''|g', x') \, i(dg'|g) \, \act(dx'|g, x) \\
			&= \int \act(U|h', x') \, m(dh'|h, g') \, i(dg'|g) \, \act(dx'|g, x) \\
			&= \int \ind_{\G \cdot U}(x') \, \act(U|h', x') \, m(dh'|h, g') \, i(dg'|g) \, \act(dx'|g, x) \\
			&\leq \act(\G \cdot U|g, x).
	\end{align*}
	Here the second and third steps apply the group axioms, while the fourth uses the fact that  $\act(U|h', x') = 0$ whenever $x' \not \in \G \cdot U$, by the definition of $\G \cdot U$.

	Finally, to show that $U \subseteq \G \cdot U$, let $x \in U$.
	Then the group axioms imply
	\begin{align*}
		\int \act(U|g, x) \, e(dg|\singleton)
			&= 1,
	\end{align*}
	where $\bullet$ denotes the unique element of the one-point space $\I$.
	Hence there exists $g \in \G$ with $\act(U|g, x) = 1$, so that $x \in \G \cdot U$.
\end{proof}

Our key reason for developing this construction in $\TopStoch$ is that the invariant topology plays well with products.
To make this precise, let $\Y$ be any topological space.
Then we always obtain an action as follows:
\begin{equation} \label{eq:orbit-map-preserved-action-1}
	\input{orbit-map-preserved-action.tikz}
\end{equation}
Analogously to $\X/\G$, we will denote by $(\X \otimes \Y)/\G$ the topological space consisting of $\X \otimes \Y$ equipped with the invariant topology induced by \eqref{eq:orbit-map-preserved-action-1}.
We then obtain the following result.

\begin{proposition} \label{prop:product-of-invariant-topologies}
	It holds that $(\X \otimes \Y)/\G = \X/\G \otimes \Y$.
\end{proposition}

\begin{proof}
	Suppose $V \subseteq \X$ is open and invariant with respect to $\act$, and $W \subseteq \Y$ is open.
	Then by definition of the product topology, $V \times W$ is open in $\X \otimes \Y$.
	It is moreover easily verified that $V \times W$ is invariant with respect to \eqref{eq:orbit-map-preserved-action-1}.
	Since the rectangle sets form a base for the product topology, this shows that every open subset of $\X/\G \otimes \Y$ is open in $(\X \otimes \Y)/\G$.

	Conversely, let $U \subseteq \X \otimes \Y$ be open and invariant with respect to $\act \otimes \id_\Y$.
	By definition of the product topology, we can write
	\[
		U = \bigcup_i V_i \times W_i
	\]
	for some open $V_i \subseteq \X$ and $W_i \subseteq \Y$.
	We claim that
	\begin{equation} \label{eq:invariant-topology-product-proof-1}
		U = \bigcup_i \G \cdot V_i \times W_i,
	\end{equation}
	from which it follows that $U$ is open in $\X/\G \otimes \Y$.
	Indeed, the $\subseteq$ inclusion follows immediately by Proposition \ref{prop:GU-properties}.
	For the $\supseteq$ inclusion, choose $x \in \G \cdot V_i$ and $y \in W_i$ arbitrarily.
	Since $U$ is invariant to \eqref{eq:orbit-map-preserved-action-1}, Proposition \ref{prop:invariant-indicator-function} implies the following for all $g \in \G$:
	\begin{align*}
		\ind_U(x, y) &= (\act \otimes \id_\Y)(U|(g, x), y) \\
			&\geq (\act \otimes \id_\Y)(V_i \times W_i|(g, x), y) \\
			&= \act(V_i |g, x) \, \ind_{W_i}(y) \\
			&= \act(V_i |g, x).
	\end{align*}
	By definition of $\G \cdot V_i$, there exists $g \in \G$ such that $\act(V_i|g, x) = 1$.
	Hence we must have $\ind_{U}(x, y) = 1$, and so $(x, y) \in U$.
\end{proof}

\subsubsection{Existence of orbit map coequalisers}

\begin{lemma} \label{lem:orbit-map-coequaliser-existence}
	Suppose $\act : \G \otimes \X \to \X$ be an action in $\TopStoch$. Let $\X/\G$ denote $\X$ equipped with the invariant topology induced by $\act$, and define $q : \X \to \X/\G$ by
	\[
		q(A|x) \coloneqq \delta_x(A) \qquad \text{for $x \in \X$ and Borel $A \subseteq \X/\G$,}
	\]
	where $\delta_x$ denotes the Dirac measure at $x$.
	Then $q$ is a coequaliser of the parallel arrows $\act, \triv : \G \otimes \X \rightrightarrows \X$.
\end{lemma}

\begin{proof}
	We first show that $q$ is indeed a well-defined morphism in $\TopStoch$.
	Certainly $A \mapsto q(A|x)$ is a probability measure for all $x \in \X$.
	Additionally, since $\X/\G$ is equipped with a coarser topology than $\X$, its Borel $\sigma$-algebra is coarser than that of $\X$ also.
	Hence for Borel $A \subseteq \X/\G$, the function $\X \to [0, 1]$ defined as
	\begin{align}
		x \mapsto q(A|x) = \ind_{A}(x) \label{eq:orbit-map-well-defined}
	\end{align}
	is always measurable with respect to the Borel $\sigma$-algebra on $\X$.
	It follows that $q$ is a Markov kernel.
	Moreover, for invariant open $A \subseteq \X$, the function \eqref{eq:orbit-map-well-defined} is lower semicontinuous with respect to the topology on $\X$, since it is the indicator function of an open subset of $\X$.
	This shows that $q$ is a morphism $\X \to \X /\G$ in $\TopStoch$.

	We now show that $q$ is a coequaliser.  %
	For this, we must first show that $q \circ \act = q \circ \triv$.
	This holds because for any invariant open $U \subseteq \X$, as well as $x \in \X$ and $g \in \G$, we have
	\begin{align*}
		\int q(U | x') \, \act(dx' | g, x)
			&= \int \ind_{U}(x') \, \act(dx'|g, x) \\
			&= \act(U | g, x) \\
			&= \ind_{U}(x) \\
			&= q(U | x) \\
			&= \int q(U | x') \, \triv(dx' | g, x),
	\end{align*}
	where the third step applies Lemma \ref{prop:invariant-indicator-function}.
	Since the invariant open sets generate the Borel $\sigma$-algebra on $\X/\G$, this shows $q \circ \act = q \circ \triv$.

	Now suppose $k : \X \to \Y$ in $\TopStoch$ also satisfies $k \circ \act = k \circ \triv$.
	Then a morphism $k' : \X/\G \to \Y$ satisfies $k' \circ q = k$ if and only if
	\begin{align*}
		k(V|x) &= \int k'(V|x') \, q(dx'|x) \\
		&= \int k'(V|x') \, \delta_x(dx') \\
		&= k'(V|x)
	\end{align*}
 	for all $x \in \X$ and open $V \subseteq \Y$.
	We are therefore done if we can show that $x \mapsto k(V|x)$ is lower semicontinuous with respect to the topology on $\X/\G$, so that we may take $k' \coloneqq k$.
	For this, choose any open $V \subseteq \Y$ and $t \in [0, 1]$.
	We would like to show that
	\[
		U \coloneqq \{x \in \X \mid k(V|x) > t\}
	\]
	is an invariant open subset of $\X$.
	Since $k$ is a morphism $\X \to \Y$ in $\TopStoch$ and hence lower semicontinuous, certainly $U$ is an open subset of $\X$.
	To prove that $U$ is invariant, we will show that $\G \cdot U \subseteq U$, which is sufficient by Proposition \ref{prop:GU-properties}.
	Indeed, if $x \in \G \cdot U$, then by definition there exists some $g \in \G$ such that $\act(U|g, x) = 1$.
	This yields
	\begin{align*}
		k(V|x) &= \int k(V|x') \, \act(dx'|g, x) \\
		&= \int_U k(V|x') \, \act(dx'|g, x) \\
		&> t,
	\end{align*}
	where the first step uses the fact that $k \circ \act = k \circ \triv$, and the third uses the definition of $U$.
	This gives $x \in U$, which is what we wanted to show.
\end{proof}

\subsubsection{Proof of Theorem \ref{thm:topstoch-has-orbits}}

\begin{proof}
	Suppose $\act : \G \otimes \X \to \X$ is an action in $\TopStoch$.
	By Lemma \ref{lem:orbit-map-coequaliser-existence}, it holds that $\act$ and $\triv$ have a coequaliser $q : \X \to \X/\G$ defined as $q(A|x) \coloneqq \delta_x(A)$, where $\X/\G$ denotes $\X$ equipped with the invariant topology induced by $\act$.
	Moreover, since $q$ is zero-one by definition, it is deterministic \citep[Example 10.4]{fritz2020synthetic}.
	Now let $\Y$ be an arbitrary topological space.
	We wish to show that this coequaliser is preserved by the functor $(-) \otimes \Y$.
	Applying Lemma \ref{lem:orbit-map-coequaliser-existence} again, now with respect to the action \eqref{eq:orbit-map-preserved-action-1}, we also obtain the following coequaliser diagram:
	\[
		\begin{tikzcd}[column sep=3em]
			(\G \otimes \X) \otimes \Y \ar[shift left=1.5]{r}{\act \otimes \id_\Y} \ar[shift right=1.5,swap]{r}{\triv} & \X \otimes \Y \ar{r}{r} & (\X \otimes \Y)/\G
		\end{tikzcd}
	\]
	where $(\X \otimes \Y)/\G$ denotes $\X \otimes \Y$ equipped with the invariant topology induced by \eqref{eq:orbit-map-preserved-action-1}, and $r(A| x, y) \coloneqq \delta_{(x, y)}(A)$.
	From Proposition \ref{prop:product-of-invariant-topologies}, we know that
	\[
		(\X \otimes \Y)/\G = \X/\G \otimes \Y.
	\]
	Additionally, given Borel $A \subseteq \X$ and $B \subseteq \Y$, we have
	\begin{align*}
		r(A \times B| x, y) &= \delta_{(x, y)}(A \times B) \\
			&= \delta_x(A) \, \delta_y(B) \\
			&= q(A|x) \, \id_{\Y}(B|y).
	\end{align*}
	This shows $r = q \otimes \id_{\Y}$, and the result now follows.
\end{proof}

\subsection{Orbit maps from coequalisers}

\begin{proposition} \label{prop:orbit-map-from-coequaliser}
	Let $\act : \G \otimes \X \to \X$ be an action in a Markov category $\C$.
	If $\act$ admits some orbit map, then every coequaliser $q$ of the form
	\begin{equation} \label{eq:orbit-map-coequaliser-proof}
		\begin{tikzcd}%
			\G \otimes \X \ar[shift left=1.5]{r}{\act} \ar[shift right=1.5,swap]{r}{\triv} & \X \ar{r}{q} & \X/\G
		\end{tikzcd}
	\end{equation}
	is an orbit map.
\end{proposition}

\begin{proof}
	Let $r : \X \to \Z$ be an orbit map for $\act$, and suppose $q : \X \to \X/\G$ a coequaliser of the form \eqref{eq:orbit-map-coequaliser-proof}
	Since $r$ is also a coequaliser, there exists a unique morphism $q'$ such that the triangle in the following diagram commutes:
	\[
		\begin{tikzcd}
			\G \otimes \X \ar[shift left=1.5]{r}{\act} \ar[shift right=1.5,swap]{r}{\triv} & \X \ar{r}{r} \ar[swap]{rd}{q} & \Z \ar[dashed, "q'"{right}, "\cong"{left}]{d}  \\
			& & \X/\G
		\end{tikzcd}
	\]
	Here $q'$ is moreover an isomorphism by the Yoneda lemma since $q$ is a coequaliser.
	Likewise, since $r$ is an orbit map, $r \otimes \id_\Y$ is a coequaliser, and so there exists a unique morphism $q''$ such that the triangle below commutes:
	\[
		\begin{tikzcd}[column sep=3em, row sep=3em]
			(\G \otimes \X) \otimes \Y \ar[shift left=1.5]{r}{\act \otimes \id_\Y} \ar[shift right=1.5,swap]{r}{\triv \otimes \id_\Y} & \X \otimes \Y \ar{r}{r \otimes \id_\Y} \ar[swap]{rd}{q \otimes \id_\Y}  & \Z \otimes \Y \ar[dashed]{d}{q''} \\
			&& \X/\G \otimes \Y
		\end{tikzcd}
	\]
	By uniqueness, we must have $q'' = q' \otimes \id_\Y$, which is then an isomorphism since $q'$ is.
	It follows that $q \otimes \id_\Y$ is a coequaliser of the parallel arrows in this last diagram, which shows that \eqref{eq:orbit-map-coequaliser-proof} is preserved by the functor $(-) \otimes \Y$.
\end{proof}

\subsection{Coset maps for semidirect products}

\begin{proposition} \label{prop:semidirect-product-coset-maps}
	Let $\C$ be a Markov category, and $\N \rtimes_\actr \HH$ a semidirect product in $\C$.
	Adopting the notation of Example \ref{ex:semidirect-product-coset-maps}, it holds that $i_\N$ and $i_\HH$ are both homomorphisms.
	Moreover, $p_\HH$ is an $i_\N$-coset map, and $p_\N$ is an $i_\HH$-coset map.
\end{proposition}

\begin{proof}
	A standard argument shows that $i_\N$ and $i_\HH$ are homomorphisms when $\C = \Set$, which translates to the general case by Remark \ref{rem:translating-classical-theory}.
	We prove that $p_\HH$ is an $i_\N$-coset map, with $p_\N$ being similar.
	For this, we must first prove that $p_\HH$ is a coequaliser of the form
	\begin{equation} \label{eq:semidirect-product-coset-proof-1}
		\begin{tikzcd}[column sep=3em]
			\N \otimes (\N \rtimes_\actr \HH) \ar[shift left=1.5]{r}{\act} \ar[shift right=1.5,swap]{r}{\triv} & \N \rtimes_\actr \HH \ar{r}{p_\HH} & \HH
		\end{tikzcd}
	\end{equation}
	where $\act$ denotes the action
	\[
		\input{semidirect-product-coset-proof-1.tikz}
	\]
	This firstly requires showing that $p_\HH$ is invariant with respect to $\act$, so that $p_\HH \circ \act = p_\HH \circ \triv$.
	By Remark \ref{rem:translating-classical-theory}, it suffices to do so when $\C = \Set$, where for $n, n' \in \N$ and $h \in \HH$ we have
	\begin{align*}
		p_\HH(n \cdot (n', h))
			&= p_\HH((n', h) \, (n^{-1}, e_\HH)) \\
			&= p_\HH(n' \, \rho(h, n^{-1}), h) \\
			&= h \\
			&= p_\HH(n', k).
	\end{align*}
	Next, suppose we have $m : \N \rtimes_\actr \HH \to \Y$ in $\C$ that is also invariant to $\act$.
	Then
	\[
		\input{semidirect-product-coset-proof-2.tikz}
	\]
	Here the first step holds because $m$ is invariant.
	The second step follows from some basic manipulations: in $\Set$, letting $\psi$ denote the dashed box, we have for $n \in \N$ and $h \in \HH$
	\begin{align*}
		\psi(n, h) &= (n, h) \, i_\HH(h)^{-1} \\
			&= (n, h) \, (e_\N, h^{-1}) \\
			&= (n, h h^{-1}) \\
			&= (n, e_\HH).
	\end{align*}
	This shows that $m = (m \circ i_\N) \circ p_\HH$.
	Since $p_\HH$ is easily seen to be an epimorphism, it follows that $m \circ i_\N$ is unique with this property.
	All up, this means \eqref{eq:semidirect-product-coset-proof-1} is a coequaliser diagram.
	A similar argument shows that this is preserved by every functor $(-) \otimes \Y$, from which it follows that $p_\HH$ is an $i_\N$-coset map.
\end{proof}

\subsection{Induced actions on orbits and cosets}

\begin{proposition} \label{prop:induced-action-equivariance}
	Let $\C$ be a Markov category, and $\act : \G \otimes \X \to \X$ and $\actb : \HH \otimes \X \to \X$ actions in $\C$ that commute in the sense of Remark \ref{rem:direct-product-actions} from the main text.
	If $q : \X \to \X/\HH$ is an orbit map for $\actb$, then there exists a unique action $\act/\HH : \G \otimes \X/\HH \to \X/\HH$ that makes $q$ equivariant with respect to $\G$ as follows:
	\begin{equation} \label{eq:canonical-action-equivariance}
		\input{canonical-action-equivariance.tikz}
	\end{equation}
\end{proposition}

\begin{proof}
	The idea is that we obtain $\act/\HH$ as the unique morphism in $\Cdet$ that makes the square below commute:
	\begin{equation} \label{eq:induced-action-equivariance-diagram}
		\begin{tikzcd}[column sep=5em]
			\G \otimes (\HH \otimes \X) \ar[shift left=1.5]{r}{\id_\G \otimes \actb} \ar[shift right=1.5,swap]{r}{\id_\G \otimes \triv} & \G \otimes \X \ar{r}{\id_\G \otimes q} \ar{d}{\act} & \G \otimes \X/\HH \ar[dashed]{d}{\act/\HH} \\
			& \X \ar{r}{q} & \X/\HH
		\end{tikzcd}
	\end{equation}
	Notice that this says that \eqref{eq:canonical-action-equivariance} holds.
	Now, by the preservation condition of orbit maps, it holds that $\id_\G \otimes q$ is a coequaliser in $\C$ of the parallel arrows shown.
	As such, $\act/\HH$ exists uniquely in $\C$ whenever
	\begin{equation*} %
		q \circ \act \circ (\id_\G \otimes \actb) = q \circ \act \circ (\id_\G \otimes \triv).
	\end{equation*}
	By Remark \ref{rem:translating-classical-theory}, it suffices to show this in $\Set$, where for $g \in \G$, $h \in \HH$, and $x \in \X$ we have simply
	\[
		q(g \cdot (h \cdot x)) = q(h \cdot (g \cdot x)) = q(g \cdot x).
	\]
	where the first step uses the assumption that $\act$ and $\actb$ commute, and the second uses the fact that $q$ is invariant to $\actb$.

	We are therefore done if we can show that $\act/\HH$ is an action.
	By Proposition \ref{prop:coequalisers-in-deterministic-category}, we know that $\act/\HH$ is deterministic, since all the other morphisms appearing in \eqref{eq:induced-action-equivariance-diagram} are.
	To see that $\act/\HH$ is associative, observe that
	\[
		\input{action-on-quotient-associative-proof-1.tikz}
	\]
	Here the first step uses \eqref{eq:canonical-action-equivariance} twice, the second uses associativity of $\act$, and the third uses \eqref{eq:canonical-action-equivariance} again.
	Now, since $q$ is an orbit map, $\id_{\K \otimes \K} \otimes q$ is a coequaliser.
	Since all coequalisers are epimorphisms, it follows that both sides here are equal even when $q$ is removed, which shows associativity.
	A similar argument shows that $\act/\HH$ is unital and completes the proof.
\end{proof}

\begin{proposition} \label{prop:induced-morphism-equivariance}
	Under the same setup of Proposition \ref{prop:induced-action-equivariance}, let $k/\HH : \X/\HH \to \Y$ be a morphism in $\C$ such that $k/\HH \circ q : \X \to \Y$ is equivariant with respect to $\act$ and some additional $\G$-action $\act_\Y : \G \otimes \Y \to \Y$.
	Then $k/\HH$ is equivariant with respect to $\act/\HH$ and $\act_\Y$ also.
\end{proposition}

\begin{proof}
	We have
	\[
		\input{induced-morphism-is-equivariant-to-quotient-action-2.tikz}
	\]
	where the first step uses the fact that $q$ is equivariant by Proposition \ref{prop:induced-action-equivariance}, and the second step uses the assumption that $k/\HH \circ q$ is equivariant.
	Since $q$ is an orbit map, $\id_\G \otimes q$ is a coequaliser, and hence an epimorphism.
	It follows that both sides are equal when $q$ is removed, which gives the result.
\end{proof}

\subsubsection{Proof of Proposition \ref{prop:induced-action-on-cosets}} \label{sec:proof-of-induced-action-on-cosets}

\begin{proof}
	We first show that the action $\mul$ of $\G$ on itself by left multiplication commutes with the action 
	\[
		\input{coset-map-definition-action.tikz}
	\]
	used in the definition of a $\varphi$-coset map.
	By Remark \ref{rem:translating-classical-theory}, it suffices to show this in $\Set$, where for all $g, g' \in \G$ and $h \in \HH$ we have
	\[
		g \cdot (h \cdot g') = g \cdot (g' \varphi(h)^{-1}) = g(g' \varphi(h)^{-1}) = (gg') \varphi(h)^{-1} = h \cdot (gg') = h \cdot (g \cdot g').
	\]
	The result now follows directly from Proposition \ref{prop:induced-action-equivariance}.
\end{proof}

\subsection{Proof of Proposition \ref{prop:lifting-via-monad}} \label{sec:proof-of-lifting-via-monad}

It is convenient to approach this result in more generality.
Let $\D$ be a Markov category, and suppose $\C$ is the Kleisli category of an affine symmetric monoidal monad on $\D$, whose unit we will denote by $\eta$.
Also denote by $L : \D \to \C$ the standard inclusion functor, which is defined as the identity on objects, and as $Lk \coloneqq \eta_\Y \circ k$ for morphisms $k : \X \to \Y$ in $\D$.
Recall from Corollary 3.2 of \citet{fritz2020synthetic} that $\C$ is canonically a Markov category, where the copy map for each object $\X$ in $\C$ is obtained as $L(\cop_\X)$.
We then have the following.

\begin{proposition} \label{prop:lifting-via-monad-general-result}
	Suppose $\D$ is a Markov category, and $\C$ is the Markov category obtained from an affine symmetric monoidal monad on $\D$ as just described.
	Then the standard inclusion functor lifts groups, homomorphisms, group actions, equivariant morphisms, orbit maps, and sections from $\D$ to $\C$.
\end{proposition}

\begin{proof}
	Let $\G$ be any group in $\D$ with operations $\mul$, $\e$, and $\inv$.
	Since the inclusion functor $L$ is the identity on objects, it lifts these operations to the following morphisms in $\C$:
	\[
		L(\mul) : \G \otimes \G \to \G \qquad L\e : \I \to \G \qquad L\inv : \G \to \G.
	\]
	Since $L$ is functorial, strictly monoidal \citep[Proposition 3.1]{fritz2020synthetic}, and preserves the copy maps in $\D$ by construction, these lifted morphisms are deterministic and satisfy the group axioms.
	In this way $\G$ becomes a group in $\C$.
	A similar argument shows that $L$ lifts homomorphisms, group actions, equivariant morphisms, and sections.
	Finally, $L$ lifts orbit maps because it is a left adjoint and so preserves colimits (see e.g.\ Section 5.1.2 of \citet{perrone2021notes} and Corollary 4.3.2 of \citet{perrone2021notes}), and therefore preserves coequalisers in particular.
\end{proof}

We now obtain the following proof of Proposition \ref{prop:lifting-via-monad}.

\begin{proof}
	As is well known, $\Stoch$ corresponds to the Kleisli category of the Giry monad on $\Meas$ \citep{giry1982categorical}, which is affine symmetric monoidal \citep[Lemma 4.1]{fritz2020synthetic}.
	The unit of this monad is given by $\delta$, where $\delta_\X(x)$ is the Dirac measure at $x \in \X$.
	As such, for $f : \X \to \Y$ in $\Meas$, we have
	\[
		Lf = \delta_\Y \circ f,
	\]
	which corresponds to the morphism $k_f$ obtained by lifting $f$ as in Remark \ref{rem:measurable-function-as-markov-kernel}.
	As such, by Proposition \ref{prop:lifting-via-monad-general-result}, this lifting operation preserves groups, homomorphisms, and so on.
	A similar story holds for $\TopStoch$, where now the relevant monad is defined on $\Top$ (see Corollary 4.17 of \citet{fritz2021probability}).
\end{proof}

\section{Proofs: Symmetrisation procedures}

\begin{proposition}  \label{prop:markov-category-of-equivariant-morphisms}
	Let $\G$ be a group in a Markov category $\C$.
	Then $\C^\G$ as described in Definition \ref{def:markov-category-of-equivariant-morphisms} is always a Markov category.
\end{proposition}

\begin{proof}
	It is straightforward to check that the composition of equivariant morphisms is equivariant, and that identity morphisms are always equivariant when their domain and codomain are equipped with the same action of $\G$.
	In this way, $\C^\G$ is a category.
	It is also clear that $\otimes$ is a bifunctor on $\C^\G$.
	We therefore only need to show that the structure maps as defined for $\C^\G$ satisfy the axioms of a Markov category (including those of a symmetric monoidal category).
	By definition, these structure maps are inherited from $\C$, and so by Lemma 10.12 of \citet{fritz2020synthetic} are all deterministic, which means we can equivalently show this for $(\Cdet)^\G$.
	But now the latter is just the Eilenberg-Moore category of the action monad $\G \otimes (-)$ on $\Cdet$ (see e.g.\ Section 5.2 of \citet{perrone2021notes}), which is cartesian monoidal \citep[Remark 10.13]{fritz2020synthetic}.
	Since the forgetful functor $(\Cdet)^\G \to \Cdet$ is monadic, it creates limits \citep[Theorem 5.6.5]{riehl2017category}, and so this cartesian monoidal structure lifts to $(\Cdet)^\G$ in the way described in Definition \ref{def:markov-category-of-equivariant-morphisms}.
\end{proof}

\section{Proofs: A general methodology for symmetrisation}

\subsection{Proof of Theorem \ref{thm:partial-left-adjoint-existence}} \label{sec:existence-of-left-adjoint-proof}

\begin{remark} \label{rem:abusing-yoneda}
	Recall from Remark \ref{rem:translating-classical-theory} that the Yoneda Lemma allows us to lift equations that hold for all groups, actions, etc.\ in $\Set$ to equations that hold more generally in the deterministic subcategory of an arbitrary Markov category.
	At several points in this subsection, we will abuse this technique by applying it even when some morphisms involved are not deterministic.
	This will streamline our proofs considerably, which become long-winded when expressed in terms of string diagrams.
	It will also demonstrate how our arguments correspond to the classical set-theoretic ones, which are standard \citep[Chapter I.1]{may1997equivariant}.
	In the few cases where we do this (which we will flag), it will be clear how to translate our set-theoretic manipulations into a general string-diagrammatic argument.
	Our approach here may therefore be regarded essentially as a convenient shorthand for the ``real'' proof.
	A more formal justification may be possible: the key idea seems to be that, in cases where this approach is valid, we do not reuse the output of any nondeterministic morphism more than once, which seems to be where potential issues could arise (see Section 7.1 of \citet{stein2021structural}).
\end{remark}

\begin{lemma} \label{lem:transpose-characterisation}
	Let $\C$ be a Markov category, and $\varphi : \HH \to \G$ a homomorphism and $q : \G \to \G/\HH$ a $\varphi$-coset map in $\C$.
	For all $k : \Res_\varphi(\X, \act_\X) \to \Res_\varphi(\Y, \act_\Y)$ in $\C^\HH$, there exists a unique $k^\transpose : \G/\HH \otimes \X \to \Y$ in $\C$ such that
	\begin{equation} \label{eq:transpose-characterisation-1}
		\input{transpose-characterisation.tikz}
	\end{equation}
\end{lemma}

\begin{proof}
	For brevity, let $\actr : \HH \otimes (\G \otimes \X) \to \G \otimes \X$ denote the following:
	\begin{equation} \label{eq:right-multiplication-times-identity-action}
		\input{right-multiplication-times-identity-action.tikz}
	\end{equation}
	This is seen to be the diagonal action (Example \ref{ex:diagonal-action}) obtained from the action \eqref{eq:rmul-definition} and the trivial action on $\X$.
	Also denote by $m : \G \otimes \X \to \Y$ denote the right-hand side of \eqref{eq:transpose-characterisation-1}.
	We claim that $m \circ \actr = m \circ \triv$, or in other words that $m$ is invariant with respect to $\actr$.
	Noting the caveat of Remark \ref{rem:abusing-yoneda}, we show this in $\Set$, where for $g \in \G$, $h \in \HH$, and $x \in \X$ we have
	\begin{align*}
		m(h \cdot (g, x)) &= m(g \, \varphi(h)^{-1}, x) \\
			&= (g \, \varphi(h)^{-1}) \cdot k((g \, \varphi(h)^{-1})^{-1} \cdot x) \\
			&= g \cdot (h^{-1} \cdot k(h \cdot (g \cdot x))) \\
			&= g \cdot (h^{-1} \cdot h \cdot k(g \cdot x)) \\
			&= g \cdot k(g \cdot x) \\
			&= m(g, x).
	\end{align*}
	Here the first two steps apply the definitions of $\actr$ and $m$, and the third uses the definition of the $\HH$-actions that equip $\Res_\varphi(\X, \act_\X)$ and $\Res_\varphi(\Y, \act_\Y)$.
	The fourth step then uses the fact that $k$ is $\HH$-equivariant.
	From invariance of $m$ and the universal property of orbit maps, we obtain a unique morphism $k^\transpose$ in $\C$ such that the triangle in the following diagram commutes:
	\[
		\begin{tikzcd}[column sep=3em, row sep=3em]
			\HH \otimes (\G \otimes \X) \ar[shift left=1.5]{r}{\actr} \ar[shift right=1.5,swap]{r}{\triv} & \G \otimes \X \ar{r}{q \otimes \id_\X} \ar[swap]{dr}{m} & \G/\HH \otimes \X \ar[dashed]{d}{k^\transpose} \\
			& & \Y
		\end{tikzcd}
	\]
	The commuting triangle here says exactly that \eqref{eq:transpose-characterisation-1} holds, which gives the result.
\end{proof}

\begin{lemma} \label{lem:transpose-characterisation-2}
	The morphism $k^\transpose$ described in Lemma \ref{lem:transpose-characterisation} is always a morphism in $\C^\G$ of the form
	\[
		k^\transpose : (\G/\HH, \mul/\HH) \otimes (\X, \act_\X) \to (\Y, \act_\Y)
	\]
	where $\mul/\HH : \G \otimes \G/\HH \to \G/\HH$ denotes the unique action induced by Proposition \ref{prop:induced-action-on-cosets}.
\end{lemma}

\begin{proof}
	We adopt the same notation as in the proof of Lemma \ref{lem:transpose-characterisation}.
	By definition of $\C^\G$, the action that equips $(\G, \mul/\HH) \otimes (\X, \act_\X)$ is as follows:
	\begin{equation} \label{eq:diagonal-action-on-k-sharp}
		\input{diagonal-action-on-k-sharp.tikz}
	\end{equation}
	We will first show that this commutes with the $\HH$-action \eqref{eq:right-multiplication-times-identity-action}.
	By Remark \ref{rem:translating-classical-theory}, it suffices to do so in $\Set$: given $g, g' \in \G$, $h \in \HH$ and $x \in \X$, we have
	\begin{align*}
		g \cdot (h \cdot (g', x)) &= g \cdot (g' \varphi(h)^{-1}, g \cdot x) \\
			&= (g g' \varphi(h)^{-1}, g \cdot x) \\
			&= h \cdot (g g', g \cdot x) \\
			&= h \cdot (g \cdot (g', x)).
	\end{align*}
	Next, we show that $m$ is equivariant with respect to \eqref{eq:diagonal-action-on-k-sharp} and $\act_\Y$.
	Noting the caveat of Remark \ref{rem:abusing-yoneda}, we demonstrate this in $\Set$, where we have
	\begin{align*}
		m(g \cdot (g', x)) &= m(gg', g \cdot x) \\
			&= (gg') \cdot k((gg')^{-1} \cdot g \cdot x) \\
			&= g \cdot (g' \cdot k((g')^{-1} \cdot x)) \\
			&= g \cdot m(g', x).
	\end{align*}
	Now recall that $q \otimes \id_\X$ is an orbit map with respect to the $\HH$-action \eqref{eq:right-multiplication-times-identity-action}.
	Since \eqref{eq:transpose-characterisation-1} says
	\[
		m = k^\transpose \circ (q \otimes \id_\X),
	\]
	Proposition \ref{prop:induced-morphism-equivariance} implies that $k^\transpose$ is equivariant with respect to \eqref{eq:diagonal-action-on-k-sharp} and $\act_\Y$ as desired.
\end{proof}

\subsubsection{Proof of Theorem \ref{thm:partial-left-adjoint-existence}}

\begin{proof}
	By Lemmas \ref{lem:transpose-characterisation} and \ref{lem:transpose-characterisation-2}, the assignment $k \mapsto k^\transpose$ defines a function of the required form
	\begin{equation} \label{eq:left-adjoint-definition-on-morphisms-3}
		\C^\HH(\Res_\varphi(\X, \act_\X), \Res_\varphi(\Y, \act_\Y)) \to \C^\G((\G/\HH, \mul/\HH) \otimes (\X, \act_\X), (\Y, \act_\Y)).
	\end{equation}
	It is also straightforward to check that this is natural in $(\X, \act_\X)$ and $(\Y, \act_\Y)$.
	We will show it also has an inverse $m \mapsto m \circ \eta$, where $\eta$ is the morphism
	\[
		\input{unit-definition-2.tikz}
	\]
	We will show the assignment $k \mapsto k^\sharp$ from Lemma \ref{lem:orbit-map-coequaliser-existence}, which is a function of the form required by the statement of this Theorem by Lemma \ref{lem:transpose-characterisation-2}, has an inverse, namely $m \mapsto m \circ \eta$.

	First, we must check that $m \mapsto m \circ \eta$ is actually well-typed.
	We do so by showing that $\eta$ is a morphism in $\C^\HH$ of the following type:
	\begin{equation} \label{eq:left-adjoint-definition-on-morphisms-1}
		\Res_\varphi(\X, \act_\X) \to \Res_\varphi((\G/\HH, \mul/\HH) \otimes (\X, \act_\X)).
	\end{equation}
	Since $\Res_\varphi$ is the identity on morphisms, it follows by definition of composition in $\C^\HH$ that $m \circ \eta = \Res_\varphi(m) \circ \eta$ is a element of the left-hand side of \eqref{eq:left-adjoint-definition-on-morphisms-3} whenever $m$ is an element of the right-hand side.
	To show \eqref{eq:left-adjoint-definition-on-morphisms-1}, since all the morphisms involved in the definition of $\eta$ are deterministic, Remark \ref{rem:translating-classical-theory} allows us to work in $\Set$, where for $h \in \HH$ and $x \in \X$ we have
	\begin{align*}
		\eta(h \cdot x) &= (q(e), \varphi(h) \cdot x) \\
			&= (q(\varphi(h) \, \varphi(h)^{-1}), \varphi(h) \cdot x) \\
			&= (\varphi(h) \cdot q(\varphi(h)^{-1}), \varphi(h) \cdot x) \\
			&= (\varphi(h) \cdot q(e), \varphi(h) \cdot x) \\
			&= h \cdot \eta(x).
	\end{align*}
	Here the third step uses the fact that $q$ is $\G$-equivariant with respect to $\mul$ and $\mul/\HH$ by Proposition \ref{prop:induced-action-equivariance}, and the fourth uses fact that $q$ is $\HH$-invariant since it is an orbit map.

	Now we claim that for any morphisms $k$ and $m$ living in the left- and right-hand sides of \eqref{eq:left-adjoint-definition-on-morphisms-3} respectively, it holds that $k = m \circ \eta$ if and only if $m = k^\transpose$.
	The ``if'' direction establishes that $m \mapsto m \circ \eta$ is surjective, while the ``only if'' direction establishes injectivity.
	Noting the caveat of Remark \ref{rem:abusing-yoneda}, we demonstrate the ``if'' direction in $\Set$ as follows, where for $x \in \X$ we have
	\begin{align*}
		(k^\transpose \circ \eta)(x) &= k^\transpose(q(e), x) \\
			&= e \cdot k(e^{-1} \cdot x) \\
			&= k(x).
	\end{align*}
	Here the second step uses \eqref{eq:transpose-characterisation-1}.
	For the ``only if'' direction, suppose $m \circ \eta = k$.
	From the uniqueness part of Lemma \ref{lem:transpose-characterisation}, it follows that $m = k^\transpose$ if we can show that \eqref{eq:transpose-characterisation-1} holds when $m \circ \eta$ is substituted for $k$ on its right-hand side.
	Again noting Remark \ref{rem:abusing-yoneda}, we demonstrate this in $\Set$: given any $g \in \G$ and $x \in \X$, we have
	\begin{align*}
		g \cdot (m \circ \eta)(g^{-1} \cdot x) &= g \cdot m(q(e), g^{-1} \cdot x) \\
			&= m(g \cdot q(e), g \cdot g^{-1} \cdot x) \\
			&= m(q(g), x),
	\end{align*}
	where the first step uses the assumption that $m$ is $\HH$-equivariant, and the second uses the fact that $q$ is $\G$-equivariant with respect to $\mul$ and $\mul/\HH$.
	Since $k^\transpose$ is unique with this property by Lemma \ref{lem:transpose-characterisation}, it follows that $m = k^\transpose$ as desired, which gives the result.
\end{proof}

\subsection{Proof of Proposition \ref{prop:sym-gamma-is-stable}} \label{sec:proof-of-sym-gamma-is-stable}

\begin{proof}
	Given arbitrary $\Pre : (\X, \act_\X) \to (\G/\HH, \mul/\HH) \otimes (\X, \act_\X)$ in $\C^\G$, we will denote
	\begin{equation} \label{eq:sym-gamma-is-stable-proof-1}
		\sym_\Pre(k) \coloneqq k^\transpose \circ \Pre,
	\end{equation}
	where $k^\transpose$ is obtained via Theorem \ref{thm:partial-left-adjoint-existence}.
	In other words, this is just like $\sym_\pre$, but where its precomposition morphism is allowed to be arbitrary, rather than taking the specific form \eqref{eq:precomposition-morphism-2} from the main text.
	Now let $k : (\X, \act_\X) \to (\Y, \act_\Y)$ be a morphism in $\C^\G$.
	It is then easily verified that $\sym_\Pre$ is \emph{natural} in the following sense:
	\[
		\sym_\Pre(k) = \sym_\Pre(k \circ \id_\X) = k \circ \sym_\Pre(\id_\X).
	\]
	As a result, $\sym_\Pre$ is stable if and only if $\sym_\Pre(\id_\X) = \id_\X$.
	Now recall that $(\id_\X)^\transpose$ is the unique morphism in $\C$ such that
	\begin{equation} \label{eq:sym-gamma-is-stable-proof-2}
		\input{stability-proof-1.tikz}
	\end{equation}
	(where the second step is shown in Remark \ref{rem:translating-classical-theory}), and so we must have
	\[
		\input{stability-proof-2.tikz}
	\]
	since certainly \eqref{eq:sym-gamma-is-stable-proof-2} holds in this case.
	It follows that $\sym_\Pre$ is stable if and only if
	\[
		\input{stability-proof-3.tikz}
	\]
	This condition is always satisfied for $\Pre$ of the form \eqref{eq:precomposition-morphism-2} from the main text.
	Conversely, if $\C$ is positive, this condition implies that $\Pre$ has the form \eqref{eq:precomposition-morphism-2} by Theorem 2.8 of \citet{fritz2023dilations}.
\end{proof}

\addtocontents{toc}{\protect\setcounter{tocdepth}{3}} %
}

\end{document}